\documentclass[sigplan,10pt]{acmart}

\usepackage{fancyvrb,fvextra}

\usepackage{multirow}
\usepackage{makecell}
\usepackage{paralist}
\usepackage{xspace}

\usepackage{enumerate}
\usepackage[utf8]{inputenc} %
\usepackage[T1]{fontenc}    %
\usepackage{url}            %
\usepackage{booktabs}       %
\usepackage{nccmath}       %
\usepackage{nicefrac}       %
\usepackage{microtype}      %
\usepackage[ruled,vlined,lined,boxed]{algorithm2e}
\usepackage{mwe}    %
\usepackage{float}
\usepackage{lineno}
\usepackage[font=small,format=plain,labelfont=bf,textfont=bf]{caption}
\usepackage[skip=0cm,list=true]{subcaption}
\usepackage{xcolor}
\definecolor{codegreen}{rgb}{0,0.6,0}
\definecolor{codegray}{rgb}{0.5,0.5,0.5}
\definecolor{codepurple}{rgb}{0.58,0,0.82}
\definecolor{backcolour}{rgb}{0.95,0.95,0.92}

\newcommand{\smartparagraph}[1]{\noindent{\bf #1}\ }

\usepackage[capitalise]{cleveref}
\usepackage{paralist}

\usepackage{amsmath}
\usepackage{amsthm}
\usepackage{amssymb} %
\usepackage{bbm}
\usepackage{mdframed} 
\usepackage{thmtools}

\newcommand*{\defeq}{\mathrel{\vcenter{\baselineskip0.5ex \lineskiplimit0pt
                     \hbox{\scriptsize.}\hbox{\scriptsize.}}}%
                     =}

\usepackage{graphicx}

\newcommand{\R}{\mathbb{R}} %
\newcommand{\mbE}{\mathbb{E}} %

\newcommand{\cD}{{\mathcal{D}}}

\newcommand{\cF}{{\mathcal{F}}}

\newcommand{\cO}{{\mathcal{O}}}

\newcommand{\cS}{{\mathcal{S}}}

\usepackage[colorinlistoftodos,bordercolor=orange,backgroundcolor=orange!20,linecolor=orange,textsize=scriptsize]{todonotes}

\newcommand{\dotprod}[1]{\left< #1\right>} %
\newcommand{\norm}[1]{\lVert#1\rVert}      %

\usepackage{footnote}
\makesavenoteenv{tabular}
\makesavenoteenv{table}

 \newtheorem*{lemma*}{Lemma}
 \newtheorem*{proposition*}{Proposition}
 \newtheorem*{theorem*}{Theorem}
 \newtheorem*{corollary*}{Corollary}
 \newtheorem*{remark*}{Remark}

\newtheorem{assumption}{Assumption}
\newtheorem{lemma}{Lemma}
\newtheorem{theorem}{Theorem}

\theoremstyle{definition}

\theoremstyle{remark}

\SetKwProg{Event}{\textbf{on event}}{:}{}

\SetCommentSty{mycommfont}

\newcommand{\scheme}{\textit{REFL}\xspace}

\definecolor{mlAc}{RGB}{255, 128, 33}
\definecolor{mlBc}{RGB}{83, 161, 69}
\definecolor{mlCc}{RGB}{255, 50, 10}

\newcommand*{\mlA}{\textcolor{mlAc}{$\blacksquare$}\xspace}
\newcommand*{\mlB}{\textcolor{mlBc}{$\blacksquare$}\xspace}
\newcommand*{\mlC}{\textcolor{mlCc}{$\blacksquare$}\xspace}

\newif\ifsubmission
\submissiontrue

\ifsubmission
\newcommand{\ahmed}[1]{}
\newcommand{\mcnote}[1]{}
\else
\newcommand{\ahmed}[1]{\textit{\textcolor{red}{{Ahmed:}#1}}}
\newcommand{\mcnote}[1]{\textit{\textcolor{blue}{[marco]:#1}}}

\fi

\usepackage{enumitem}[noitemsep,topsep=0pt,leftmargin=0pt]
\setitemize{noitemsep,topsep=0pt,parsep=0pt,partopsep=0pt}
\setenumerate{noitemsep,topsep=0pt,parsep=0pt,partopsep=0pt}

\copyrightyear{2023} 
\acmYear{2023} 
\setcopyright{acmlicensed}\acmConference[EuroSys '23]{Eighteenth European Conference on Computer Systems}{May 8--12, 2023}{Rome, Italy}
\acmBooktitle{Eighteenth European Conference on Computer Systems (EuroSys '23), May 8--12, 2023, Rome, Italy}
\acmPrice{15.00}
\acmDOI{10.1145/3552326.3567485}
\acmISBN{978-1-4503-9487-1/23/05}


\begin{document}

\title{REFL: Resource-Efficient Federated Learning}

\author{Ahmed M. Abdelmoniem}
\affiliation{\institution{Queen Mary University of London}}
\email{ahmed.sayed@qmul.ac.uk}
\authornote{Corresponding author, also with Assiut University, Egypt. Work done primarily while the author was with KAUST.}
\author{Atal Narayan Sahu}
\affiliation{\institution{KAUST}}
\author{Marco Canini}
\affiliation{\institution{KAUST}}
\author{Suhaib A. Fahmy}
\affiliation{\institution{KAUST}}

\begin{abstract}
Federated Learning (FL) enables distributed training by learners using local data, thereby enhancing privacy and reducing communication.
However, it presents numerous challenges relating to the heterogeneity of the data distribution, device capabilities, and participant availability as deployments scale, which can impact both model convergence and bias. Existing FL schemes use random participant selection to improve the fairness of the selection process; however, this can result in inefficient use of resources and lower quality training. In this work, we systematically address the question of resource efficiency in FL, showing the benefits of intelligent participant selection, and incorporation of updates from straggling participants. We demonstrate how these factors enable resource efficiency while also improving trained model quality.
\end{abstract}

\maketitle
\renewcommand{\shortauthors}{Ahmed M. Abdelmoniem, Atal N. Sahu, Marco Canini, and Suhaib A. Fahmy}

\section{Introduction}
\label{sec:intro}

Recently distributed machine learning (ML) deployments have sought to push computation towards data sources in an effort to enhance privacy and security~\cite{Bonawitz19,kairouz2019advances}. Training models using this approach is known as Federated Learning (FL). FL presents a variety of challenges due to the high heterogeneity of participating devices, ranging from powerful edge clusters and smartphones to low-resource IoT devices (e.g., surveillance cameras, sensors, etc.). These devices produce and store the application data used to train a shared ML model. FL is deployed by large service providers such as Apple, Google, and Facebook to train computer vision (CV) and natural language processing (NLP) models in applications such as image classification, object detection, and recommendation systems~\cite{tff,yang2018applied,FAI,opacus,applefl,googleVCFL,hartmann2019federated}. FL has also been deployed to train models on distributed medical imaging data~\cite{nvidiafl}, and smart camera images~\cite{cameraai}.%

The life-cycle of FL training is as follows. First, the FL operator builds the model architecture and determines hyper-parameters with a standalone dataset. The model's training is then conducted on participating learners for a number of centrally managed rounds until satisfactory model quality is obtained. The main challenge in FL is the heterogeneity in terms of computational capability and data distribution among a large number of learners which can impact the performance of training~\cite{Bonawitz19,kairouz2019advances}. 

Time-to-accuracy is a crucial performance metric and is the focus of much work in this area~\cite{kairouz2019advances,Li2020FedProx,yang2020heterogeneityaware,Safa-Wu2021,Oort-osdi21}. %
It depends on both the statistical efficiency and system efficiency of training.
The number of learners, minibatch size, local steps affect the former. It is common for these factors to be treated as hyper-parameters to be tuned for a particular FL job.
System efficiency is primarily regulated by the time to complete a training round, which  depends on which learners are selected and whether they become stragglers whose updates do not complete in time. It is common to configure a reporting deadline to cap the round duration, but if only an insufficient number of learners complete within this deadline, the entire round fails and is re-attempted from scratch.
Since a tight deadline can yield more failed rounds, this can be mitigated by overcommitting the number of selected learners in each round to increase the likelihood that a sufficient number will finish by the deadline. Failed rounds and overcommitted participants lead to wasted computation, which has mostly been ignored in previous FL approaches.
A focus on time has also resulted in schemes that are not robust to non-I.I.D. data distributions as they favor certain learner profiles~\cite{Li2021a}.
Finally, learners also have varying availability for training ~\cite{kairouz2019advances,Bonawitz19,Li2021a,yang2020heterogeneityaware}, which requires consideration when dealing with data heterogeneity 

All the above factors can lead to resource wastage---where learners perform training work that does not contribute to enhancing the model, whether due to updates that are ultimately discarded, or poor data distribution. We argue that this resource wastage deters users from participating in FL and makes the scaling of FL systems to larger deployments and more varied computational capabilities of learners problematic. We aim to optimize the design of FL systems for their resource-to-accuracy in a heterogeneous setting. This means the computational resources consumed to reach a target accuracy is reduced without a significant impact on time-to-accuracy. By considering heterogeneity at the heart of our design, we also intend to demonstrate improved robustness to realistic data distributions among learners.

Existing efforts aim to improve convergence speed (i.e., boosting model quality in fewer rounds)~\cite{Li2020FedProx,wang2019adaptive} or system efficiency (i.e., reducing round duration)~\cite{mcmahan2017,McMahan2018}, or selecting learners with high statistical and system utility~\cite{Oort-osdi21}. These approaches ignore the importance of maximizing the utilization of available resources while reducing the amount of wasted work. To address these problems, we introduce resource-efficient federated learning (\scheme), a practical scheme that maximizes FL systems' resource efficiency without compromising the statistical and system efficiency. \scheme accomplishes this by decoupling the collection of participant updates from aggregation into an updated model. \scheme also intelligently selects among available participants that are least likely to be available in the future. To the best of our knowledge, this is the first approach to directly account for predicted availability as the basis for participant selection in FL and to demonstrate its importance in the overall process.  %
\scheme  can be integrated as a plug-in module to existing FL systems~\cite{Bonawitz19,Li2020FedProx,lai2021fedscale,Oort-osdi21} and is compatible with existing FL privacy-preservation techniques~\cite{Keith2018,Bonawitz19secure}. 

In summary, we make the following contributions: 
\begin{enumerate}
\item We highlight the importance of resource usage of the learners' limited capability and availability in FL and present \scheme to intelligently select participants and efficiently make use of their resources.
\item We propose staleness-aware aggregation and intelligent participant selection algorithms to improve resource usage with minimal impact on time-to-accuracy. %
\item We implement and evaluate \scheme using real-world FL benchmarks and compare it with state-of-the-art solutions to show the benefits it brings to FL systems.
\end{enumerate}

This work does not raise any ethical issues. \scheme is released as open source at \url{https://github.com/ahmedcs/REFL}.

\section{Background}
\label{sec:motive}

We review the FL ecosystem with a focus on system design considerations, highlighting major challenges based on empirical evidence from real datasets. We motivate our work by highlighting the main drawbacks of existing designs.

\subsection{Federated Learning}

We consider the popular FL setting introduced in federated averaging (FedAvg)~\cite{mcmahan2017,Bonawitz19}.
The FedAvg model consists of a (logically) centralized server and distributed learners, such as smartphones or IoT devices.
Learners locally maintain private data and collaboratively train a joint global model.
A key assumption in FL is a lack of trust, implying training data should not leave the data source~\cite{Keith2018}, and any possible breach of private data during communication should be avoided~\cite{Bonawitz19secure}.

\begin{figure}[!t]
    \centering
    \includegraphics[width=1\linewidth]{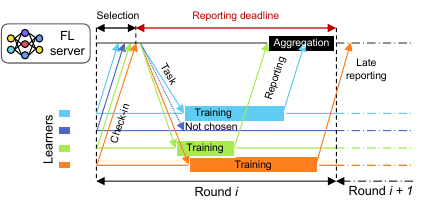}
    \caption{A round of training in the reference FL setting. A sample of learners, called participants, perform training in a given round. Only the updates received within the reporting deadline are aggregated and used to update the model prior to the next round.}
    \label{fig:federatedlearn}
\end{figure}

The training of the global model is conducted over a series of rounds. As shown in \cref{fig:federatedlearn}, at the beginning of each training round, the server waits (during a selection window) for a sufficient number of available learners to check-in.\footnote{A learner is available if it meets certain participation conditions: typically, being connected to power, being idle, and using an unmetered network~\cite{Bonawitz19}.} Then, the server samples a subset of the checked-in learners -- called participants -- to train in the current round. The participants fetch the latest version of the model along with any necessary configurations (e.g., hyper-parameter settings).

Each participant trains the model on its local data for a specified number of epochs and produces a model update (i.e., the delta from the global model) which it sends to the server. The server waits until a target number of participants send their updates to aggregate them and update the global model. This concludes the current round and the former steps are repeated in each round until a certain objective is met (e.g., target model quality or training budget).

To ensure progress, the server generally waits for model updates until a reporting deadline. Updates from stragglers that may arrive beyond the deadline are discarded. A round is considered successful if at least a target number of participants' updates are received by the deadline, else the round is aborted and a new one is attempted.

The FL setting is also distinct from conventional ML training because the distributed learners may exhibit the following types of heterogeneity: 
\begin{inparaenum}[1)]
    \item \textit{\textbf{data heterogeneity:}} learners generally possess variable data points in number, type, and distribution;
    \item \textit{\textbf{device heterogeneity:}} learner devices have different speeds owing to different hardware and network capabilities;
    \item \textit{\textbf{behavioral heterogeneity:}} the availability of learners varies across rounds and there may be learners that abandon the current round if they become unavailable.
\end{inparaenum}

Heterogeneity creates several challenges for FL system designers because both the quality of the trained model and the training speed are majorly affected by which participants are selected at each round.
Below we briefly review existing designs that serve as context to motivate our distinct approach (\S\ref{sec:motivation}).
We discuss additional related work in \S\ref{sec:related}.

\subsection{Existing FL Systems}
Accounting for the unreliability of learners, SAFA~\cite{Safa-Wu2021} enables semi-asynchronous updates from straggler participants. 

SAFA flips the participant selection process of FedAvg: it runs training on \emph{all} learners and ends a round when a pre-set percentage of them return their updates. SAFA allows participants to report after the round deadline, in which case the updates are cached and applied in a later round. However, SAFA only tolerates updates from learners that are within a bounded staleness threshold.
Therefore, the round duration in SAFA is reduced by only waiting for a fraction of the participants, while the cache ensures that the computational effort of straggling participants is not entirely wasted and is able to boost statistical efficiency.
FLeet~\cite{fleet} enables stale updates but adopts a dampening factor to give smaller weight as staleness increases. This is beneficial for not discarding updates that exceed the staleness threshold. However, their AdaSGD protocol is not directly compatible with the traditional FL settings such as FedAvg and FLeet synchronizes model gradients after every local mini-batch.

Oort~\cite{Oort-osdi21} uses a participant selection algorithm that favors learners with higher utility. The utility of a learner in Oort is comprised of statistical and system utility. The statistical utility is measured using training loss as a proxy while system utility is measured as a function of completion time.
Oort preferentially selects fast learners to reduce the round duration. At the same time, it uses a pacer algorithm that can trade longer round duration to include unexplored (or slow) learners when required for statistical efficiency.

\section{The Case for Resource-Efficient FL}
\label{sec:motivation}

We motivate \scheme by highlighting the trade-offs between system efficiency and resource diversity as conflicting optimization goals in FL. Navigating the extremes of these two objectives, as exhibited by the SOTA FL systems, we show how they fall short on common FL benchmarks.

\subsection{System Efficiency vs. Resource Diversity}

Current FL designs either aim at reducing the time-to-accuracy (i.e., system efficiency)~\cite{Oort-osdi21} or increasing coverage of the pool of learners to enhance the data distributions and fairly spread the training workload (i.e., resource diversity)~\cite{Xie2019,Li2020FedProx,Safa-Wu2021}, but do not consider the cost of wasted work by learners.
The first goal results in a discriminatory approach towards certain categories of learners, either preferentially selecting computationally fast learners or learners with model updates of high quality (i.e., those with high statistical utility)~\cite{Li2020FedProx,Oort-osdi21}. The second goal entails spreading out the computations ideally over all available learners but at the cost of potentially longer round duration~\cite{Xie2019,Safa-Wu2021} and significant wasted work.

These two conflicting goals present a challenging trade-off for designers of FL systems to navigate. On one side of the extreme, Oort aggressively optimizes system efficiency and ignores the diversity of learners' data in order to improve time-to-accuracy. The implication of this extreme is less robustness to high levels of data heterogeneity due to poor selection fairness, potentially producing a global model that does not cover the majority of learners' data. On the other hand, SAFA foregoes pre-training selection, selecting all available learners to maximize resource diversity, at the cost of significantly increased resource wastage.

To strike a balance between the two extremes, the FL system should achieve a sufficient level of resource diversity without sacrificing significantly in terms of system efficiency. Our goal is to synthesize the opportunities presented by existing systems and devise a new holistic approach that can fulfill the resource diversity and system efficiency goals simultaneously while considering cumulative resource usage as a primary metric.\footnote{As a concrete metric, we use the time units of resource usage (i.e., for compute resources, the time spent to perform on-device training and for communication resources, the time to communicate with the server) accumulated at every participant. This metric is proportional to energy consumption but affords us avoiding fine-grained power measurements, which are difficult to accurately account for in simulation at scale.} We first show that the existing systems fail to achieve both of these goals and result in significant wastage of resources. We also highlight the opportunities they present which we embrace in our design of \scheme.

\subsection{Stale Updates \& Resource Wastage}

\begin{figure}[!t]
\captionsetup[subfigure]{justification=centering}
\centering
    \includegraphics[width=.8\columnwidth]{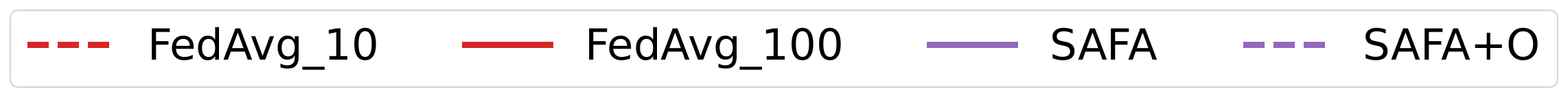}
    \\
     \includegraphics[width=0.95\columnwidth]{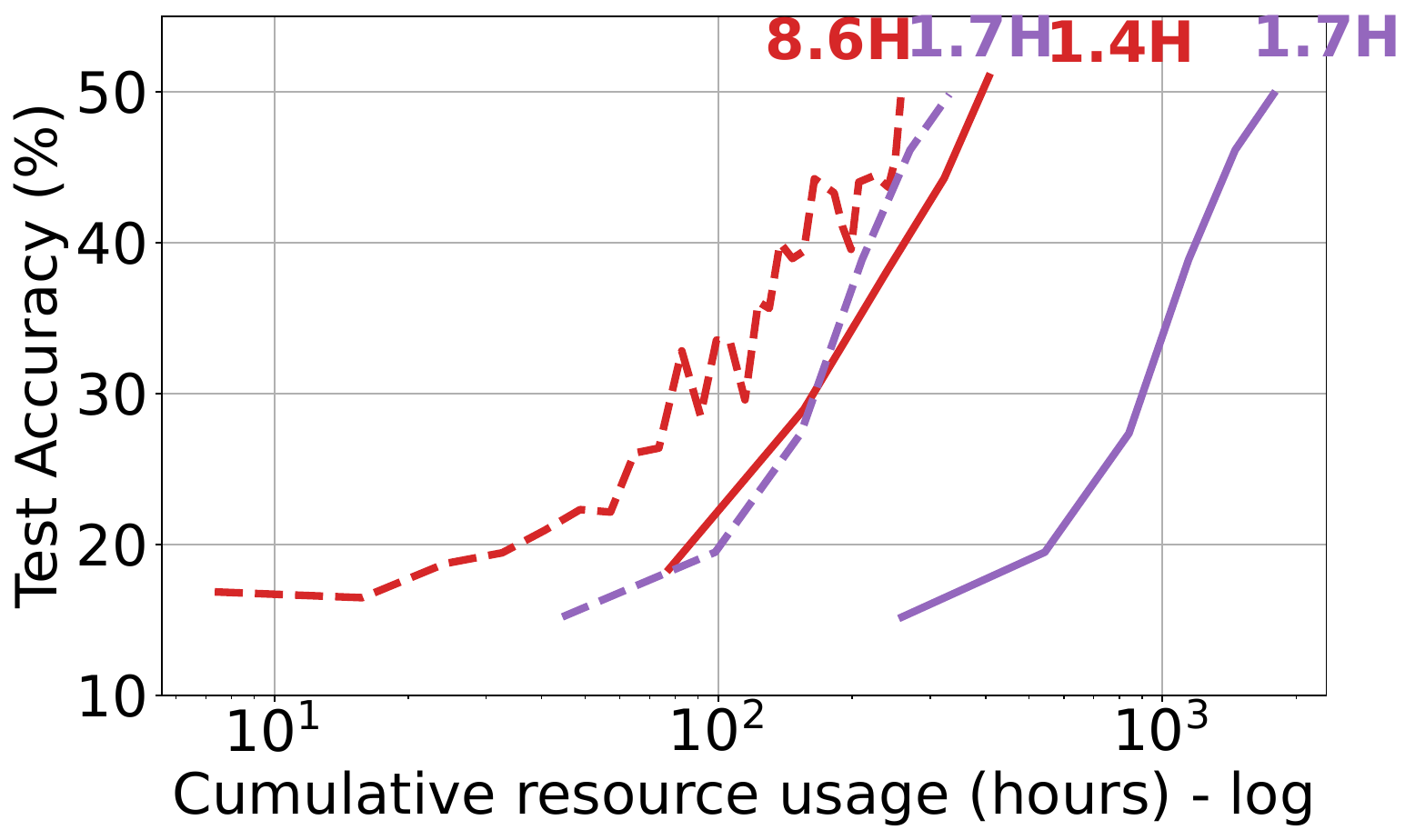}
\caption{Resource usage comparison of SAFA versus an ideal resource-optimized version denoted as SAFA+O, and FedAvg with 10 or 100 participants. We train the model up to a target accuracy. The run time of each approach is shown as a colored annotation near the final data point.}
\label{fig:safa-motive}
\end{figure}

Taking inspiration from asynchronous methods~\cite{Qirong2013,Xie2019}, SAFA allows straggling participants to contribute to the global model via stale updates.
We first evaluate SAFA's resource usage (i.e., the time cumulatively spent by learners in training), and resource wastage (i.e., the time cumulatively spent by learners producing updates that are \emph{not} incorporated into the model).
We compare the performance of SAFA as described in~\cite{Safa-Wu2021} against a version (called SAFA+O) that assumes a perfect oracle that knows which stale updates are eventually aggregated (i.e., will not exceed the staleness threshold). We set the staleness threshold to 5 rounds and the target participant percentage to $10\%$. We use an audio dataset of spoken words provided by Google, hereafter referred to as the Google Speech benchmark~\cite{googlespeech}, and use FedScale's data-to-learner mappings~\cite{lai2021fedscale} (c.f. \cref{tab:benchmarks} and \cref{sec:eval} for details). We set the total number of learners to 1,000, and the round deadline to 100s. We use a real-world user behavior trace to induce learner's availability dynamics~\cite{yang2020heterogeneityaware}. 

\cref{fig:safa-motive} shows the resource usage (x-axis) and resulting test accuracy (y-axis); the lines are annotated with the time to achieve the final accuracy (this style is repeated in other figures in this paper). Since the round time is bounded by a deadline, both SAFA and SAFA+O have equal run times.
Notably, SAFA is inefficient in terms of resource usage, consuming nearly 5$\times$ the resources of SAFA+O to achieve the same final accuracy. By selecting all available learners, then eventually discarding a large number of the computed updates, SAFA wastes around 80\% of learners' computation time. The plot also includes runs of FedAvg with Random selection of 10 and 100 participants. Despite having low resource wastage, FedAvg with 10 participants incurs significantly higher run time (5$\times$) to reach the same accuracy of SAFA; the resource usage could be traded for lower run time with 100 participants, to achieve the same accuracy at similar resource usage to SAFA+O. We note that uniform random data mapping yields similar results.

\smartparagraph{Opportunity.}
In principle, allowing stale updates enables a reduction in round duration and achieves better time-to-accuracy while preserving stragglers' contribution.
The main challenge is, however, to balance the number of participants to avoid significant resource wastage.
This is difficult because the system must estimate the on-device training time and reason about the probability of learners dropping out or exceeding the staleness threshold.
This suggests that beyond stale updates, we must also tackle resource diversity directly.

\begin{figure}[!t]
\captionsetup[subfigure]{justification=centering}
\centering
     \begin{subfigure}[ht]{0.49\linewidth}
     \includegraphics[width=\linewidth]{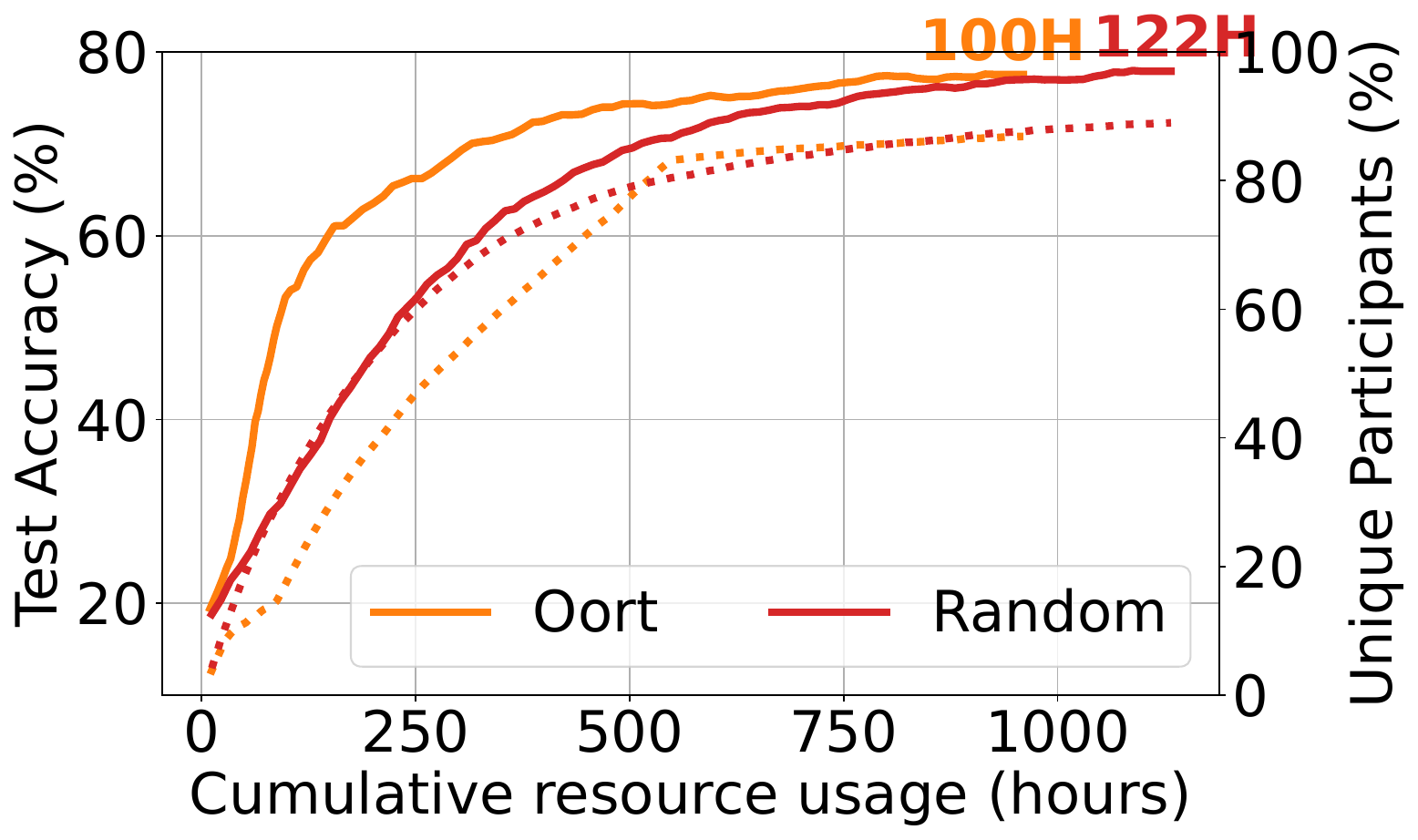}
	\caption{FedScale data mapping}
	\label{fig:motive1-uniform-random}
     \end{subfigure}
     \hfill
    \begin{subfigure}[ht]{0.49\linewidth}
     \includegraphics[width=\linewidth]{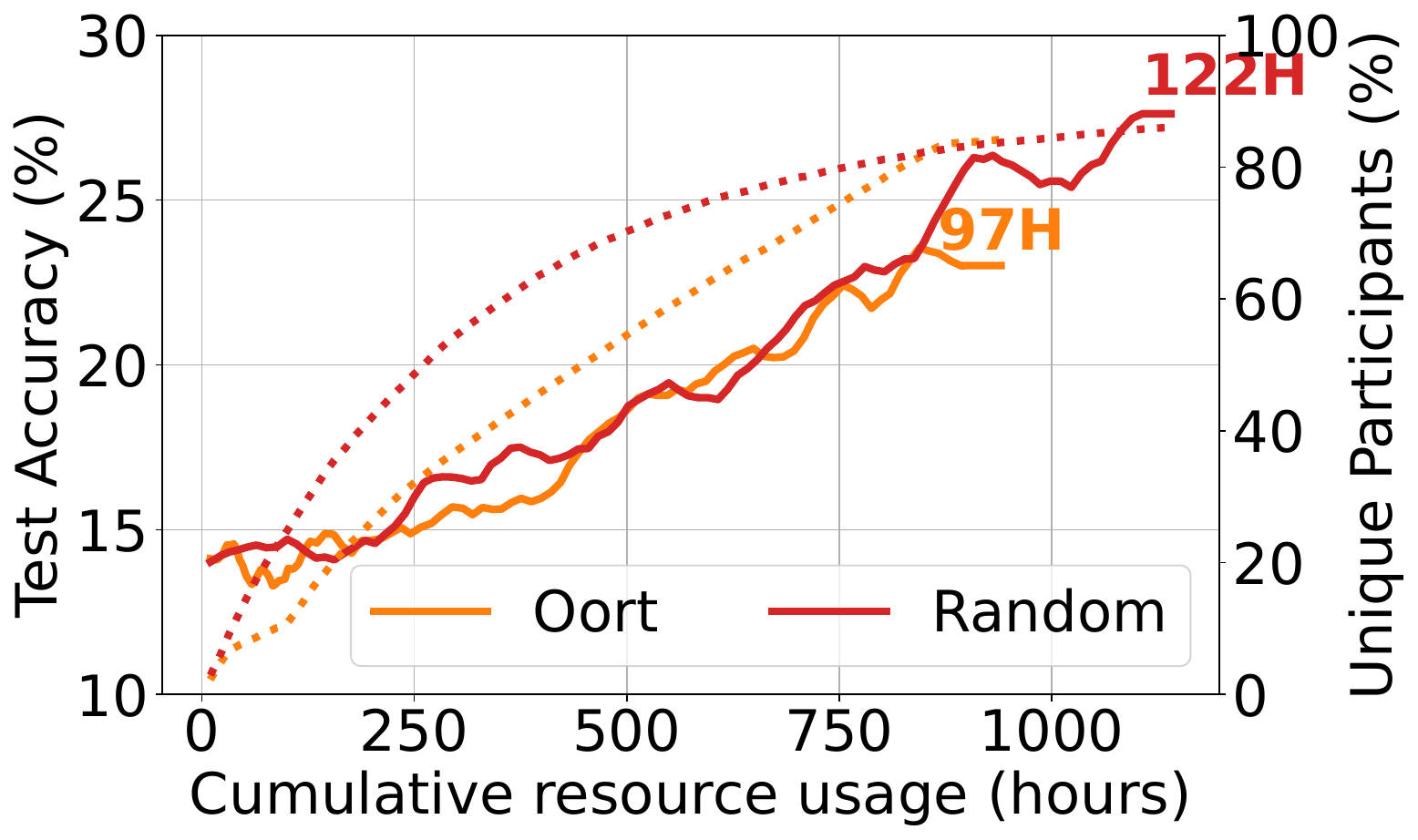} 
	\caption{Label-limited (non-IID)}
	\label{fig:motive1-label-limit}
     \end{subfigure}
\caption{Impact of data heterogeneity on test accuracy in two data mappings. The right y-axes and dotted lines indicate the percentage of unique participants during training.}
\label{fig:motive1}
\end{figure}

\subsection{Participant Selection \& Resource Diversity}

Many existing FL systems select participants using a uniform random sampler~\cite{caldas2018leaf, yang2018applied, Bonawitz19}. 
As noted in Oort~\cite{Oort-osdi21}, this simple strategy is prone to select learners with disparate computing capabilities and prolong round duration due to stragglers. On the other hand, Oort's approach of selecting fast learners has unfavorable consequences, by biasing the model to a subset of the learners that can reduce data diversity.

To see this in practice, we compare the Oort participant selector with a random sampler (Random). We use the Google Speech benchmark for 1,000 training rounds and compare two cases with different data mappings. In the first case, data points are mapped to the learners using FedScale's client-to-data mappings~\cite{lai2021fedscale}.\footnote{In~\cref{sec:eval}, we show that FedScale's client-to-data mapping is comparable to that of an Independent and Identically Distributed (IID) data.} In the second case, data points are also uniformly distributed among the participants but each participant is constrained to have $\approx$10\% of all labels (non-IID). To emphasize the effect of the sampling strategy, we set all learners to be always available; we investigate the effects of availability dynamics later.

\cref{fig:motive1} shows the resulting test accuracy against resource usage. 
In the FedScale's data mapping scenario, Oort is clearly superior to random selection as Oort significantly reduces the round duration by exploiting fast learners.
Conversely, in the non-IID case (label-limited mapping), random selection achieves higher accuracy with a tolerable increase in run time due to a higher resource (and data) diversity.

Participant availability impacts the global data distribution represented in the global model~\cite{Huang2021}.
Our analysis of a large-scale device behavior trace from \cite{yang2020heterogeneityaware} involving more than 136K users of an FL application over a week reveals that 70\% of the learners are available for at most 10 minutes while 50\% are available for at most 5 minutes. This means in practice FL rounds should typically last a few minutes to obtain updates from the majority of participants.
The analysis also suggests that low availability learners may require special consideration to increase the number of unique participants without adversely impacting overall training time.

 \begin{figure}[!t]
\captionsetup[subfigure]{justification=centering}
\centering
    \includegraphics[width=1\linewidth]{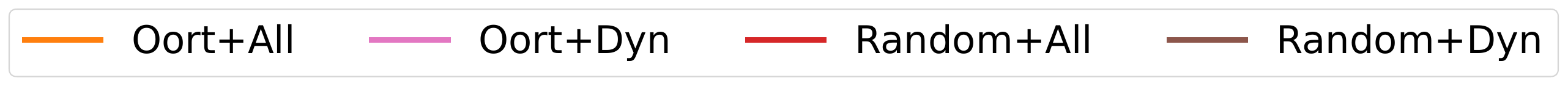}
    \\
    \begin{subfigure}[ht]{0.49\linewidth}
     \includegraphics[width=\linewidth]{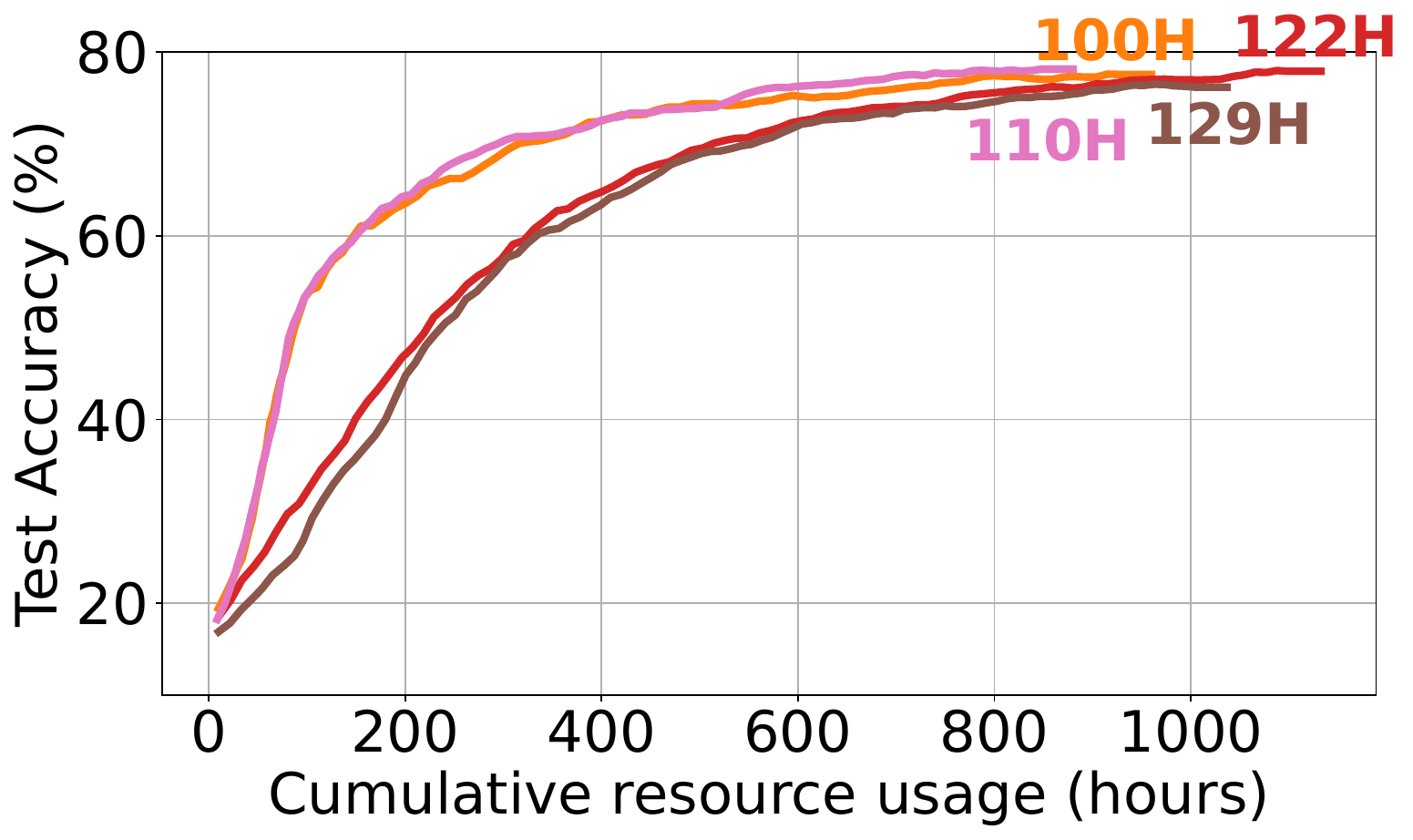}
	\caption{FedScale data mapping}
	\label{fig:uniform-random}
     \end{subfigure}
     \hfill
    \begin{subfigure}[ht]{0.49\linewidth}
     \includegraphics[width=\linewidth]{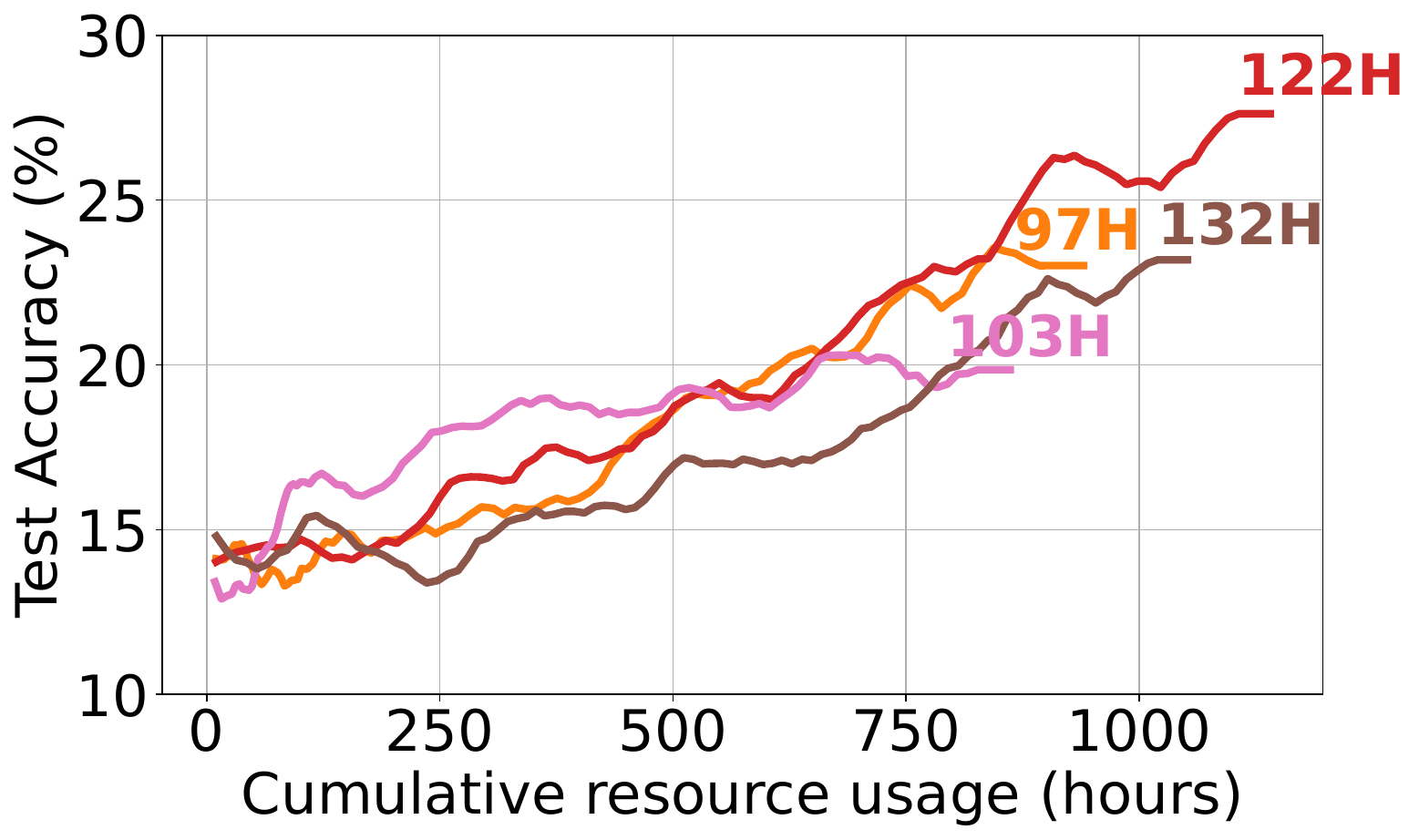} 
	\caption{Label-limited (non-IID)}
	\label{fig:label-limit}
     \end{subfigure}
\caption{Impact of availability on test accuracy in two data mappings.}
\label{fig:motive2}
\end{figure}

We now repeat similar experiments on the FedScale and non-IID cases of the Google Speech benchmark and contrast the execution of Oort and Random participant selection methods in two conditions:
\begin{inparaenum}[1)]
\item all learners are available (AllAvail);
\item learners' availability is dynamic based on the trace of device behavior (DynAvail).
\end{inparaenum}
\cref{fig:motive2} shows that learner availability has no tangible impact in the FedScale case since learners hold data points with comparable distributions.
However, in the non-IID case, learners' availability has a significant effect on model accuracy (we observe a 10-point drop).

\smartparagraph{Opportunity.}
To achieve better model generalization performance, the model should be trained jointly on data samples from a large fraction of the learner population. While Oort's insights into informed participant selection result in faster round duration, there needs to be more consideration with regard to the dynamic availability of learners to ensure wider learner coverage.
This suggests that beyond learners' diversity and compute capabilities, we need to effectively prioritize learners whose availability is limited.

\begin{figure}[!t]
    \centering
    \includegraphics[width=0.49\textwidth]{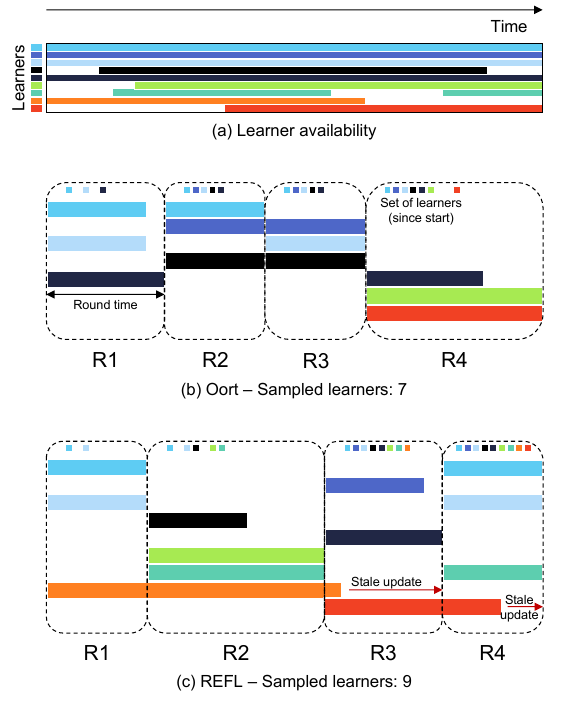}
    \caption{Example trace of 4 training rounds illustrating the main differences between Oort (b) and \scheme (c). The dynamic availability of 9 color-coded learners is shown in (a). By optimizing for time-to-accuracy, Oort skews participant selection towards faster learners during the early phases of training, thereby missing limited-availability learners (\mlA and \mlB). Oort's round time is determined by stragglers. By allowing stale updates, \scheme lowers the dependency on stragglers. By prioritizing learners based on estimated availability, \scheme samples a more diverse set of learners.}
    \label{fig:time-device}
\end{figure}

\section{\scheme Design}
\label{sec:scheme}
\scheme's objective is to enhance the resource efficiency of the FL training process by maximizing resource diversity without sacrificing system efficiency. \scheme achieves this by reducing resource wastage from delayed participants and prioritizing those with reduced availability. It  leverages a theoretically-backed method to incorporate stale updates based on their quality which helps improve the training performance. It proposes a scaling rule for aggregation weights to mitigate stale updates' impact.

The two core components of \scheme are:
\begin{enumerate}
    \item \emph{\bf Intelligent Participant Selection (IPS):} to prioritize participants that improve resource diversity.
    \item \emph{\bf Staleness-Aware Aggregation (SAA):} to improve resource efficiency without impacting time-to-accuracy.
\end{enumerate}

\smartparagraph{Overview by example.}
To illustrate the main differences, \cref{fig:time-device} contrasts \scheme with Oort. First, \scheme enables learners'  tracking of the availability patterns which help  with predicting the future availability. Therefore \scheme is able to prioritize the least available participants (i.e., \mlA and \mlB in \cref{fig:time-device}) to maximize training coverage of different learners' data distributions. \scheme also allows straggling participants to submit late results beyond the set round duration (i.e., \mlA and \mlC in \cref{fig:time-device}). Unlike Oort, which discards these updates due to their inferior device capabilities, this approach reduces the wasted work of stragglers who might own valuable data for the model to be trained on.

\subsection{Intelligent Participant Selection (IPS)}
\label{subsec:ips}
IPS increases resource diversity to allow the global model to capture a wide distribution of learners' data. Moreover, it provides an optional component to further reduce resource wastage by intelligently adapting the number of participants in every round. 

\begin{algorithm}[!t]
    \SetAlgoLined
    \SetKwInOut{Input}{Input}
    \SetKwInOut{Output}{Output}
    \DontPrintSemicolon
    \SetInd{0.5em}{0.75em}
    \Input{$N_t$-Target number of participants}
    \Output{$S$-List of selected participants}
    Initialize $S_t=\emptyset$, $P_t=\emptyset$, $a=(\mu_t, 2\mu_t)$;\\
    \Event{Learner\_Check\_In}
    {
      Send slot $a$ to learner $l$;\\
      Receive learner $l$'s availability probability $p_l$;\\
      $P_t = P_t \cup p_{l}$;\\
    }
    \Event{End\_Selection\_Window}
    {
      Sort in ascending order $P_t$;\\
      Randomly shuffle $P_t$ for probabilities with ties;\\
      Return $S_t$ as the top $N_t$ learners in $P_t$;\\
    }

\caption{Priority Selection Algorithm}\label{algo:priority}
\end{algorithm}

\smartparagraph{Least available prioritization:} \cref{algo:priority} describes how the IPS component intelligently selects participants from the large pool of available learners. Each learner periodically trains a model that predicts its future availability.%
 Upon check-in of the learner $l$, the server sends the running average estimate of round duration $\mu_t$. The learner uses the prediction model to determine the probability of its availability in the time slot $[\mu_t, 2\mu_t]$ and reports this to the server. At the end of the selection window, the server sorts, in ascending order, the learners' probabilities $P$ and randomly shuffles tied learners. Then, the server selects the top $N_t$ learners to participate in this round (i.e., the least available learners). Similar to Google's FL system~\cite{Bonawitz19}, the participants hold from checking-in with the server for few rounds (e.g., 5 rounds) after submitting the updates. 
 
\smartparagraph{Availability prediction model:} A prediction model should be simple with low overhead and trained locally on the learners' devices to preserve privacy. In this work, we do not propose new availability models and use off-the-shelf time-series models to predict the future availability of the learners. Linear models such as Auto-Regressive Integrated Moving Average (ARIMA) or Smoothed ARIMA can be trained on a minimal set of features collected from on-device events of change in the state such as idle, charging, connection to WIFI, screen locked, etc~\cite{HYNDMAN2002439,forecastingbook}.\footnote{Several works used mobile traces to learn user patterns~\cite{prediction1,prediction2,stunner}.}
We use the Prophet forecasting tool~\cite{prophet}, which is based on the aforementioned linear models. We train a prediction model on the Stunner dataset, which is a large-scale dataset comprising device events from a large number of mobile users (e.g., the charging state of the devices)~\cite{stunner}. Given a time window in the future, the model produces a probability for the charging state (eq. availability) of the device within the queried time window.\footnote{%
To address any privacy concerns, the server query is on a limited time window in the future and the server has no access to the history of the device's state. Moreover, the learner may choose not to share this information in which case the server assumes that it is available in the queried time window.} In \S\ref{sec:eval}, we show that the trained model can provide availability predictions with high accuracy.

\smartparagraph{Adaptive Participant Target (APT):} IPS can optimize resource usage by adapting the pre-set target number of participants $N_0$ selected by the operator. %
First, the server updates its moving average estimate of round duration $\mu_{t} = (1-\alpha) D_{t-1} + \alpha \mu_{t-1}$, where $D_{t-1}$ is the duration of the previous round $t-1$. Then, before commencing round $t$, the server probes each current straggler $s \in L_s$ (from round $t-1$) for an estimate of its expected remaining time to upload the update $RT_s$. Next, the server computes how many stragglers ($B_t$) can complete within the duration of the current round (i.e., $RT_s \le \mu_t$). And so, the target number of participants is adjusted for round $t$ to $N_t=max(1, N_0 - B_t)$. This ensures that, in each round, a roughly constant number of updates $N_0$ is aggregated (i.e., the total fresh and stale updates). In large-scale scenarios, this could potentially further improve resource consumption. Note, irrespective of clients' availability, APT is an add-on scheme to not over-commit the participants, further reducing resource consumption.

\subsection{Staleness-Aware Aggregation (SAA)}
\label{subsec:saa}

This component enables the participants to submit their updates past the round deadline and processes these stale updates along-with the fresh updates. Stale updates can be noisy since the model can drift significantly by the time a stale update arrives. In order to mitigate this impact, we multiply the stale updates based on a boosting factor (\S\ref{subsec:scale}). 

We also provide a convergence analysis to substantiate the benefits of staleness. We analyze this effect independently of other \scheme components because all the \scheme components complement each other. Since FedAvg \cite{mcmahan2017,Bonawitz19} is one of the most prominent FL algorithms, we analyze (\S\ref{subsec:asfv}) FedAvg with staleness (termed Stale Synchronous FedAvg, c.f. \cref{algo:ssfedavg}). Under standard assumptions, our convergence analysis shows that the error due to staleness is small in each round, and hence the gradient does not differ significantly (see Lemma \ref{lem:asff}). Due to this small error per round, we show in Theorem \ref{thm:big_oh} that Stale Synchronous FedAvg converges at the same asymptotic rate as FedAvg. 

\subsubsection{Convergence Analysis}\label{sec:convergence_analysis}
We theoretically demonstrate that FedAvg with stale updates can  converge and obtain the convergence rate. Consider the following federated optimization problem consisting of a total of $m$ devices:
\begin{equation}\label{eq:opt}
 \textstyle    \min_{x\in\R^d}  f(x) \defeq \frac{1}{m}\sum_{j\in[m]}f_j(x) ,
\end{equation}
where $f_i(x)=\mbE_{z_i \sim \cD}{l(x;z_i)}$ such that $l(x;z_i)$ denotes the loss function evaluated on input $z_i$ sampled from $\cD$. 

Algorithm \ref{algo:ssfedavg} gives the pseudo-code of Stale Synchronous FedAvg to solve \eqref{eq:opt}. Here, $g_{t,k}^i$ denotes the stochastic gradient computed at the $i^{th}$ participant at round $t$, and at local iteration $k$, such that $\textstyle g_{t,k}^i = \nabla f(y_{t,k}^i) + \xi_{t,k}^i$ with $\mbE [\xi_{t,k}^i| x_{t,k}^i] = \mathbf{0}$.

\begin{algorithm}[!t]
    \SetKwInOut{Input}{Input}
    \SetKwFor{ForP}{for}{in parallel do}
    \DontPrintSemicolon
    \SetInd{0.5em}{0.75em}
    \Input{$K$-synchronization interval, $\tau$-delay in rounds, $N_t$-number of participants}
    Initialize $x_{0}=x_{1}=\hdots=x_{\tau-1}\in \R^d$\;
    \For{round $t=0,\hdots,T-1$}{
    The server samples $S_t$ learners with $|S_t|=N_t$\;
    \ForP{participant $i \in [n]$}{
    $y_{t,0}^i = x_t$\;
    \For{iteration $k=0, \hdots,K-1$}{
    Compute a stochastic gradient $g_{t,k}^i$\;
    $y_{t,k+1}^i = y_{t,k}^i - \gamma g_{t,k}^i$   \tcp*{Local participant update}
    }
    Let $\Delta_t^i = y_{t,K}^i-y_{t,0}^i =-\gamma \sum_{k=0}^{K-1}g_{t,k}^i$\;
    Send $\Delta_t^i$ to the server\;
    }
  	\textbf{At Server}: \\
  	\If{$t<\tau$}{
	   	Broadcast $x_{t+1}$ to participants \tcp*{Aggregation starts $t=\tau$}
  	}
  	\Else{
  	 Receive $\Delta_{t-\tau}^i, i \in S$ \tcp*{Update arrives with delay $\tau$}
	 Let $\Delta_{t-\tau} = \frac{1}{|S|} \sum_{i\in S}\Delta_{t-\tau}^i$\;\
	Server update: $x_{t+1}=x_t+\gamma\Delta_{t-\tau}$\;
	Broadcast $x_{t+1}$ to participants\;
  	}
	} 
\caption{Stale Synchronous FedAvg}
\label{algo:ssfedavg}   
\end{algorithm}

The staleness of model updates is modeled as a round delay.
For ease of exposition, we consider a fixed $\tau$ round delay in \cref{algo:ssfedavg}. However, our analysis holds for variable delays bounded by $\tau$.

\smartparagraph{Assumptions:} we consider the following general assumptions on the loss function.
\begin{assumption}\label{ass:smoothness}
\textbf{(Smoothness)} The function, $f_i: \R^d\to \R$ at each participant, $i \in [n]$ is $L$-smooth, i.e., for every $x,y\in\R^d$ we have, $f_i(y)\leq f_i(x)+\dotprod{\nabla f_i(x), y-x}+\frac{L}{2}\norm{y-x}^2$. 
\end{assumption}
\begin{assumption}\label{ass:minimum}
\textbf{(Global minimum)} 
There exists $x_\star$ such that, $f(x_\star)=f^\star\leq f(x)$, for all $x\in\R^d.$  
\end{assumption}
\begin{assumption}\label{ass:m-sigma2-bounded_noise}
\textbf{(($M, \sigma^2$) bounded noise)} \cite{stich2019error} For every stochastic noise $\xi_{t,k}^i$, there exist $M\geq 0, \sigma^2>0$, such that 
$\mbE [\norm{\xi_{t,k}^i}^2 \mid x_{t}] \leq M\norm{\nabla f(x_{t,k}^i)}^2+\sigma^2$, for all $x_t\in\R^d$.
\end{assumption}

\subsubsection{Convergence Result}\label{subsec:asfv}
The next theorem provides the non-convex convergence rate.%
\begin{theorem}\label{thm:big_oh}
Let Assumptions \ref{ass:smoothness}, \ref{ass:minimum}, and \ref{ass:m-sigma2-bounded_noise} hold. Then, for \cref{algo:ssfedavg}, we have
\begin{eqnarray*}
&\frac{1}{nTK}\sum_{t=0}^{T-1}\sum_{i=1}^n\sum_{k=0}^{K-1}\mbE\norm{ \nabla f(y_{t,k}^i)}^2=\\
& \textstyle \cO\left(\frac{\sigma\sqrt{L(f(x_0-f^\star))}}{\sqrt{nTK}}
+\frac{\max\{L\sqrt{\tau K(n\tau K+M)}, L(K+M/n)\}}{TK}\right).
\end{eqnarray*}
\end{theorem}
\smartparagraph{Overview of analysis.} Our analysis builds on top of the Error Feedback framework \cite{stich2019error}, which is in turn inspired by Perturbed Iterate Analysis \cite{mania2017perturbed}. 
We first define the update of Stale Synchronous FedAvg:
\begin{equation}\label{eq:v_t}
v_t =
\begin{cases}
0, &\text{if } t < \tau\\
\frac{1}{n}\sum_{i=1}^n\sum_{k=0}^{K-1}g_{t-\tau,k}^i, &\text{otherwise.}
\end{cases}
\end{equation}
Using this definition of $v_t$, we have
\begin{equation}
x_{t+1} = x_t - v_t \quad \forall t.
\end{equation}
Let us also define the \textit{error} $e_t$ due to asynchrony as
\begin{equation}
e_t =\sum_{j=1}^{\tau}\mathbb{1}_{(t-j)\geq0}\left( \frac{\gamma}{n}\sum_{i=1}^n\sum_{k=1}^K g_{t-j, k}^i\right).
\end{equation}
where $\mathbb{1}_Z$ denotes the indicator function of the set $Z$. 

The recurrence relation in the next lemma is instrumental for perturbed iterate analysis of Algorithm \ref{algo:ssfedavg}. 
\begin{lemma}\label{lem:virtual_iterate}
~Define the sequence of iterates $\{\Tilde{x}_{t}\}_{t \geq 0}$ as $\Tilde{x}_{ t}=x_{t}-\bar{e}_{t}$, with $\Tilde{x}_{0}=x_0$. Then $\{\Tilde{x}_{t}\}_{t \geq 0}$ satisfy the recurrence: $\Tilde{x}_{t+1}=\Tilde{x}_t - \frac{\gamma}{n} \sum_{i=1}^n\sum_{k=0}^{K-1}g_{t,k}^i$. 
\end{lemma}

Lemmas \ref{lem:asdd} and \ref{lem:asff} are useful for bounding intermediate terms.

\begin{lemma}\label{lem:asdd}
We have
\begin{equation*}\label{eq:acbk}
    \mbE_t \norm{\frac{1}{n}\sum_{i=1}^n \sum_{k=0}^{K-1}g_{t,k}^i}^2 \leq \frac{1}{n}\sum_{i=1}^n\sum_{k=0}^{K-1}(K+\frac{M}{n}) \norm{\nabla f(y_{t,k}^i)}^2 + \frac{K\sigma^2}{n},
\end{equation*}
where $\mbE_t[\cdot]$ denotes expectation conditioned on the iterate $x_t$, that is, $\mbE[\cdot|x_t]$.
\end{lemma}

\begin{lemma}\label{lem:asff}
With a constant step-size $\gamma \leq \frac{1}{2L\sqrt{\tau K(n\tau K+M)}}$, we have
\begin{eqnarray*}\label{eq:sdjh}
&\sum_{t=0}^{T-1}\frac{1}{n}\sum_{i=1}^n\sum_{k=0}^{K-1}\mbE\norm{\tilde{x}_t-y_{t,k}^i}^2 \\
&\textstyle \leq \frac{1}{4L^2}\sum_{t=0}^{T-1}\frac{1}{n}\sum_{i=1}^n\sum_{k=0}^{K-1}\mbE\norm{\nabla f(y_{t,k}^i)}^2 + \frac{\gamma^2}{n}T\tau K^2\sigma^2.
\end{eqnarray*}
\end{lemma}

\begin{lemma}\label{lem:without_order_notation}
Let Assumptions \ref{ass:smoothness}, \ref{ass:minimum} and \ref{ass:m-sigma2-bounded_noise} hold.~If $\{x_t\}_{t\geq0}$ denote the iterates of Algorithm \ref{algo:ssfedavg} for a constant step-size,\\ $\gamma \leq\min\{\frac{1}{2L\sqrt{\tau K(n\tau K+M)}}, \frac{n}{2L(nK+M)}\}$, then
\begin{eqnarray*}
&\frac{1}{nTK}\sum_{t=0}^{T-1}\sum_{i=1}^n\sum_{k=0}^{K-1}\mbE\norm{ \nabla f(y_{t,k}^i)}^2 \leq \\
& \textstyle \frac{8}{\gamma TK}\left(f(x_0)-f^{\star}\right) + \frac{4\gamma L \sigma^2}{n} +  \frac{4\gamma^2 L^2 \tau K\sigma^2}{n}.
\end{eqnarray*}
\end{lemma}

The above leads us to the stated theorem for an appropriate choice of parameters. %

\smartparagraph{Significance of results.}
We achieve a $\cO(\frac{1}{\sqrt{nTK}})$ asymptotic rate, which improves with $K$, the number of local steps per round. In comparison, the Asynchronous algorithm of \cite{Xie2019} does not improve with the number of local steps. Moreover, asynchrony (captured by $\tau$) only affects the faster decaying $\cO(\frac{1}{T})$ term. Thus, Stale Synchronous FedAvg has the same asymptotic convergence rate as synchronous FedAvg \cite{stich2019error}, and hence achieves asynchrony \textit{for free}. Furthermore, this asymptotic rate is better for Stale Synchronous FedAvg in practice, since the rate improves with $n$, the total number of learners contributing to an update, and staleness relaxation should result in more learners contributing to an update.

\subsubsection{Mitigating the Impact of Increased Staleness}
\label{subsec:scale}
We note that our convergence guarantee in \S\ref{sec:convergence_analysis} depends on the maximum round delay $\tau$, and large round delays could negatively impact convergence, which must be mitigated.
To mitigate the impact, prior work on distributed asynchronous training proposes to scale the weight of stale updates before aggregation~\cite{staleaware,Jiang2017,fleet}. 
We denote the set of fresh and stale updates in a round as $\cF$, and $\cS$ respectively. Let $n_{\cF}$ be the number of fresh updates, and 
$\hat{u}_{\cF}$ be the average of the fresh updates. Moreover, let $n_\cS$ be the number of stale updates, and for a straggler $s \in \cS$, $u_s$ be the stale update, and $\tau_s$ be the number of rounds the straggler is delayed by. The following are the scaling rules in the literature:
\begin{enumerate}
\item \textbf{Equal:}  same weight as fresh updates (i.e., $w_s = 1$);
\item \textbf{DynSGD:} linear inverse of the number of staleness rounds $w_s = \frac{1}{\tau_s + 1}$~\cite{Jiang2017};
\item \textbf{AdaSGD:} exponential damping of the number of staleness rounds $w_s = e^{-(\tau_s + 1)}$~\cite{fleet}.
\end{enumerate}
AdaSGD also proposed a boosting multiplier to increase the weights for stale updates to account for learners with data distributions that more significantly deviate from the global data distribution. AdaSGD showed that boosting the weight of important stale updates is critical since stragglers may possess more valuable (dissimilar) data compared to fast learners. However, this approach may violate privacy because computing the boosting factor requires learners to share information about their data.

Therefore, we propose a privacy-preserving boosting factor and combine it with the staleness-based damping rule of DynSGD~\cite{Jiang2017}. The proposed boosting factor favors a stale update based on how much it deviates from the fresh updates' average and hence it does not require any information about learner's data. Let, %
$\Lambda_s=\frac{\norm{\hat{u}_\cF-\frac{u_s+n_\cF \hat{u}_f}{n_\cF+1}}^2}{\norm{\hat{u}_{\cF}}^2}$ 
be the deviation of the stale update $u_s$ from the average of the fresh updates $\hat{u}_{\cF}$. Let $\Lambda_{max}=\max_{s\in\cS}{\Lambda_s}$. The boosting factor term scales a stale update $s$ proportional to $1 - e^{-\frac{\Lambda_s}{\Lambda_{max}}}$. Finally, our rule to compute the scaling factor is:
\begin{equation}
\label{eq:hybrid}
 \textstyle w_s = (1 - \beta) \frac{1}{\tau_s + 1} + \beta (1 - e^{-\frac{\Lambda_s}{\Lambda_{max}}})
\end{equation}
where $\beta$ is a tunable weight for the averaging. %

For every fresh update $f\in\cF$, we choose a scale value of one, i.e., $w_f=1$. The final coefficients for weighted averaging are the normalized weights. That is, for an update $i\in\cF\cup\cS$, the final coefficient as:
$$\hat{w}_i=\frac{w_i}{\sum_{i\in\cF\cup\cS}w_i}$$ %
Hence, in aggregation, $w_i < w_f $ meaning that weights applied to stale updates are strictly less than weights for new updates. This in principle reduces the impact from malicious learners who delay the updates to gain any advantage because of the boosting factor. Further analysis in adversarial settings is left to future work.

\begin{table*}[!t]
\caption{Summary of benchmarks and characteristics of the mappings used in this work.}
\centering
\resizebox{0.9\linewidth}{!}{%
    \begin{tabular}{|cccccccc|cc|ccc|}
    \toprule
         \multirow{3}{*}{Task} &  \multirow{3}{*}{Model} &	 \multirow{3}{*}{Dataset}	&  \multirow{3}{*}{\makecell{Model Size \\\# of Parameters}} &  \multirow{3}{*}{\makecell{Learning\\Rate}} & \multirow{3}{*}{\makecell{Local\\Epochs}} & \multirow{3}{*}{\makecell{Batch\\Size}} & \multirow{3}{*}{\makecell{\# of\\ Labels}} & \multicolumn{2}{c}{FedScale Mapping} & \multicolumn{3}{c}{Label-limited Mapping}\\
          \cmidrule{9-13}
          & & & & & & & & \multirow{2}{*}{\makecell{\# of \\samples}} & \multirow{2}{*}{\makecell{Total \\ Learners}} & \multirow{2}{*}{\makecell{\# of\\samples}} & \multirow{2}{*}{\makecell{Total \\ Learners}} & \multirow{2}{*}{\makecell{\# of\\ Labels}} \\
          & & & & & & & & & & & & 
        \\\midrule
        \multirow{2}{*}{\makecell{Image \\ Classification}} & ResNet18~\cite{resnet} \ & CIFAR10~\cite{cifar10}	& 11.4M &  0.01 & 1 & 10 & 10 & N/A & N/A & 50K & 3K & 4
        \\ 
        & ShuffleNet~\cite{shufflenet} & OpenImage~\cite{openimage} &	2.23M & 0.01 & 5 & 30 & 600 & 1.3M & 14K & 1.3M & 3K & 60
        \\ 
        \midrule
         \makecell{Speech \\ Recognition} &  ResNet34~\cite{resnet} & Google Speech~\cite{googlespeech} & 21.5M & 0.005 & 1 & 20 & 35 & 200K & 3K & 200K & 3K & 4
         \\\midrule
            \multirow{2}{*}{\makecell{Natural Language\\ Processing}} & \multirow{2}{*}{Albert~\cite{albert}}  & Reddit~\cite{REDDIT} & 11M & 0.0001 & 5 & 40 & N/A & 42M & 16M  & 8.4M & 3K & N/A
         \\
          &   & Stackoverflow~\cite{mcmahan2017} & 11M &  0.0008 & 5 & 40 & N/A & 43M & 300K & 8.6M & 3K & N/A
         \\
        \bottomrule
    \end{tabular}
    }
    \label{tab:benchmarks}
\end{table*}

\section{Evaluation}
\label{sec:eval}
Our evaluation addresses the following questions:
\begin{itemize}
    \item Is prioritizing learners based on availability beneficial?
    \item Does aggregation of stale updates reduce resource usage and improve model quality?
    \item Should stale updates be weighted differently to fresh updates?
    \item Is \scheme scalable and future-proof?
\end{itemize}

\noindent We summarize the main observations as follows:
\begin{itemize}
    \item \scheme achieves better model quality with up to 2$\times$ less resources compared to existing systems.
    \item \scheme results in quality gains in different scenarios involving both IID and non-IID scenarios.
    \item \scheme's weight scaling boosts the statistical efficiency.
    \item \scheme scales well and is more robust in future scenarios. 
\end{itemize}

\subsection{Experimental Setup}
Our experiments simulate FL benchmarks consisting of learners using real-world device configurations and availability traces. Our experiments capture different scenarios, models, datasets and data distributions as detailed next. We use a cluster of NVIDIA GPUs to interleave the training of the emulated learners. The participants train in parallel on time-multiplexed GPUs using PyTorch v1.8.0.

\smartparagraph{Implementation:} %
We implement \scheme within FedScale~\cite{lai2021fedscale}, a framework for emulation and evaluation of FL systems.
The SAA and IPS components are implemented as Python modules and integrated into FedScale's server aggregation logic and participant selection procedures, respectively.
We defer to \S\ref{sec:integrate} a discussion on deployment considerations and integration with popular FL frameworks.

\smartparagraph{Emulation environment:}
Our results are gathered via emulation in FedScale. This framework comprises three key components:

\noindent 1) The Aggregator selects participants, assigns the task, and handles the aggregation of updates. It employs a Client Manager to track learners' availability and selects the target number of participants in each round. It also employs an Event Monitor which processes system events and invokes the appropriate event handler.

\noindent 2) The Executor runs the learners' logic, loads the corresponding federated dataset, and trains the model using the PyTorch backend. The latency is determined for computation by \textit{\# of samples} $\times$ \textit{latency per sample}, and for communication by \textit{size in bytes} $/$ \textit{bandwidth}. This allows for having a simulated run time using realistic traces.

\noindent 3) The Resource Manager manages the available computational resources (e.g., GPUs). It employs a queue for the learners waiting for training on the computing resource and assigns a learner to the first available computing resource.\footnote{The queue does not affect the simulated time which is maintained by the event monitor which advances a global virtual clock based on the events and their correct time order.}

\smartparagraph{Benchmarks:} We experiment with the benchmarks listed in \cref{tab:benchmarks}, that span several common FL tasks of different scales to cover varied practical scenarios. The datasets consist of hundreds or up to millions of data samples, we refer the reader to \cite{lai2021fedscale,Oort-osdi21} for the description of the datasets. As in~\cite{Oort-osdi21}, by default, FedAvg~\cite{mcmahan2017} is used as the aggregation algorithm for CIFAR10, and YoGi~\cite{YoGi} for other benchmarks.

\smartparagraph{Data partitioning:}
To account for realistic heterogeneous data mappings, we partition the labeled training dataset among the learners using different methods, from easy to hard.
As commonly done in the literature, the baseline case is a random uniform mapping (IID).
Second, we adopt the FedScale data-to-learner mappings~\cite{lai2021fedscale}, which encompass thousands to millions of learners whose data distribution reflects real data sources.\footnote{For instance, images collected from Flickr in OpenImage have an AuthorProfileUrl attribute that can be used to map data learners, though these may not reflect real data mappings in FL scenarios.}
However, upon analyzing the frequency of label appearances across learners in the FedScale data mapping for the Google speech benchmark (c.f.  \cref{fig:labeldist}), we observe that most labels appear at least once on more than $40\%$ of the learners, making this close to a uniform distribution, and thus simplifying training. Similar observations are made for CV benchmarks.

\begin{figure}[!t]
    \centering
    \includegraphics[width=0.49\textwidth]{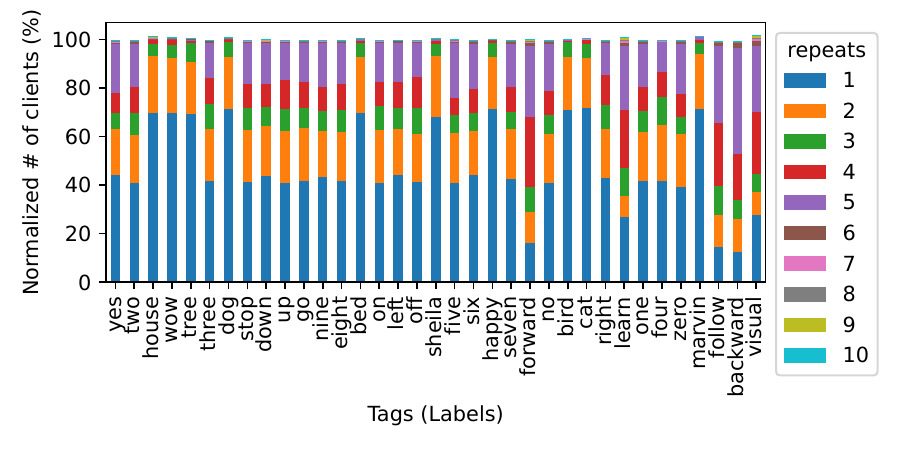}
    \caption{Number of label repetitions across learners.}
    \label{fig:labeldist}
\end{figure}

To consider other realistic heterogeneous data mappings, we introduce \textbf{label-limited mappings} where learners are assigned data samples drawn from a random subset of labels as listed in \cref{tab:benchmarks}, with data samples per learner following particular distributions as follows.
\begin{inparaenum}[L1)]
\item \textbf{Balanced distribution}: using an equal number of samples for each data label on each learner;
\item \textbf{Uniform distribution}: using uniform random assignment of data points to labels on each learner;
\item \textbf{Zipf distribution}: Zipfian distribution with $\alpha=1.95$ to have higher level of label skew (popularity).
\end{inparaenum}

\smartparagraph{System performance of learners: } Learners' hardware performance is assigned at random from profiles of real device measurements from the AI~\cite{AIranking} and MobiPerf~\cite{mobiperf} benchmarks for inference time and network speeds of mobile devices, respectively. AI Benchmark catalogs inference times for popular DNN models (e.g., MobileNet) on a wide range of Android devices (e.g., Samsung Galaxy S20 and Huawei P40). %
The profiles include devices with at least 2GB RAM using WiFi, which matches the common case in FL settings~\cite{yang2018applied,Bonawitz19,Oort-osdi21}.

We show how the distribution of floating-point and quantized inference times from the AI benchmark and device profiles can be clustered into 6 different device configurations demonstrating significant device heterogeneity with a long tail distribution, as shown in \cref{fig:dev-cdf}. \cref{fig:dev-cluster} shows that learners could be grouped into 6 clusters of different device configurations according to their computational capabilities, demonstrating that learners can have highly variable completion time during training.

\begin{figure}[!t]
\captionsetup[subfigure]{justification=centering}
\centering
     \begin{subfigure}[ht]{0.49\columnwidth}
     \includegraphics[width=\linewidth]{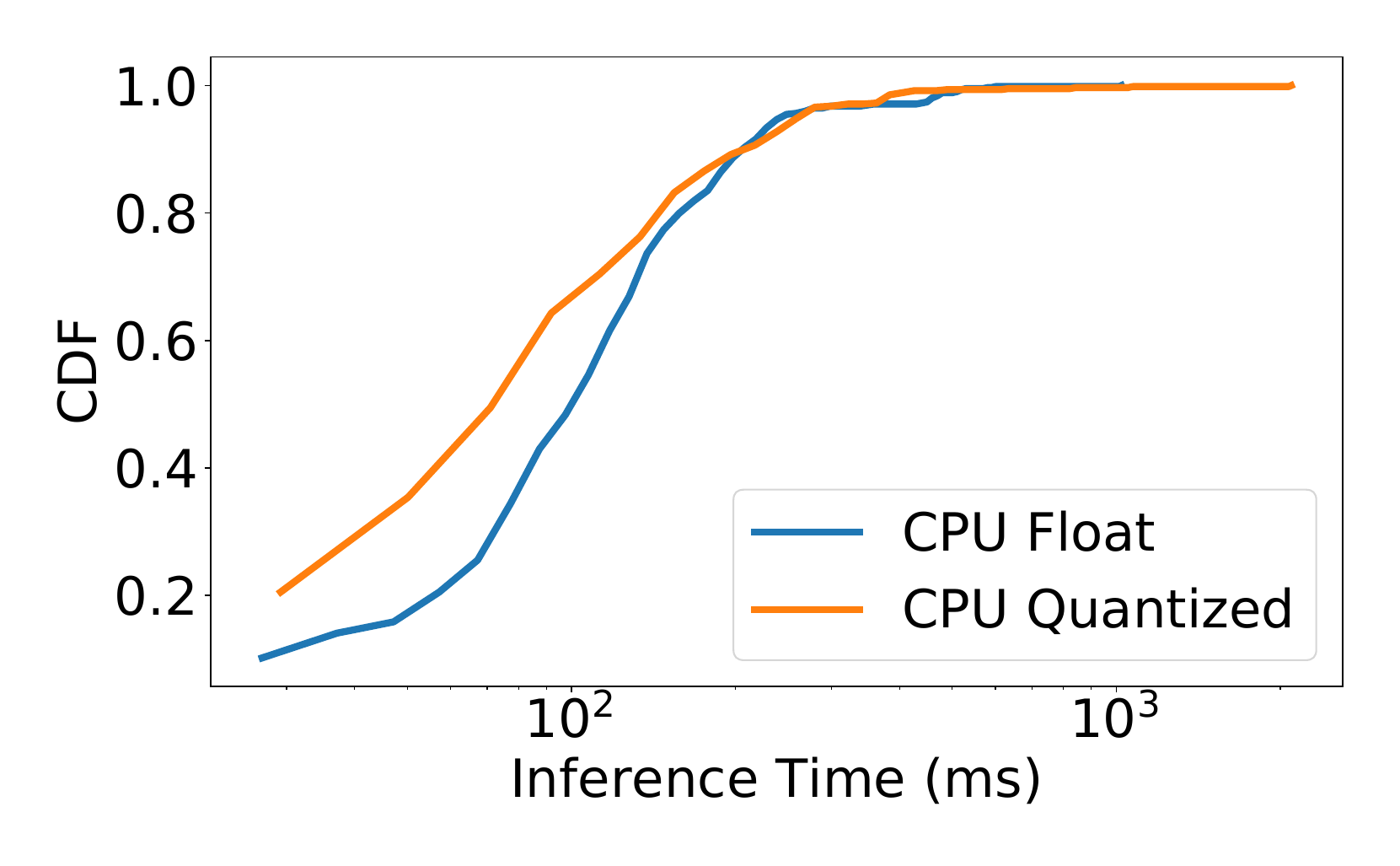}
	\caption{CDF of inference time}
	\label{fig:dev-cdf}
     \end{subfigure}
     \hfill
    \begin{subfigure}[ht]{0.49\columnwidth}
     \includegraphics[width=\linewidth]{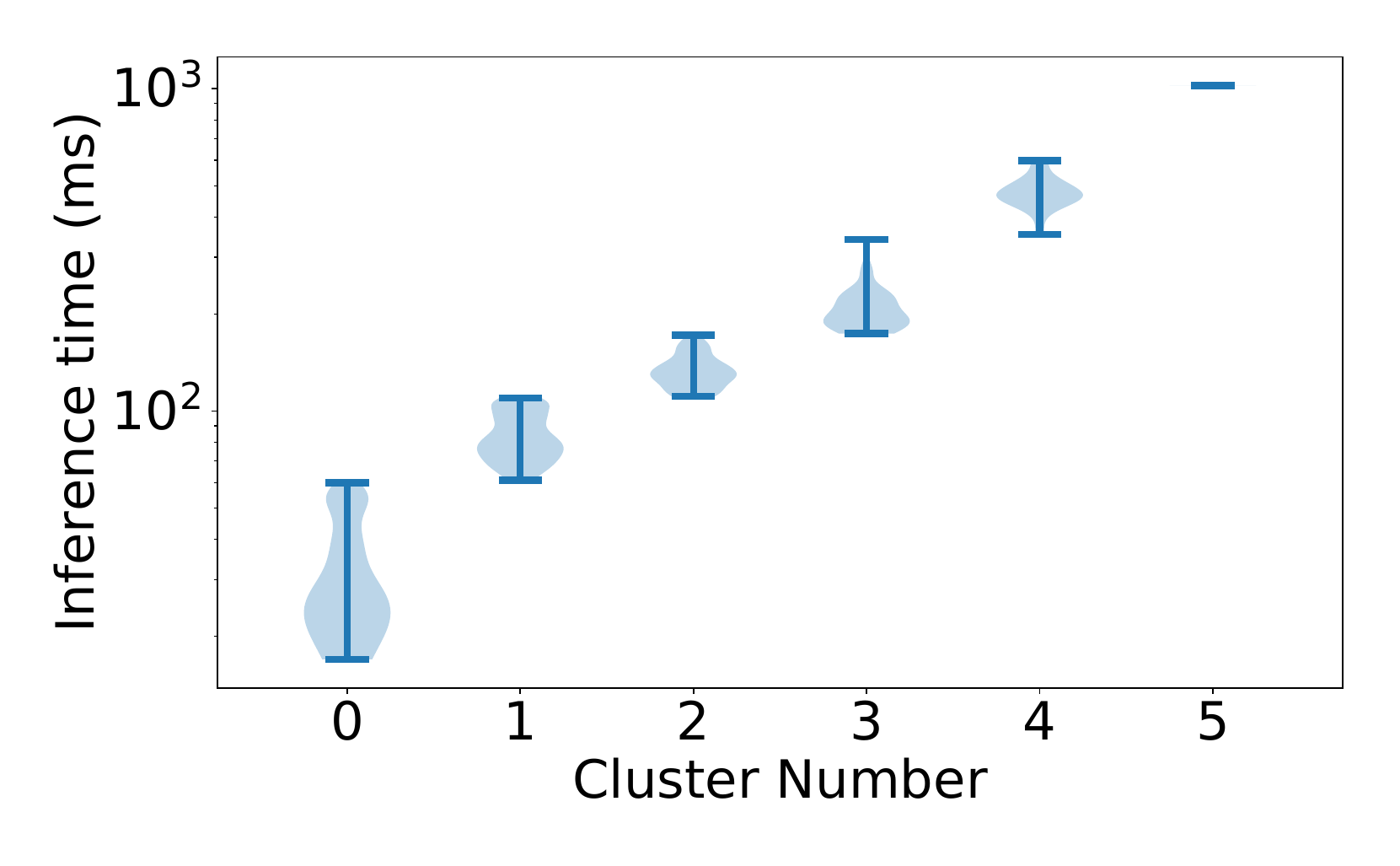} 
	\caption{Clustering of devices}
	\label{fig:dev-cluster}
     \end{subfigure}
     \\
    \begin{subfigure}[ht]{0.49\columnwidth}
     \includegraphics[width=\linewidth]{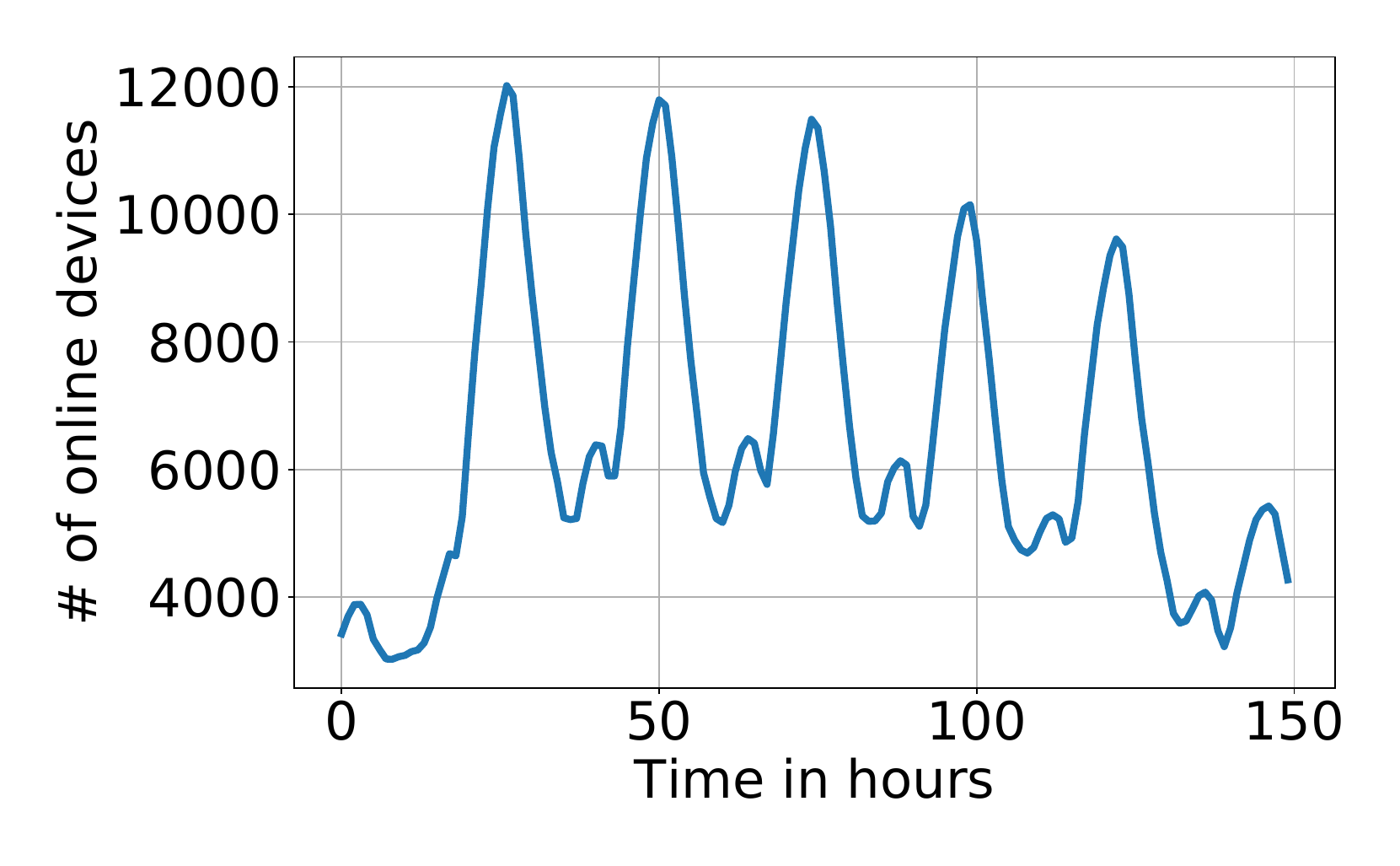}
	\caption{Available learners over time}
	\label{fig:online-clients}
     \end{subfigure}
     \hfill
    \begin{subfigure}[ht]{0.49\columnwidth}
     \includegraphics[width=\linewidth]{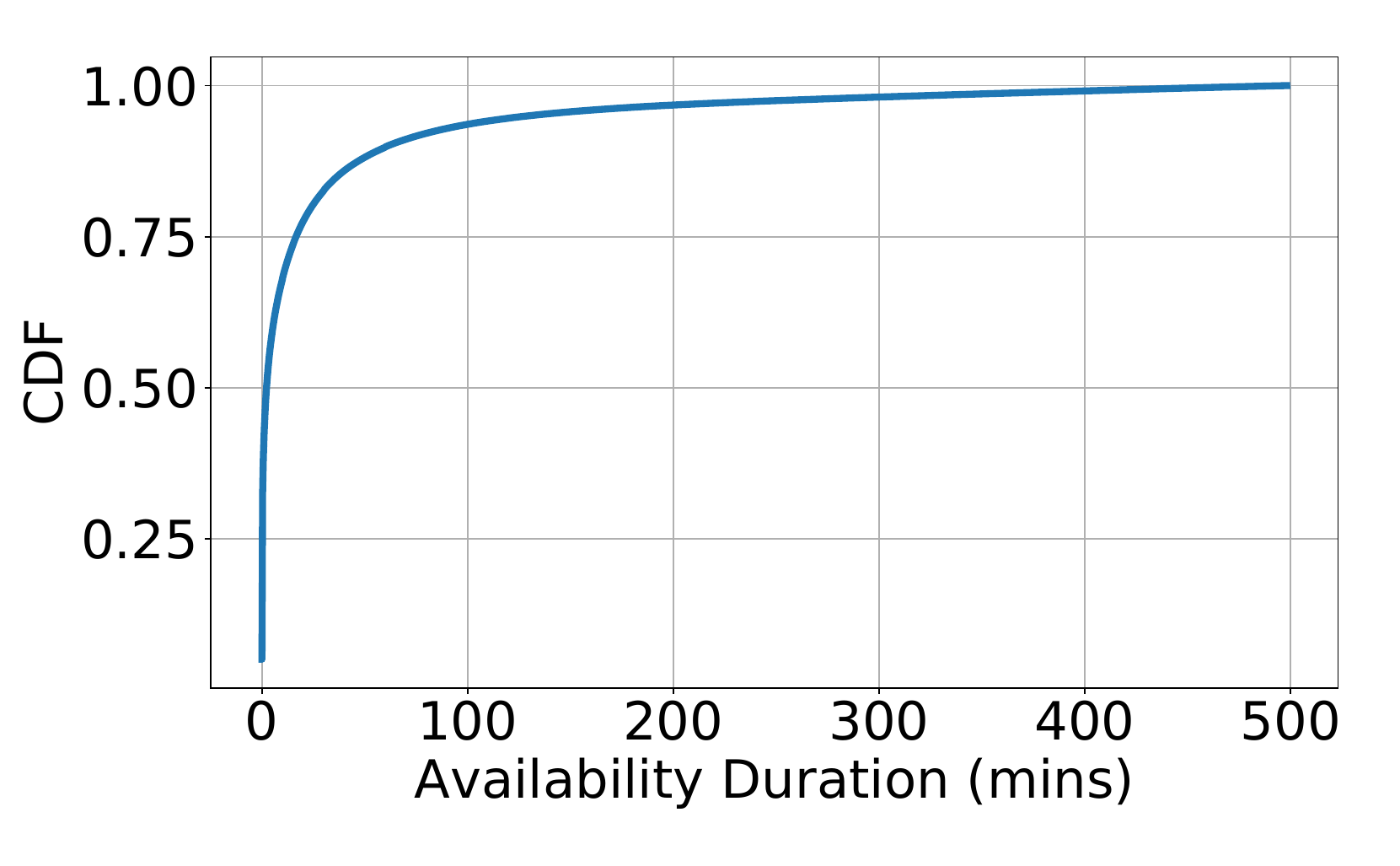} 
	\caption{CDF of availability period}
	\label{fig:online-cdf}
     \end{subfigure}
\caption{Computational [(a),(b)] and availability [(c),(d)] characteristics of learners' profiles used in experiments.} %
\label{fig:dev-conf}
\end{figure}

\begin{figure}[!t]
\captionsetup[subfigure]{justification=centering}
\centering
    \includegraphics[width=1\linewidth]{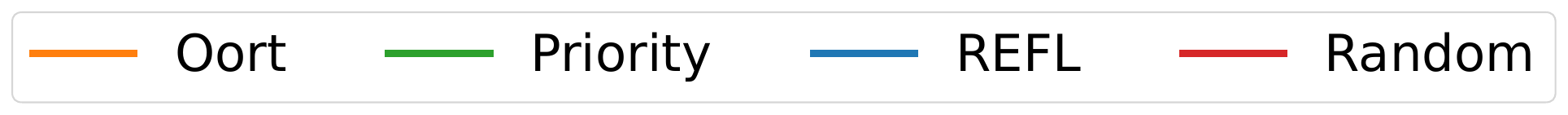}
    \\
     \begin{subfigure}[ht]{0.49\linewidth}
     \includegraphics[width=\linewidth]{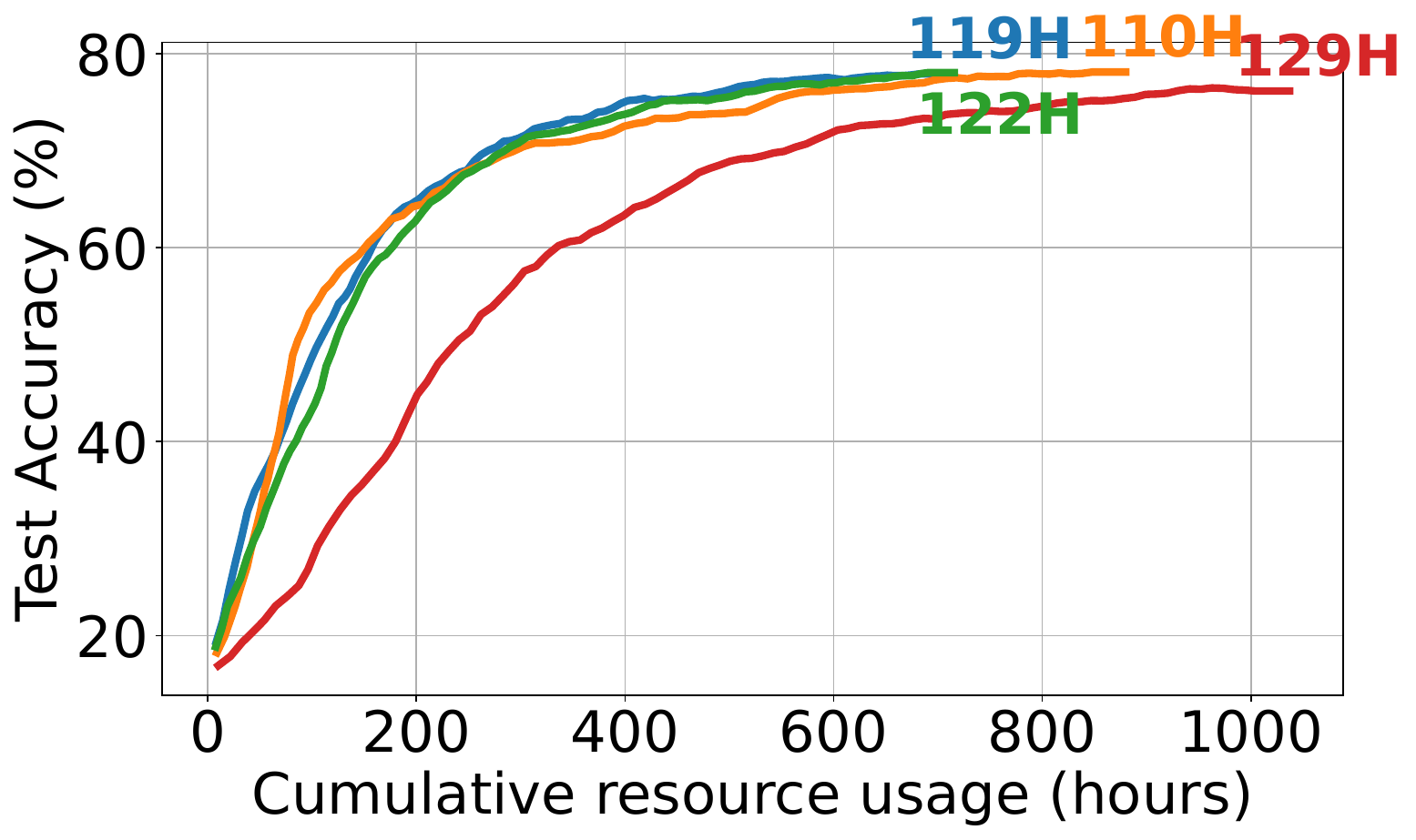}
	\caption{FedScale's data mapping}
	\label{fig:avail-exp1-part0}
     \end{subfigure}
     \hfill
      \begin{subfigure}[ht]{0.49\linewidth}
     \includegraphics[width=\linewidth]{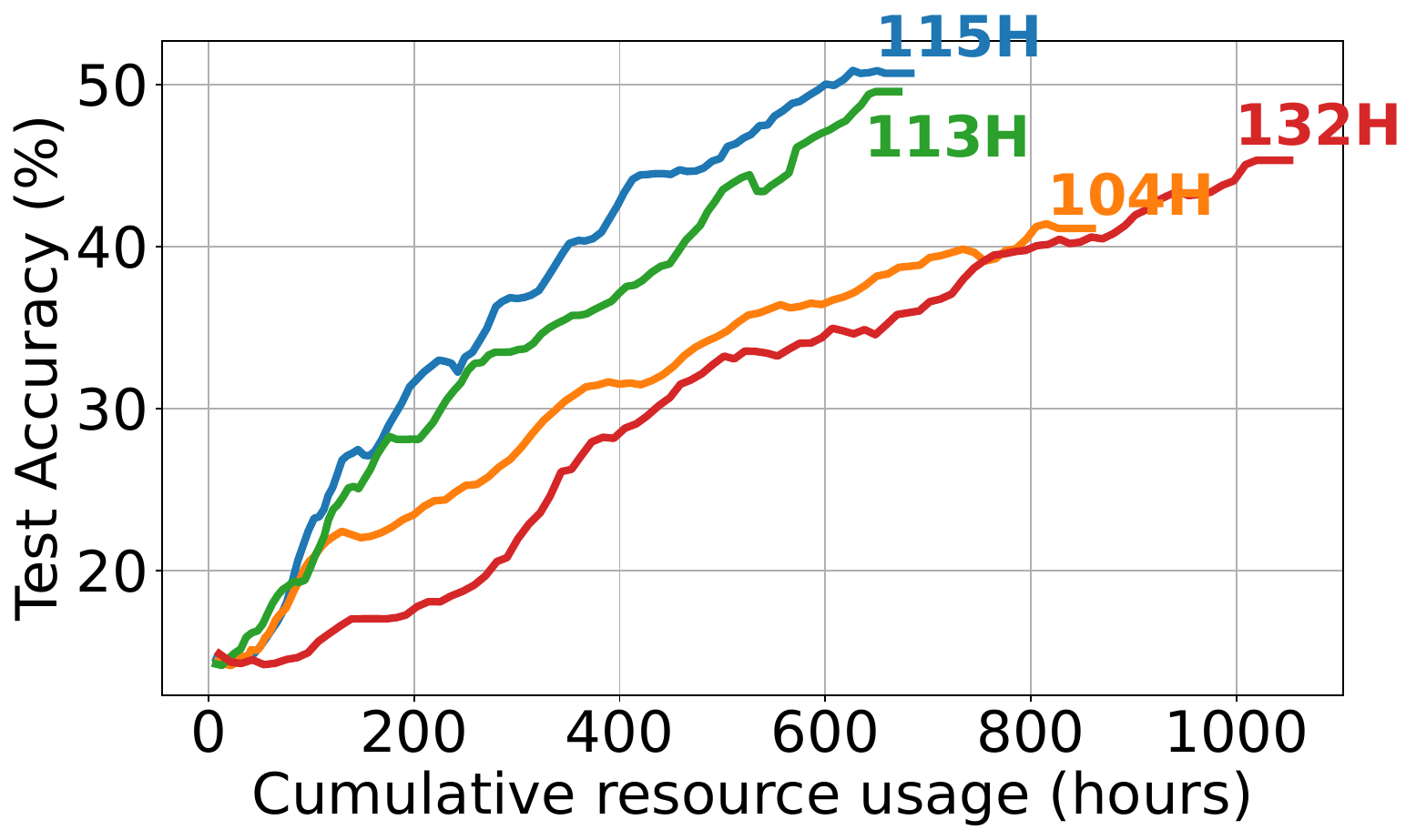} 
	\caption{Label-limited (balanced)}
	\label{fig:avail-exp1-part3}
     \end{subfigure}
     \\
    \begin{subfigure}[ht]{0.49\linewidth}
     \includegraphics[width=\linewidth]{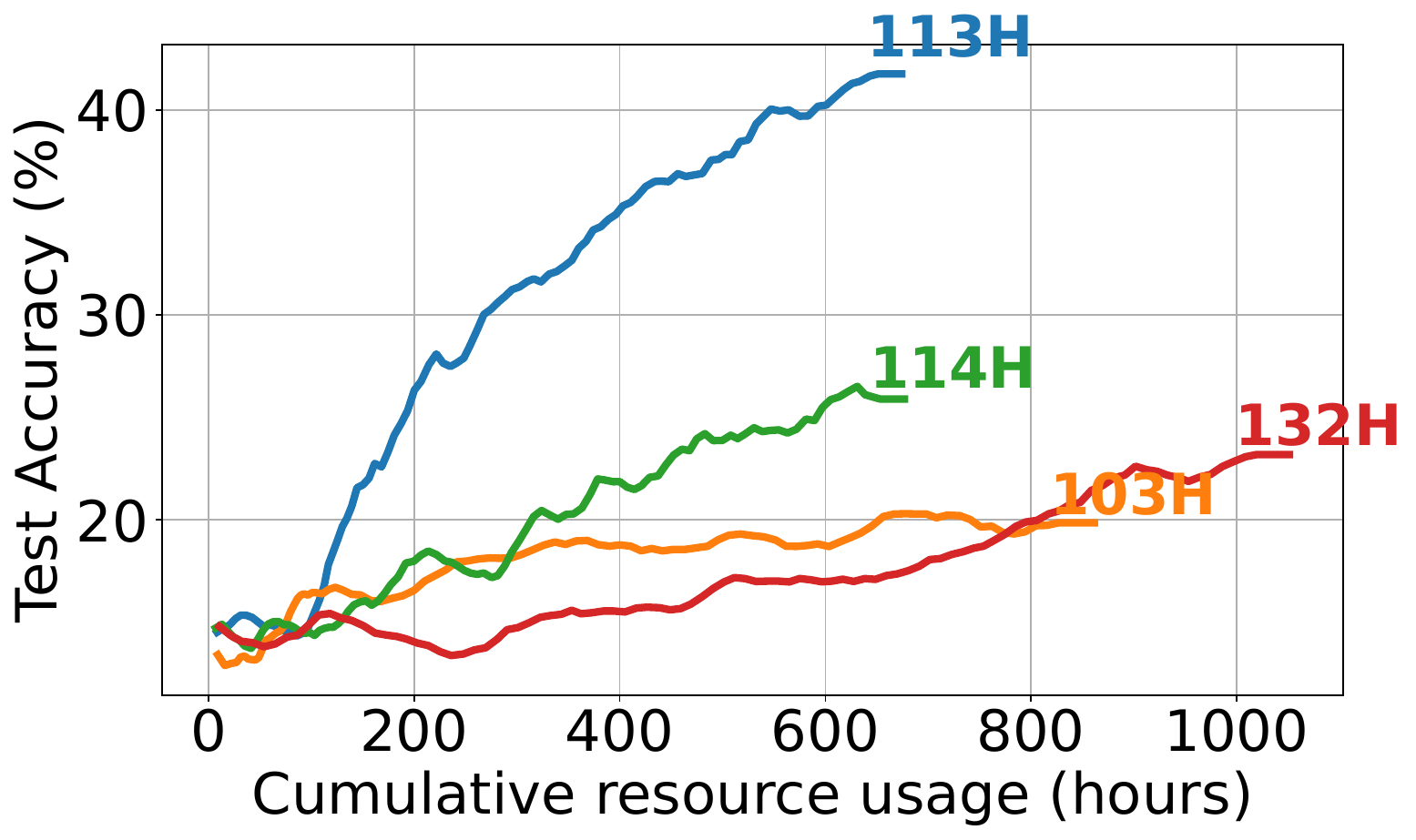} 
	\caption{Label-limited (uniform)}
	\label{fig:avail-exp1-part1}
     \end{subfigure}
     \hfill
         \begin{subfigure}[ht]{0.49\linewidth}
     \includegraphics[width=\linewidth]{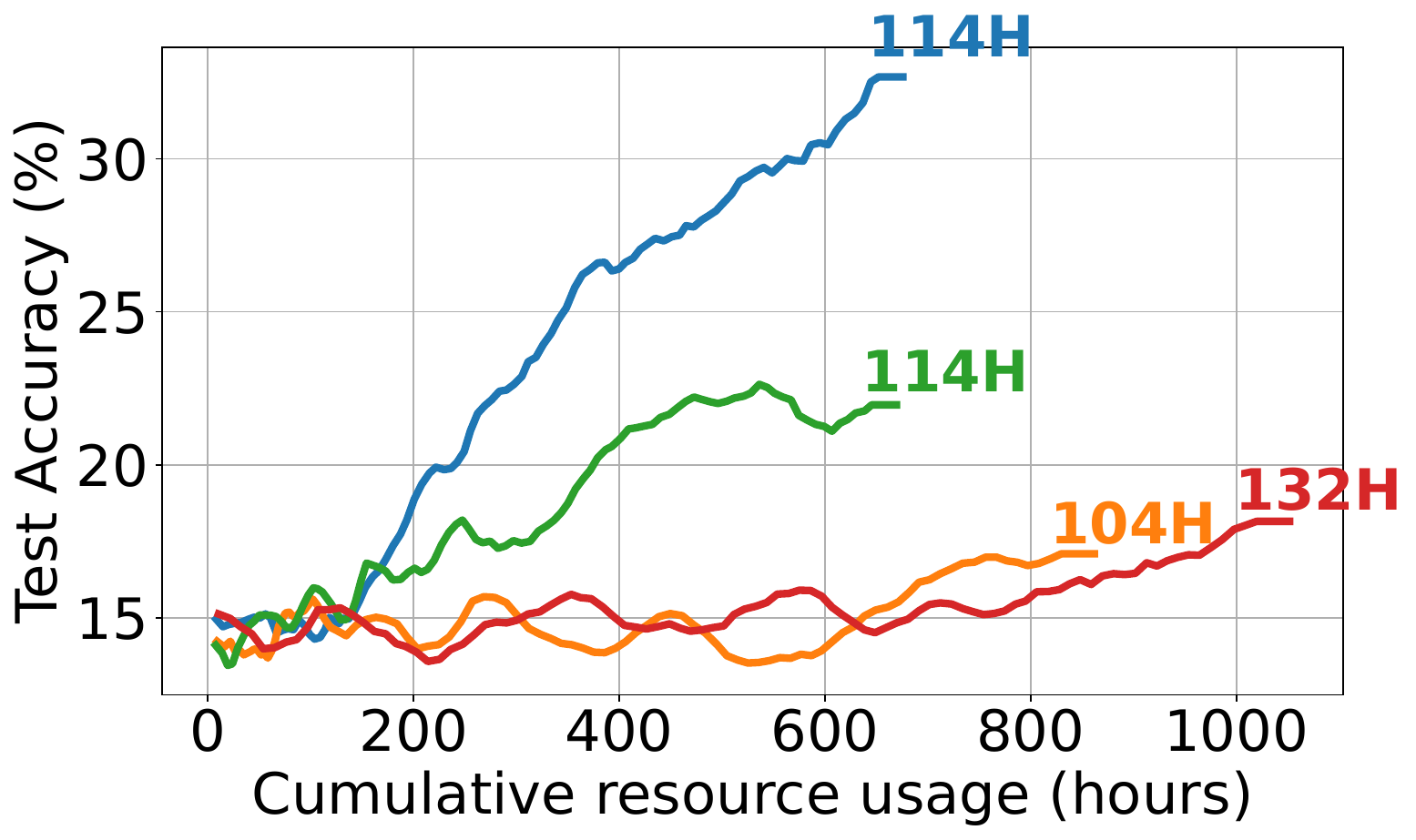}
	\caption{Label-limited (Zipfian)}
	\label{fig:avail-exp1-part2}
     \end{subfigure}
     \\
     \begin{subfigure}[ht]{0.49\linewidth}
     \includegraphics[width=\linewidth]{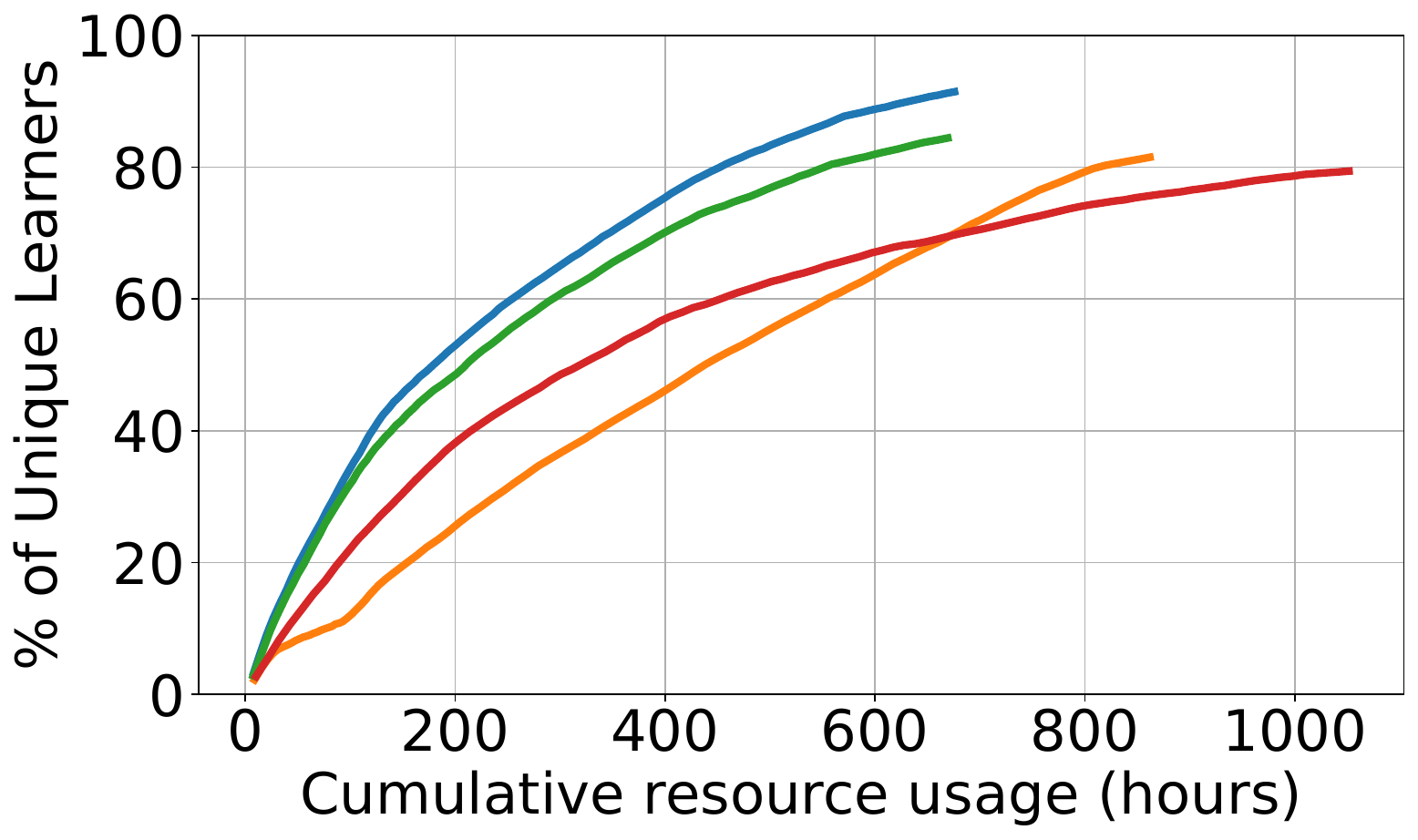} 
	\caption{Percent. of unique learners}
	\label{fig:avail-exp1-uniquelearners}
     \end{subfigure}
     \hfill
     \begin{subfigure}[ht]{0.49\linewidth}
     \includegraphics[width=\linewidth]{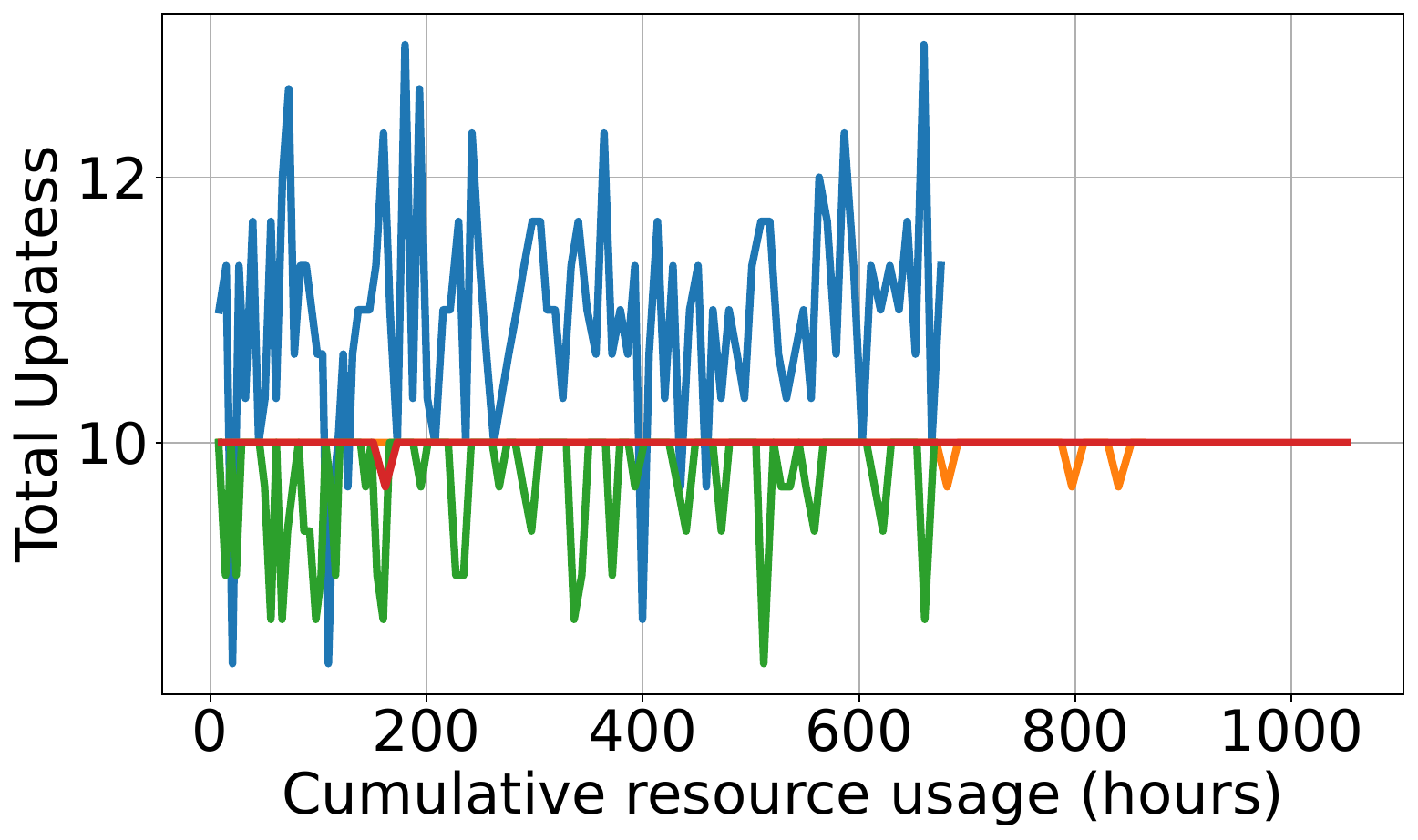}
     \caption{\# of aggregated updates}
	\label{fig:avail-exp1-totalupdates}
     \end{subfigure}
\caption{Training performance comparison under  \emph{\bf OC+DynAvail} across different data mappings.}
\label{fig:avail-exp-google}
\end{figure}

\begin{figure*}[!h]
\captionsetup[subfigure]{justification=centering}
\centering
      \begin{subfigure}[ht]{0.32\linewidth}
     \includegraphics[width=\linewidth]{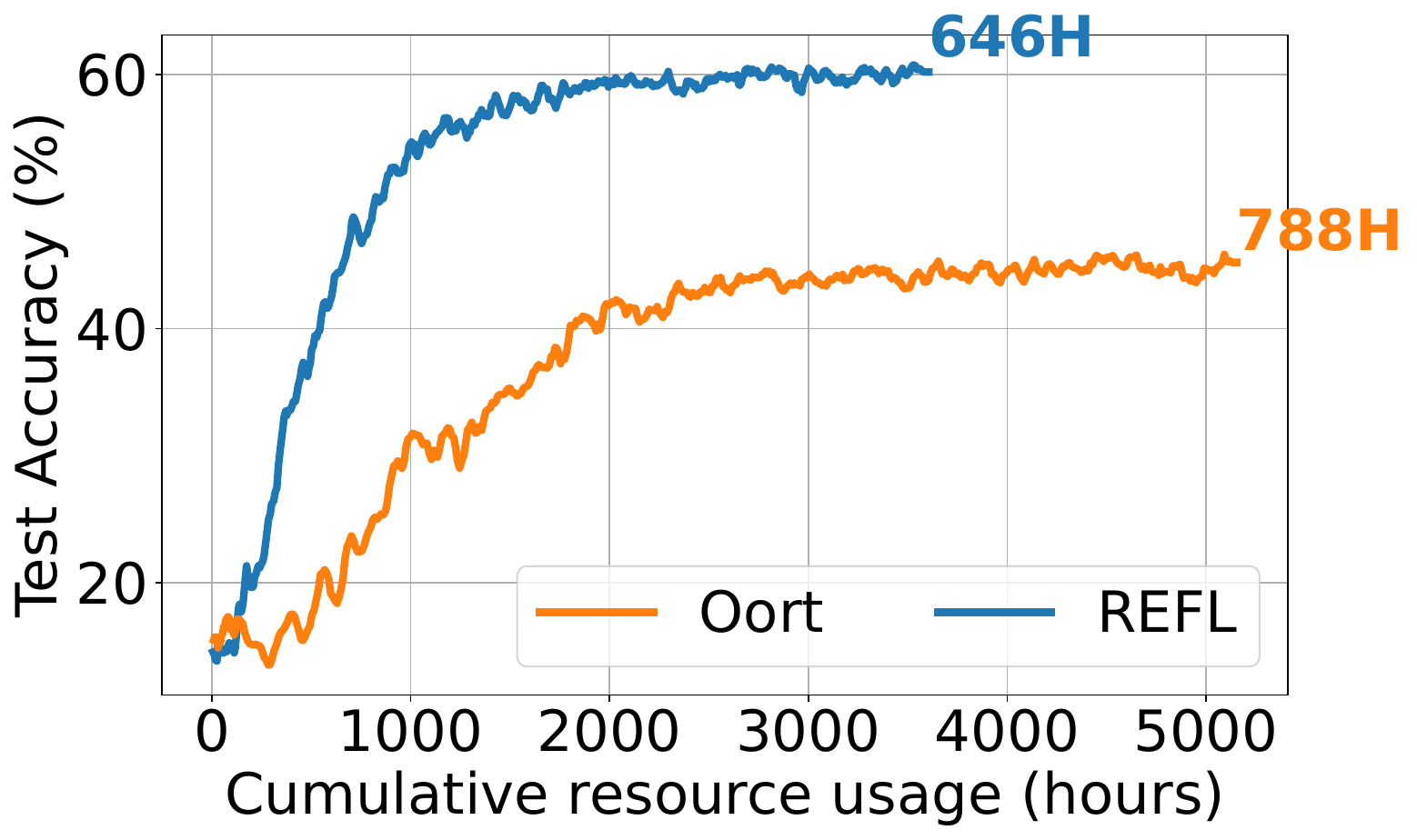}  
	\caption{Balanced}
	\label{fig:avail-exp1-part3-long}
     \end{subfigure}
     \hfill
    \begin{subfigure}[ht]{0.32\linewidth}
     \includegraphics[width=\linewidth]{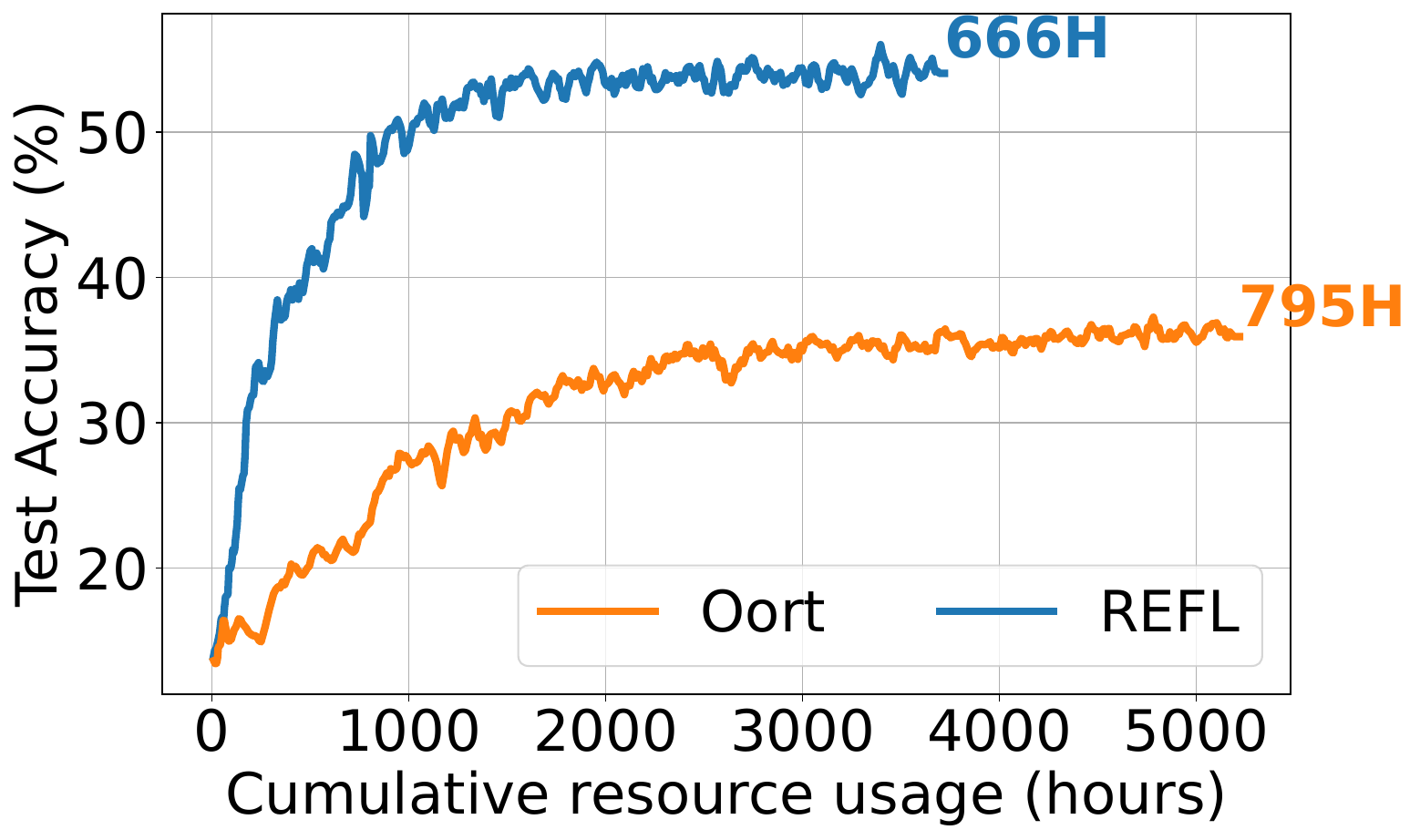} 
	\caption{Uniform}
	\label{fig:avail-exp1-part1-long}
     \end{subfigure}
     \hfill
         \begin{subfigure}[ht]{0.32\linewidth}
     \includegraphics[width=\linewidth]{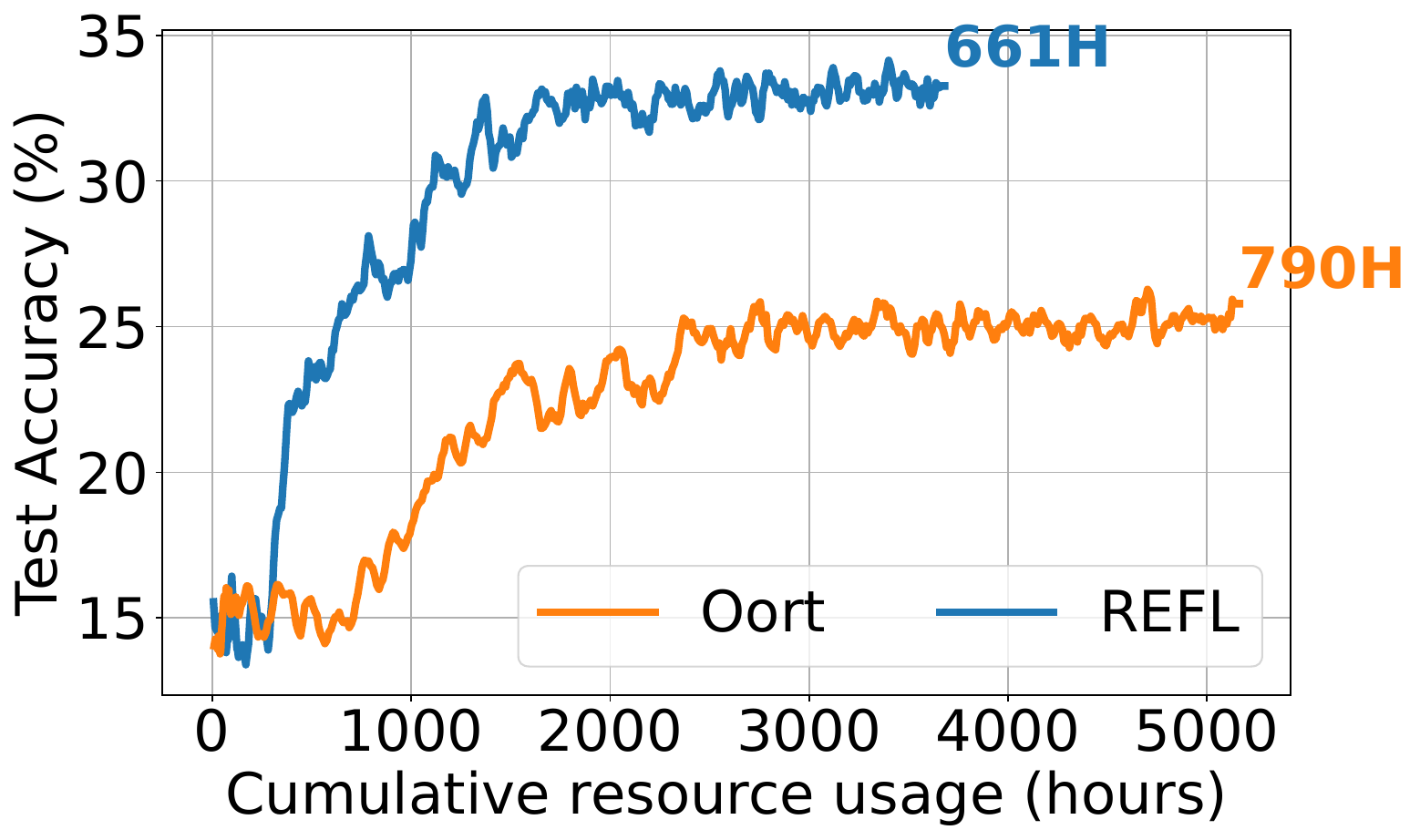}
	\caption{Zipfian}
	\label{fig:avail-exp1-part2-long}
     \end{subfigure}
\caption{Training convergence comparison under  \emph{\bf OC+DynAvail} across the different Label-limited (non-iid) data mappings.}
\label{fig:avail-exp-google-long}
\end{figure*}

\smartparagraph{Availability dynamics of learners: } We use a trace of 136K mobile users from different countries over a period of 1-week~\cite{yang2020heterogeneityaware}. The trace contains $\approx$180 million entries for events such as connecting to WiFi, charging the battery, and (un)locking the screen. Availability is defined as when a device is plugged to a charger and connected to the network, similar to~\cite{tff,Oort-osdi21}.

We see that learners exhibit variations (and cyclic behavior) and most learners stay available for only a few minutes. We extract availability dynamics from the user behavior trace in~\cite{yang2020heterogeneityaware}. \cref{fig:online-clients} shows that the number of available learners over time varies significantly and they exhibit a diurnal pattern over the days of the week where large numbers of learners are mostly available (i.e., charging) during the night. \cref{fig:online-cdf} shows a CDF of the length of learners' availability slots which exhibits a very long tail. Most clients (up to 70\%) are available for less than 10 minutes.

\smartparagraph{Hyper-parameter settings:} The FL and learning hyper-parameters were the default values set by the FedScale framework and no further tuning was done. The common FL hyper-parameters were the same for all methods in the comparison. We used the recommended parameter settings for the evaluated methods (i.e., Oort~\cite{Oort-osdi21} and SAFA~\cite{Safa-Wu2021}).

\smartparagraph{\scheme parameters:} Unless otherwise stated, no maximum threshold is applied to staleness when incorporating updates. We set $\alpha=0.25$ for the moving average of round duration which is chosen to give more weight to the most recent round duration values. We set $\beta=0.35$ for the weight of stale updates in \cref{eq:hybrid} to favour dampening over scaling the stale updates. We tested, within our experimental capacity, different values and found the aforementioned values worked well. %
We leave a detailed sensitivity analysis and ablation study of hyper-parameters to future work.

\smartparagraph{Availability prediction model:} In the experiments, we assume the model has 90\% accuracy for future availability (i.e., 1 out of 10 selections is a false positive), which matches the model accuracy obtained from training a simple linear model on real-world trace as detailed in~\cref{subsec:results}.  %

\begin{figure}[!t]
\captionsetup[subfigure]{justification=centering}
\centering
     \begin{subfigure}[ht]{0.49\linewidth}
      \includegraphics[width=\linewidth]{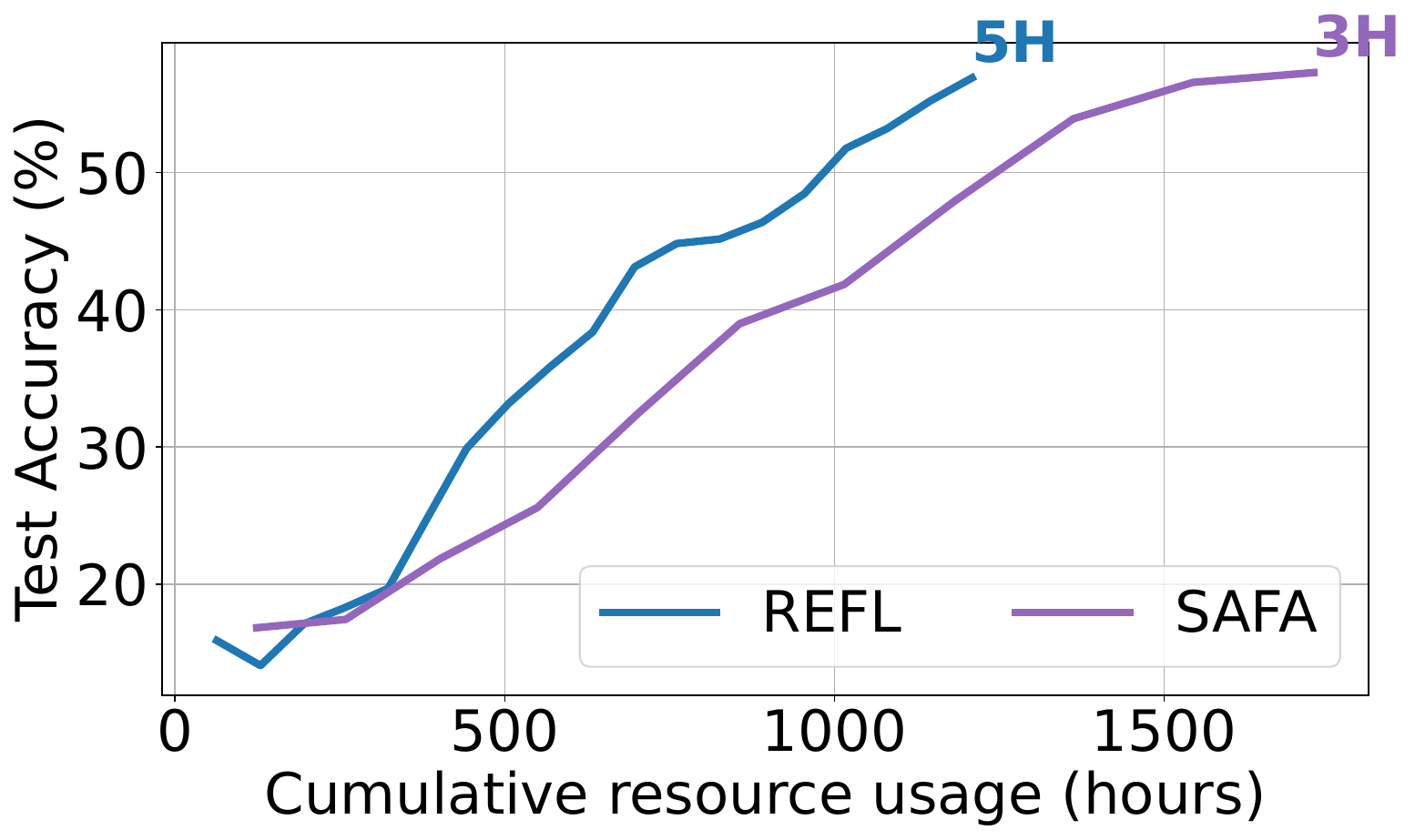}
	\caption{Uniform mapping (iid)}
	\label{fig:safa-real}
     \end{subfigure}
     \hfill
    \begin{subfigure}[ht]{0.49\linewidth}
     \includegraphics[width=\linewidth]{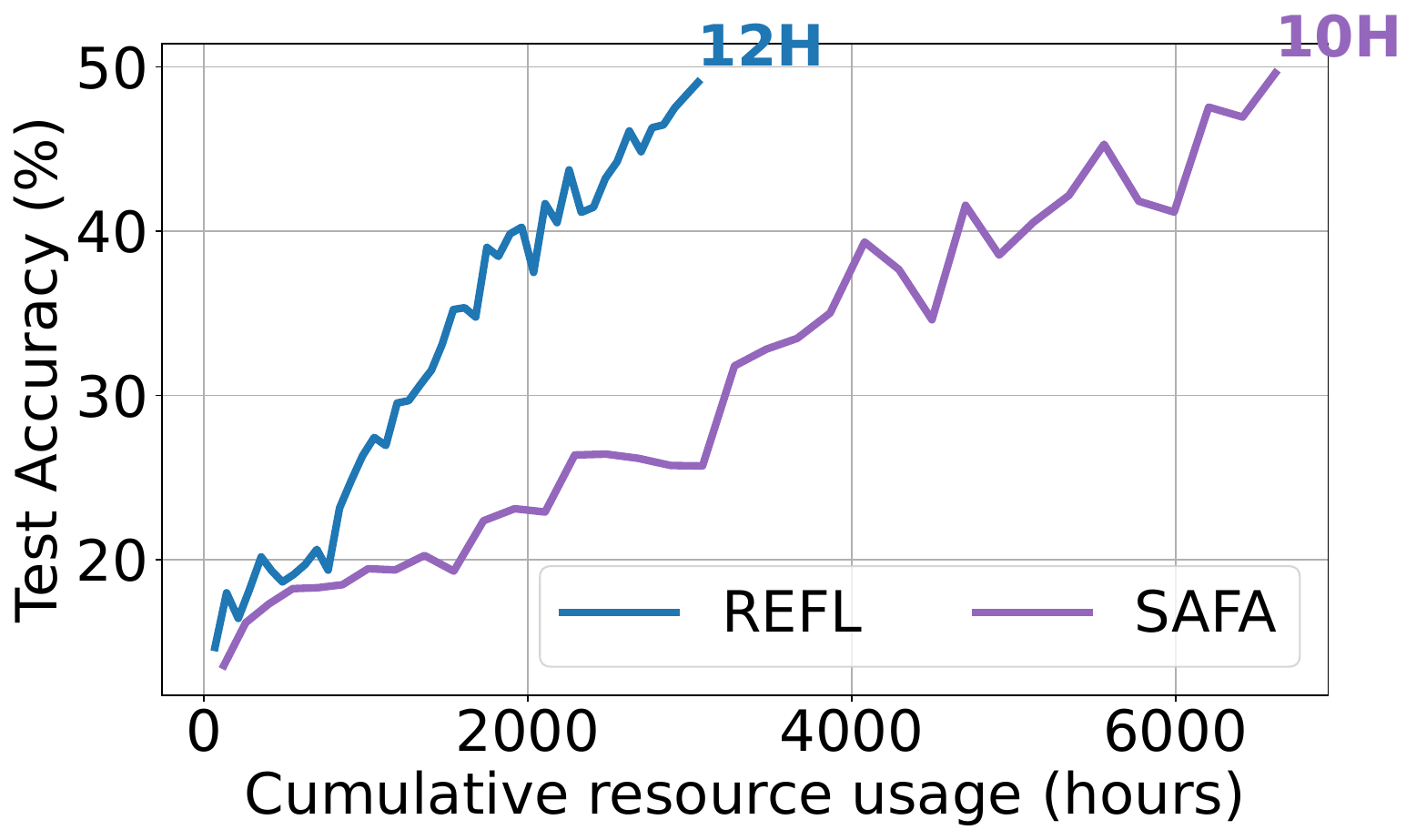}
	\caption{Label-limited (uniform)}
	\label{fig:safa-limited}
     \end{subfigure}
\caption{Comparison against SAFA.}
\label{fig:safa}
\end{figure}

\smartparagraph{Experimental scenarios:} We consider two experimental settings used in the literature:
\begin{enumerate}
\item \emph{\bf OC:} the FL server over-commits the target number of participants $N_t$ by 30\% and waits for the updates from $N_t$ participants (as in~\cite{Oort-osdi21,lai2021fedscale});
\item \emph{\bf DL:} the FL server chooses a target number of participants $N_t$ and aggregates any number of updates received before the end of a pre-set reporting deadline (as in~\cite{yang2020heterogeneityaware, Bonawitz19}).
\end{enumerate}
Unless otherwise mentioned, the target number of participants is 10; each experiment is repeated 3 times with different sampling seeds and the average of the three runs is shown. The experiments use $\approx$13K hours of GPU time.

\subsection{Experimental Results}
\label{subsec:results}

We evaluate the amount of learner resources (and run time) spent to achieve a certain test accuracy (\emph{lower resources and lower run time are better}).
Since our results are based on emulation, we quantify resource usage using the time accumulated at every learner as a proxy.
In particular, the cumulative resources for each round are computed as the cumulative sum of computation and communication time for all participants in the round.

Here, we focus on Google Speech and present the results of other benchmarks and experimental settings in \S\ref{subsec:moreresults}. \cref{tab:baselinequality} show the baseline accuracy of the benchmarks in a semi-centralized training setting (i.e., data-parallel) where the dataset is uniformly divided among 10 learners that participate in every training round.  %

\begin{figure}[!t]
\captionsetup[subfigure]{justification=centering}
\centering
   \includegraphics[width=1\linewidth]{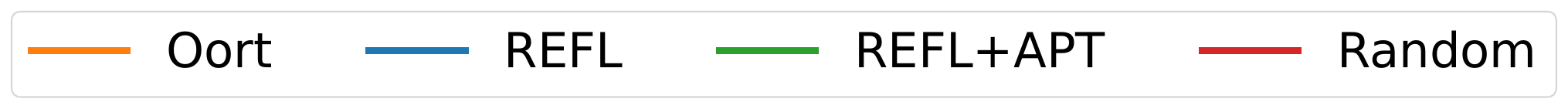}
       \begin{subfigure}[ht]{0.49\linewidth}
     \includegraphics[width=\linewidth]{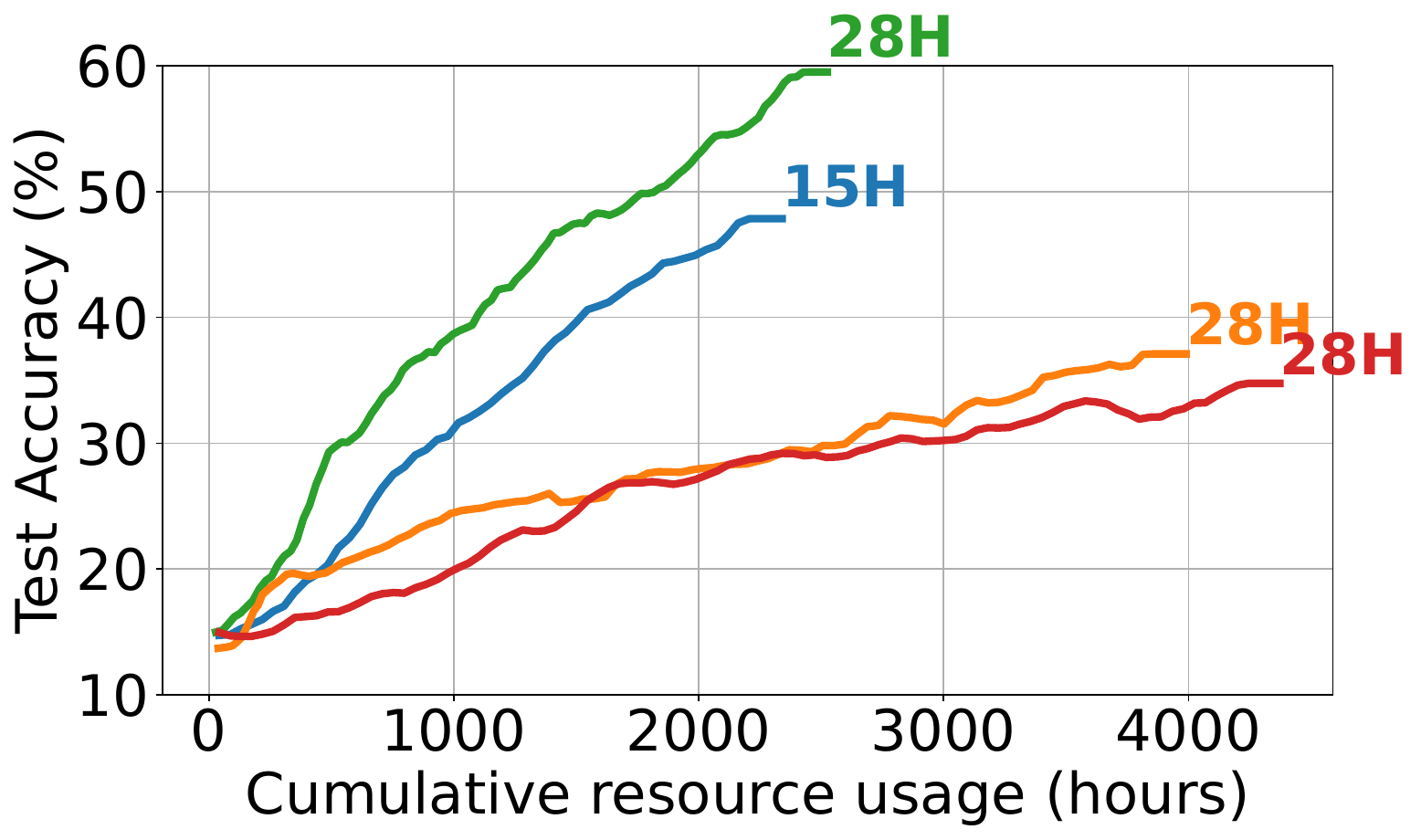} 
	\caption{AllAvail}
	\label{fig:APS-all}
     \end{subfigure}
     \hfill
       \begin{subfigure}[ht]{0.49\linewidth}
     \includegraphics[width=\linewidth]{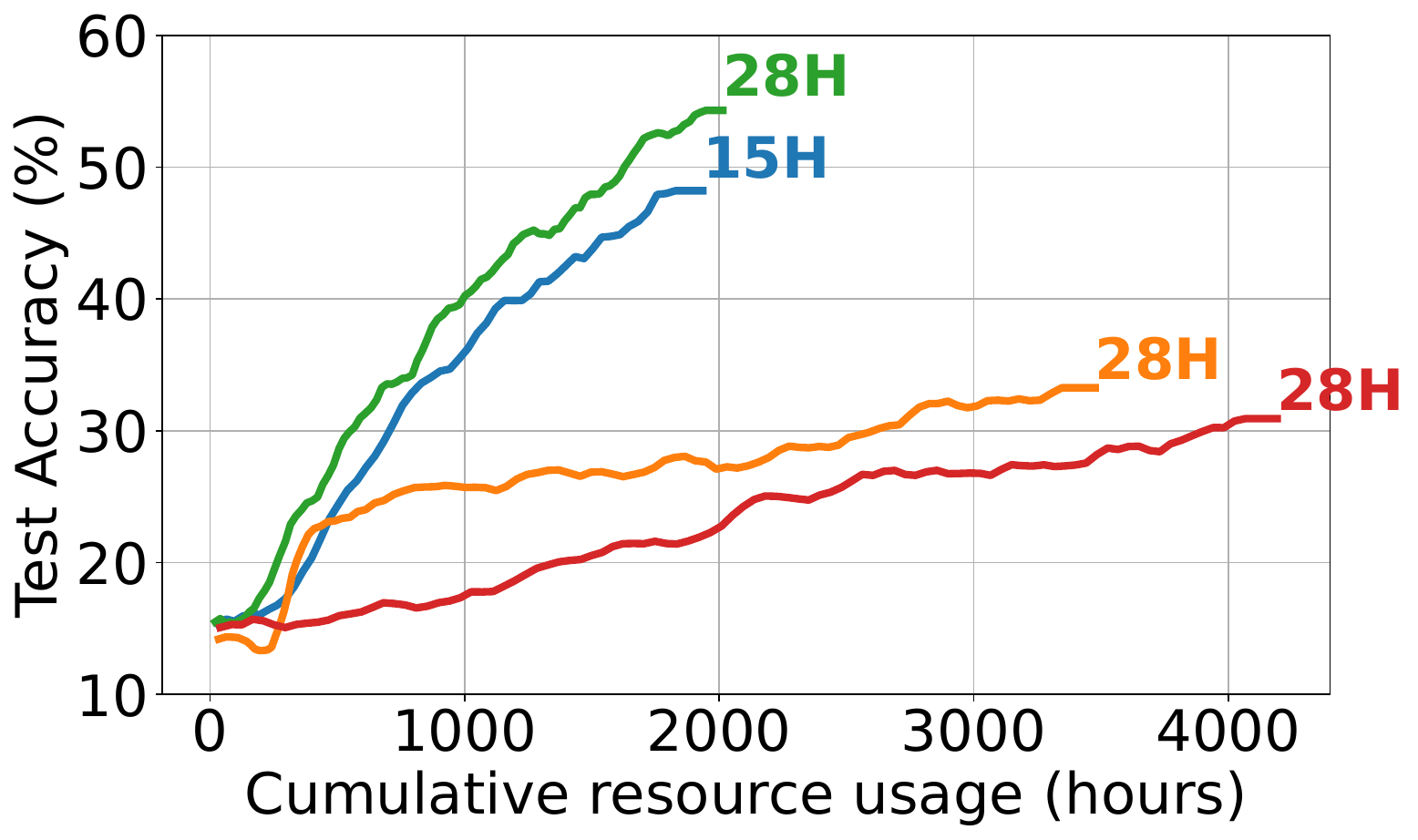}
	\caption{DynAvail}
	\label{fig:APS-dyn}
     \end{subfigure}
\caption{Training performance of \scheme with Adaptive Participant Target using 50 participants in OC and different availability.}
\label{fig:APS}
\end{figure}

\begin{figure}[!t]
\captionsetup[subfigure]{justification=centering}
\centering
    \includegraphics[width=0.7\linewidth]{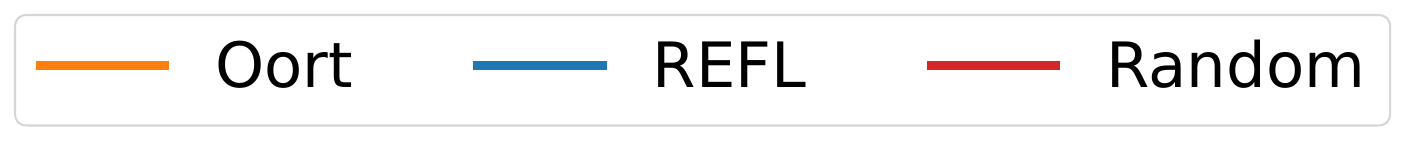}
    \\
     \begin{subfigure}[ht]{0.49\linewidth}
     \includegraphics[width=\linewidth]{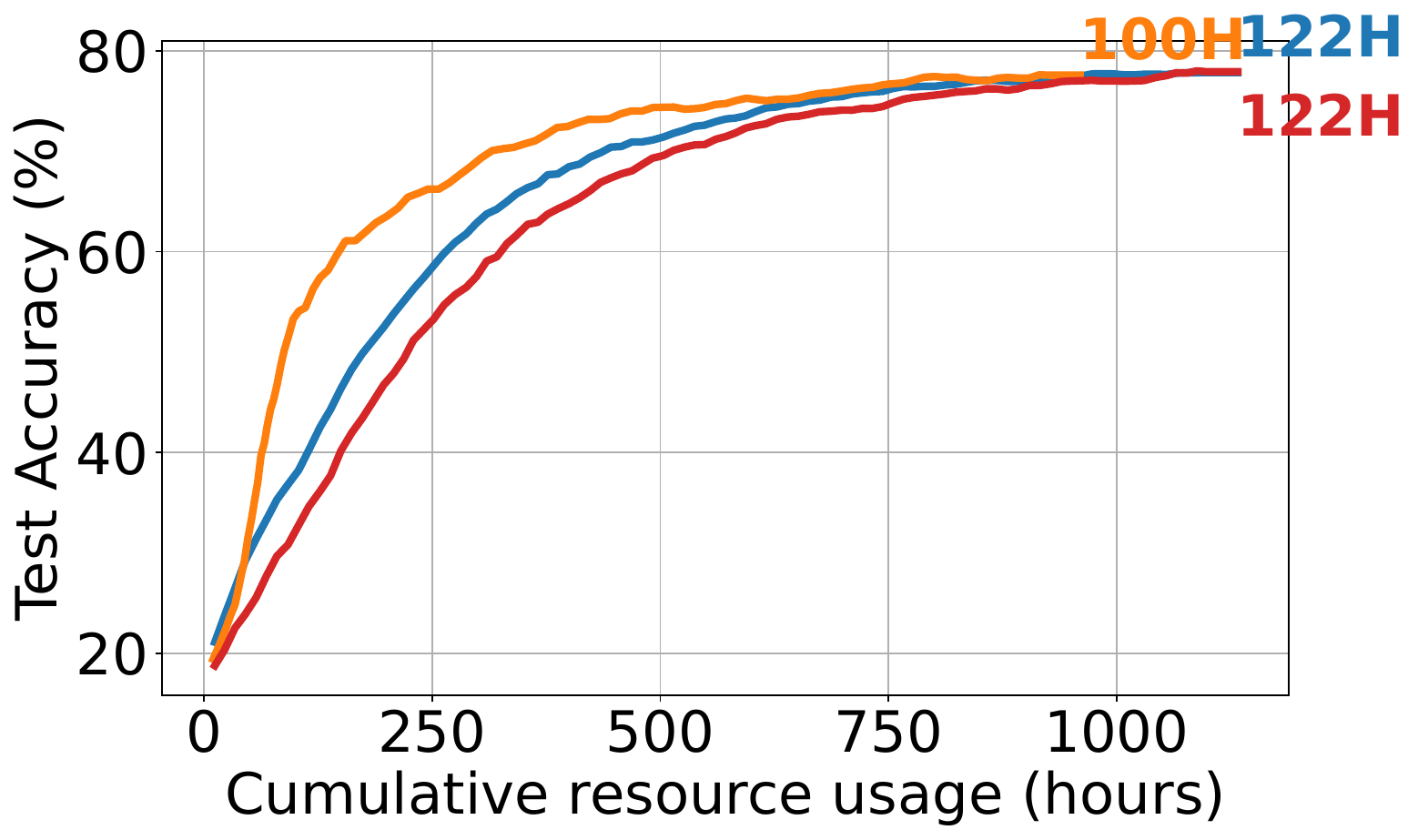}
	\caption{FedScale data mapping}
	\label{fig:gs-exp1-part0}
     \end{subfigure}
     \hfill
    \begin{subfigure}[ht]{0.49\linewidth}
     \includegraphics[width=\linewidth]{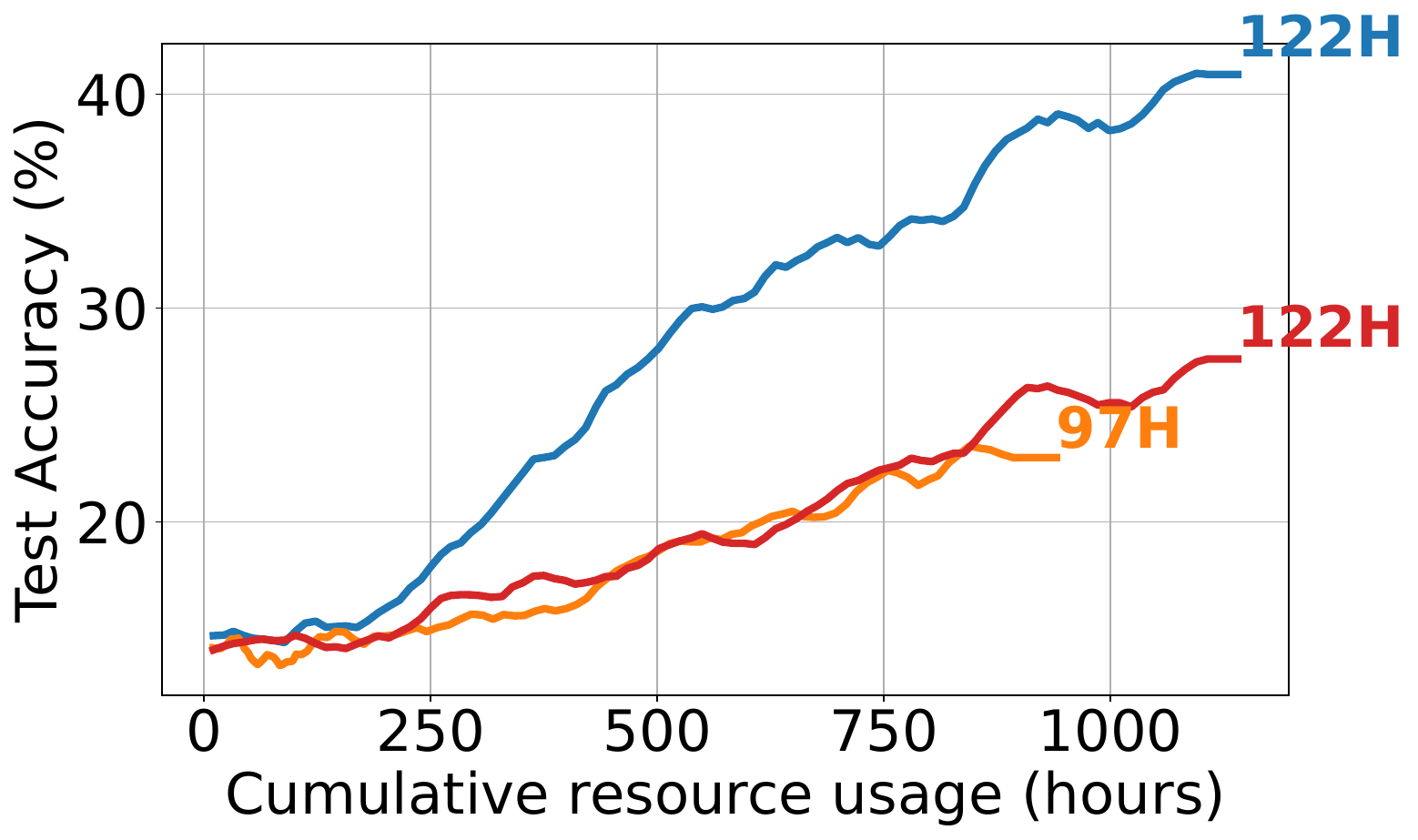} 
	\caption{Label-limited (uniform)}
	\label{fig:gs-exp1-part1}
     \end{subfigure}
\caption{Performance comparison of \scheme vs Oort vs Random in the \emph{\bf OC+AllAvail} experimental setting.}
\label{fig:stale-exp1-google}
\end{figure}

\begin{figure*}[!t]
\captionsetup[subfigure]{justification=centering}
\centering
    \includegraphics[width=0.5\linewidth]{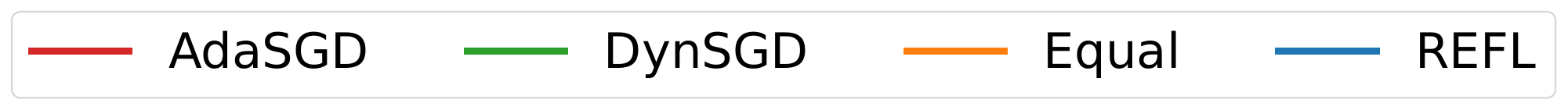}
    \\
     \begin{subfigure}[t]{0.19\linewidth}
     \includegraphics[width=\linewidth]{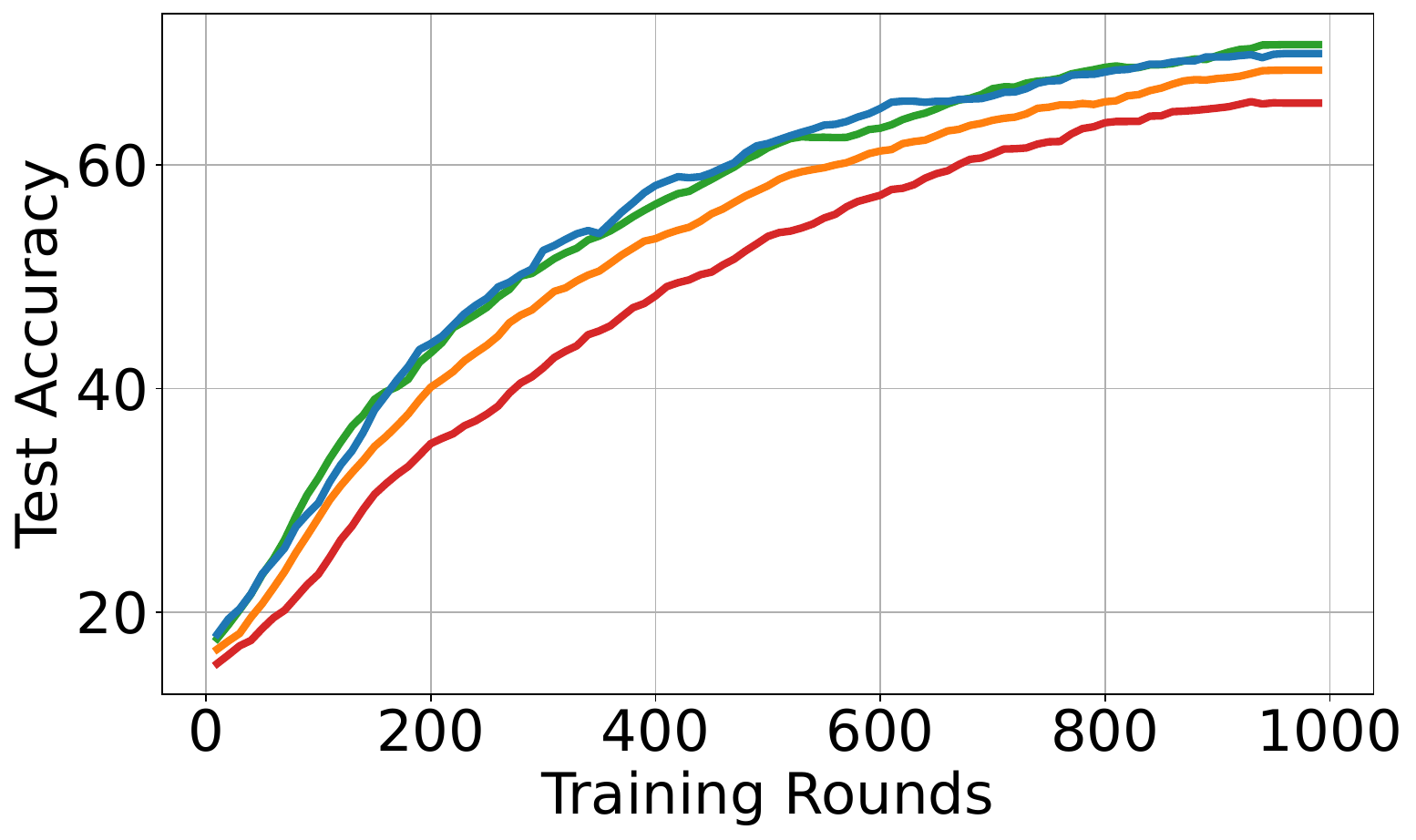}
	\caption{Uniform random\\ (IID)}
	\label{fig:scaling-uniform}
     \end{subfigure}
     \hfill
     \begin{subfigure}[t]{0.19\linewidth}
     \includegraphics[width=\linewidth]{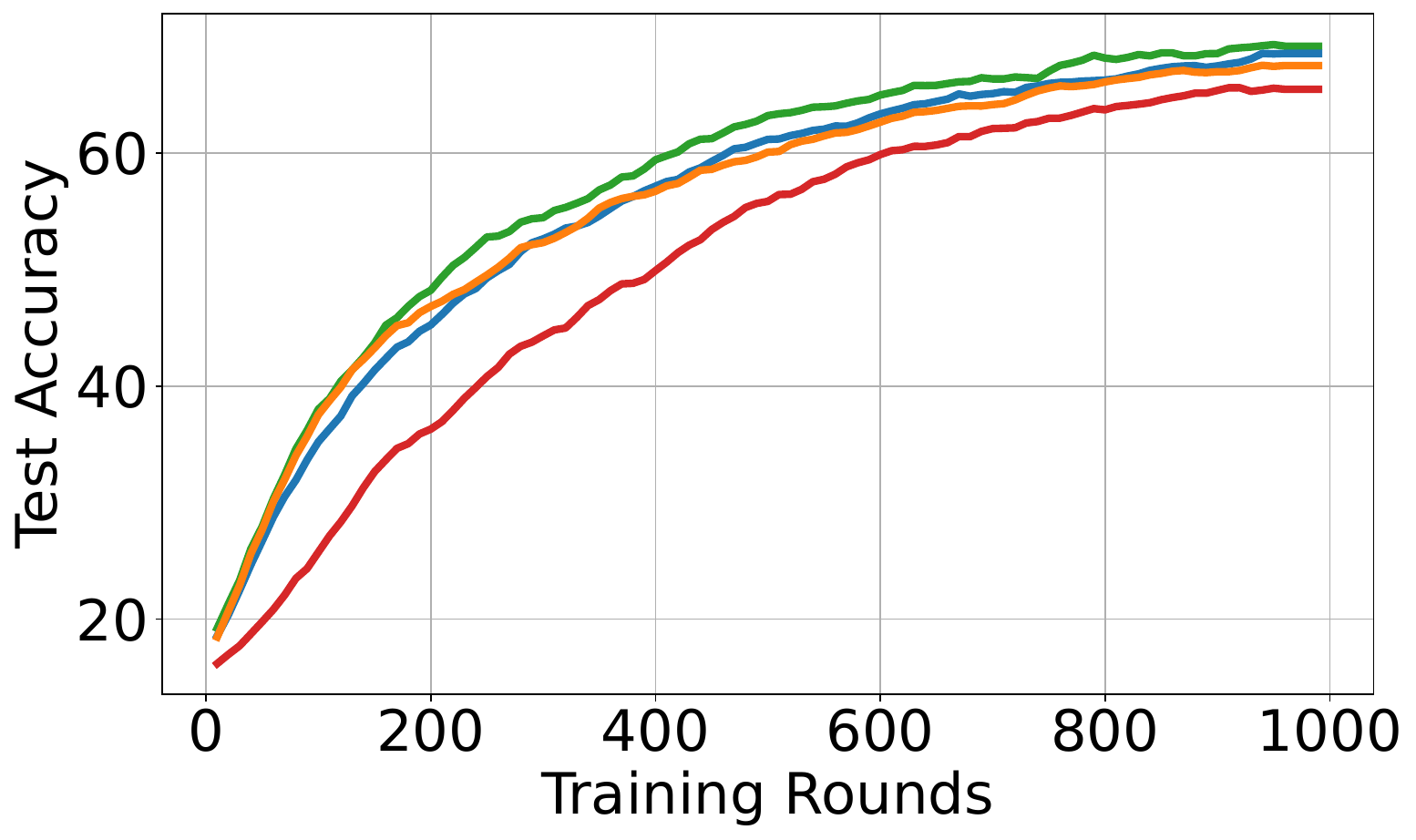}
	\caption{FedScale mapping}
	\label{fig:scaling-real}
     \end{subfigure}
     \hfill
    \begin{subfigure}[t]{0.19\linewidth}
     \includegraphics[width=\linewidth]{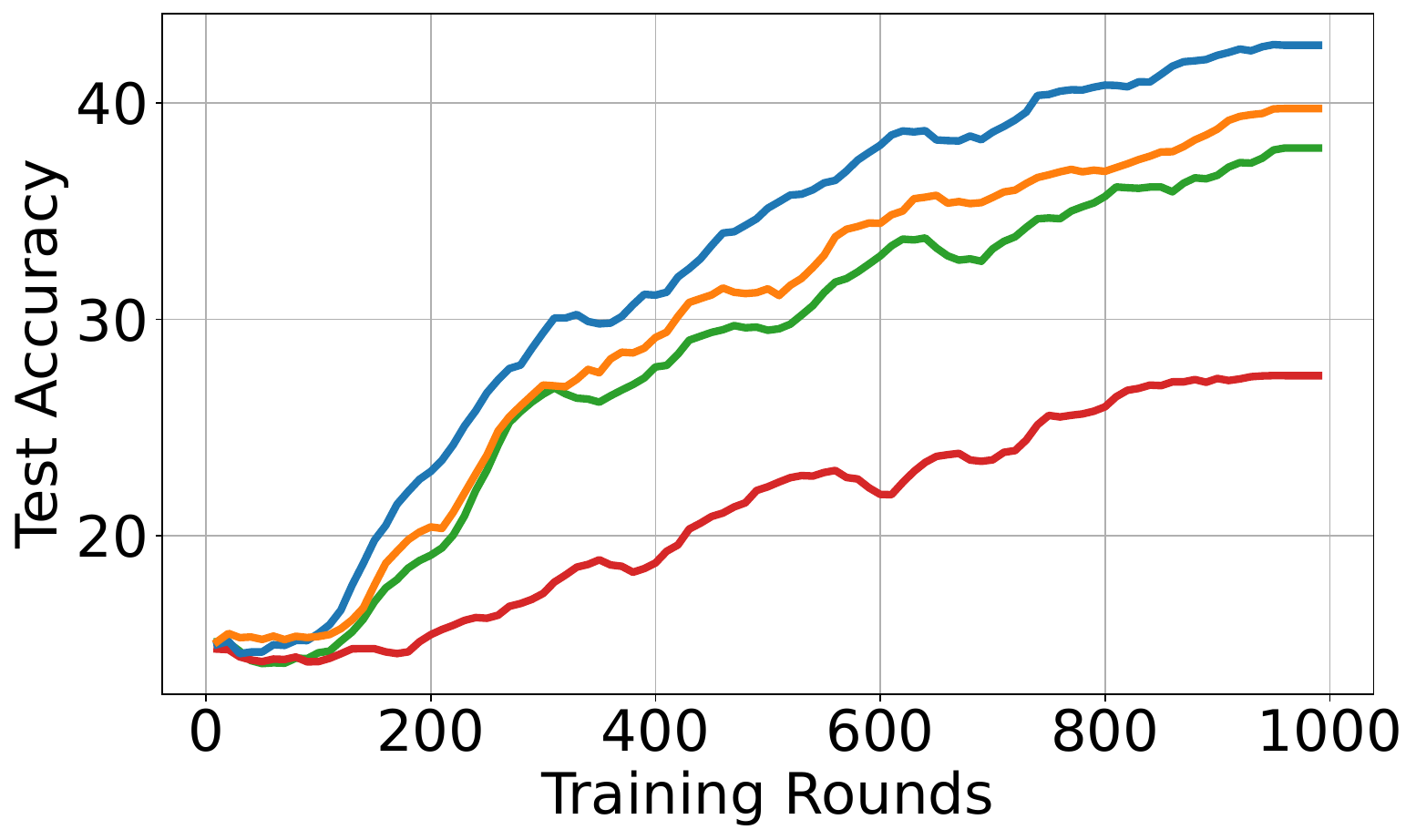} 
	\caption{Label-limited (uniform)}
	\label{fig:scaling-limited-uniform}
     \end{subfigure}
     \hfill
    \begin{subfigure}[t]{0.19\linewidth}
     \includegraphics[width=\linewidth]{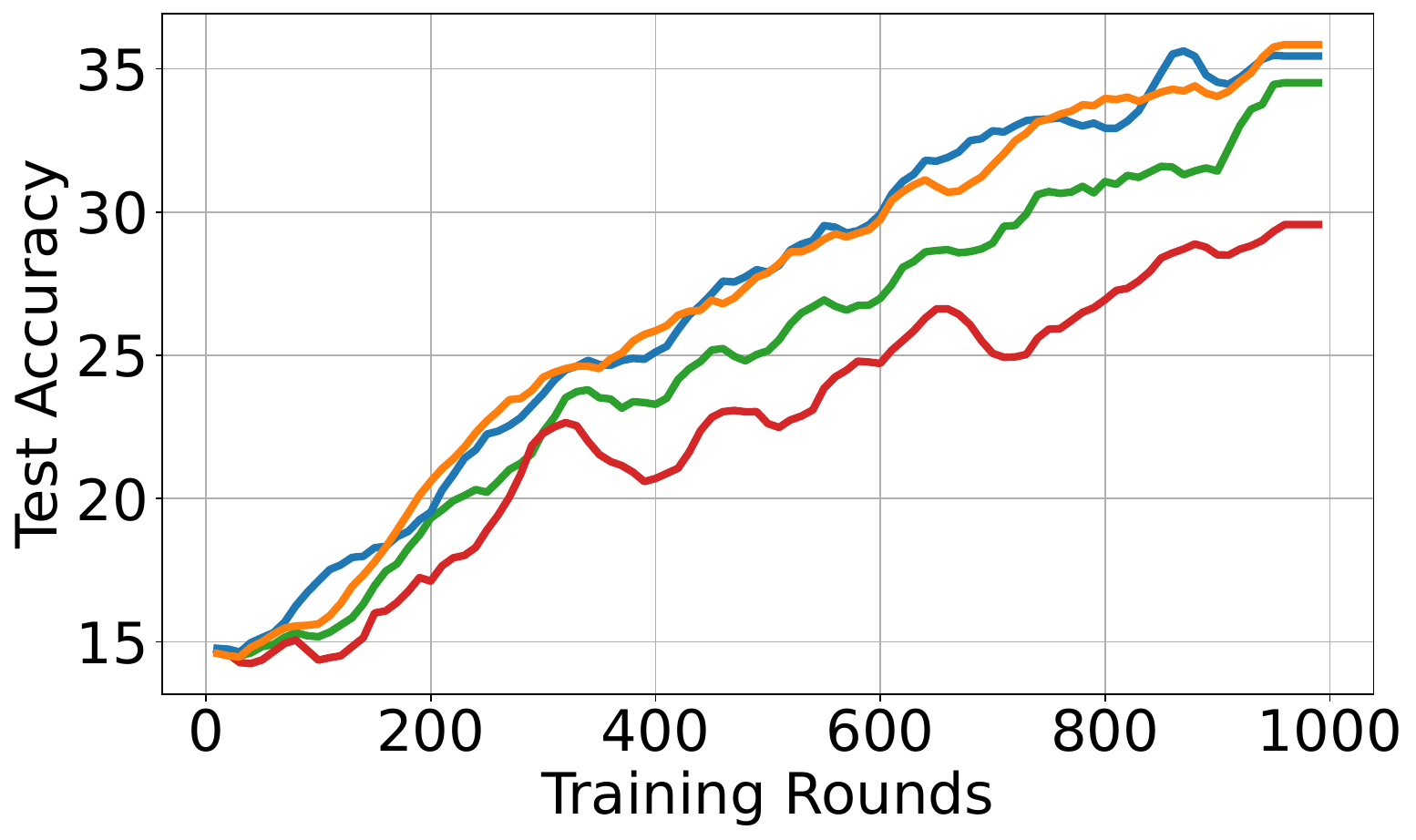} 
	\caption{Label-limited\\ (Zipfian)}
	\label{fig:scaling-limited-zipfian}
     \end{subfigure}
     \hfill
    \begin{subfigure}[t]{0.19\linewidth}
     \includegraphics[width=\linewidth]{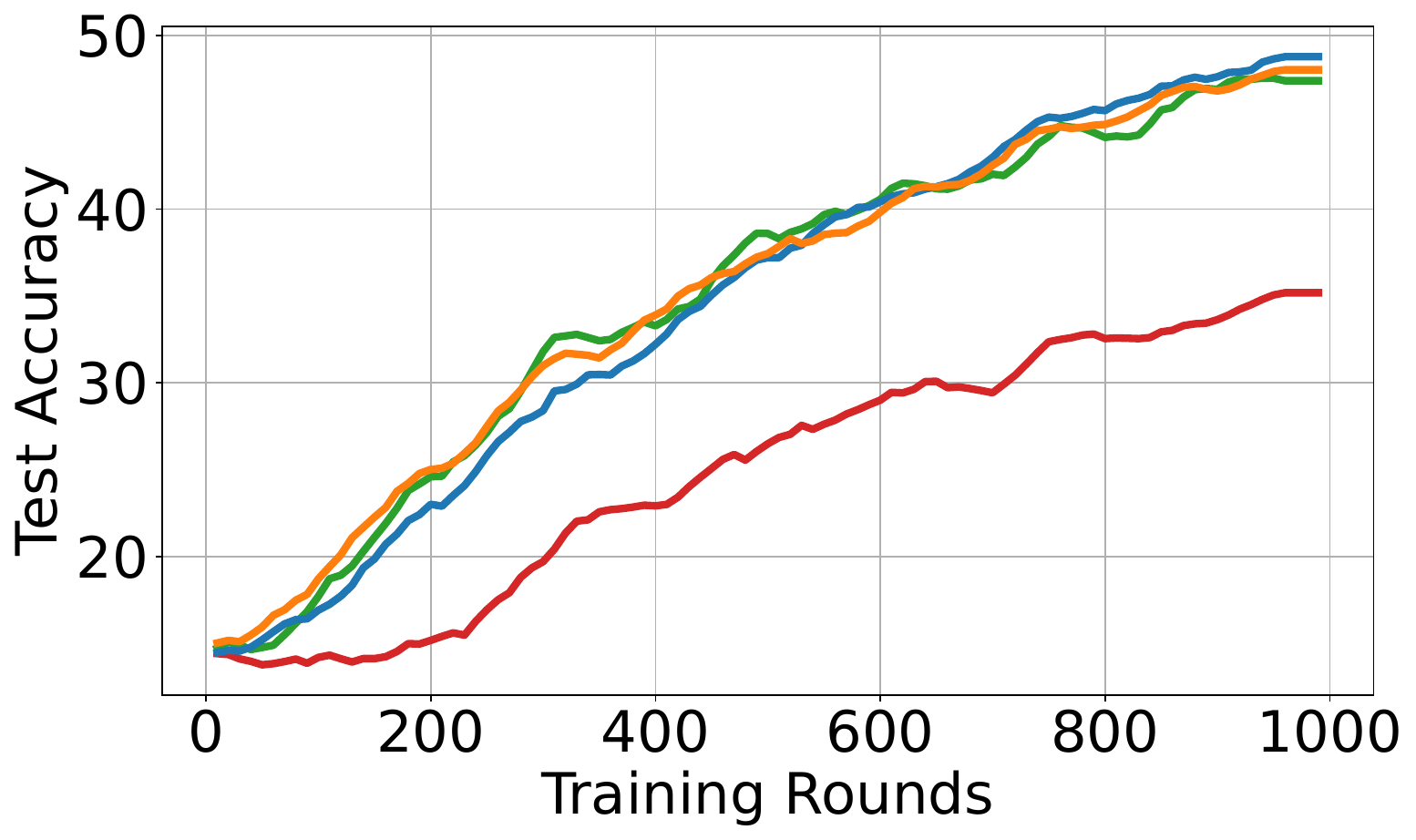} 
	\caption{Label-limited (balanced)}
	\label{fig:scaling-limited-balanced}
     \end{subfigure}
    \\
\caption{Performance of various scaling rules for stale update weighting in     the aggregation step.}
\label{fig:scaling}
\end{figure*}

\begin{table}[!t]
\caption{Baseline performance with centralized training.}
\centering
\resizebox{1\columnwidth}{!}{%
    \begin{tabular}{ccccccc}
    \toprule
        \multirow{3}{*}{Benchmark} & \multirow{3}{*}{Quality Metric} & \multirow{3}{*}{\makecell{Aggregation\\Algorithm}} & \multicolumn{4}{c}{Data to Learner Mapping}
        \\\cmidrule{4-7}
        &   & & \multirow{2}{*}{Uni. Rand.} & \multicolumn{3}{c}{Label-limited} 
      \\\cmidrule{5-7}
        & & & & Uni. Rand. & Zipf Dist. & Balanced
         \\\midrule
        CIFAR10 & Top-5 Test Acc & FedAvg & 90.4 & 86.1 & 76.4 & 86.4
        \\\midrule
        Open Image & Top-5 Test Acc & YoGi & 70.7 & 30.6 & 32.3 & 35.5 %
        \\\midrule
         Google Speech & Top-5 Test Acc. & YoGi & 76.5 & 34.7 & 33.4 & 37.1 %
         \\\midrule
         Reddit & Test Perp. & YoGi & 43.6 & N/A & N/A & N/A
         \\\midrule
         Stackoverflow & Test Perp. & YoGi & 40.2 & N/A & N/A & N/A
         \\\bottomrule
    \end{tabular}
    }
    \label{tab:baselinequality}
\end{table}

\subsubsection{Performance of selection algorithms}
\label{part:comparevsoort}
We use the experimental setting \emph{\bf OC+DynAvail}. We compare \scheme with Oort, Random, and Priority selection. Priority is the IPS component of \scheme (i.e., SAA component is disabled).

\cref{fig:avail-exp-google} shows that, over the FedScale and different non-IID (label-limited) data mappings, with minimal resource usage, \scheme achieves better accuracy over other methods (i.e., Oort, Random, and Priority). %
\scheme achieves superior performance thanks to the availability-based prioritization (\cref{fig:avail-exp1-uniquelearners}) and aggregation of stale updates (\cref{fig:avail-exp1-totalupdates}). When the FL training process is run for more rounds in the label-limited non-IID case, \cref{fig:avail-exp-google-long} shows that \scheme converges to significantly higher accuracy than Oort, in less time and with lower resource usage. %

\subsubsection{Performance of aggregation algorithms} 
\label{part:comparevssafa}
Comparing SAFA and \scheme, we use the \emph{\bf DL+DynAvail} setting with a total learner population of 1,000 and a round deadline of 100s. We use FedAvg as the underlying aggregation algorithm. %
\scheme pre-selects 100 participants and the target ratio is set to $10\%$ and $80\%$ for SAFA and \scheme, respectively. For both schemes, we set the staleness threshold to 5 rounds. %

The results in \cref{fig:safa} show that run times of SAFA and \scheme are comparable, but SAFA consumes significantly more resources. In the case of the FedScale mapping (\cref{fig:safa-real}), the results show that \scheme achieves higher accuracy with $\approx\!20\%$ fewer resources than SAFA. In the non-IID mapping (\cref{fig:safa-limited}), \scheme significantly improves the accuracy by 10 points using $\approx\!60\%$ fewer resources compared to SAFA. %

\begin{figure}[!t]
\captionsetup[subfigure]{justification=centering}
\centering
     \begin{subfigure}[t]{0.49\columnwidth}
     \includegraphics[width=\linewidth]{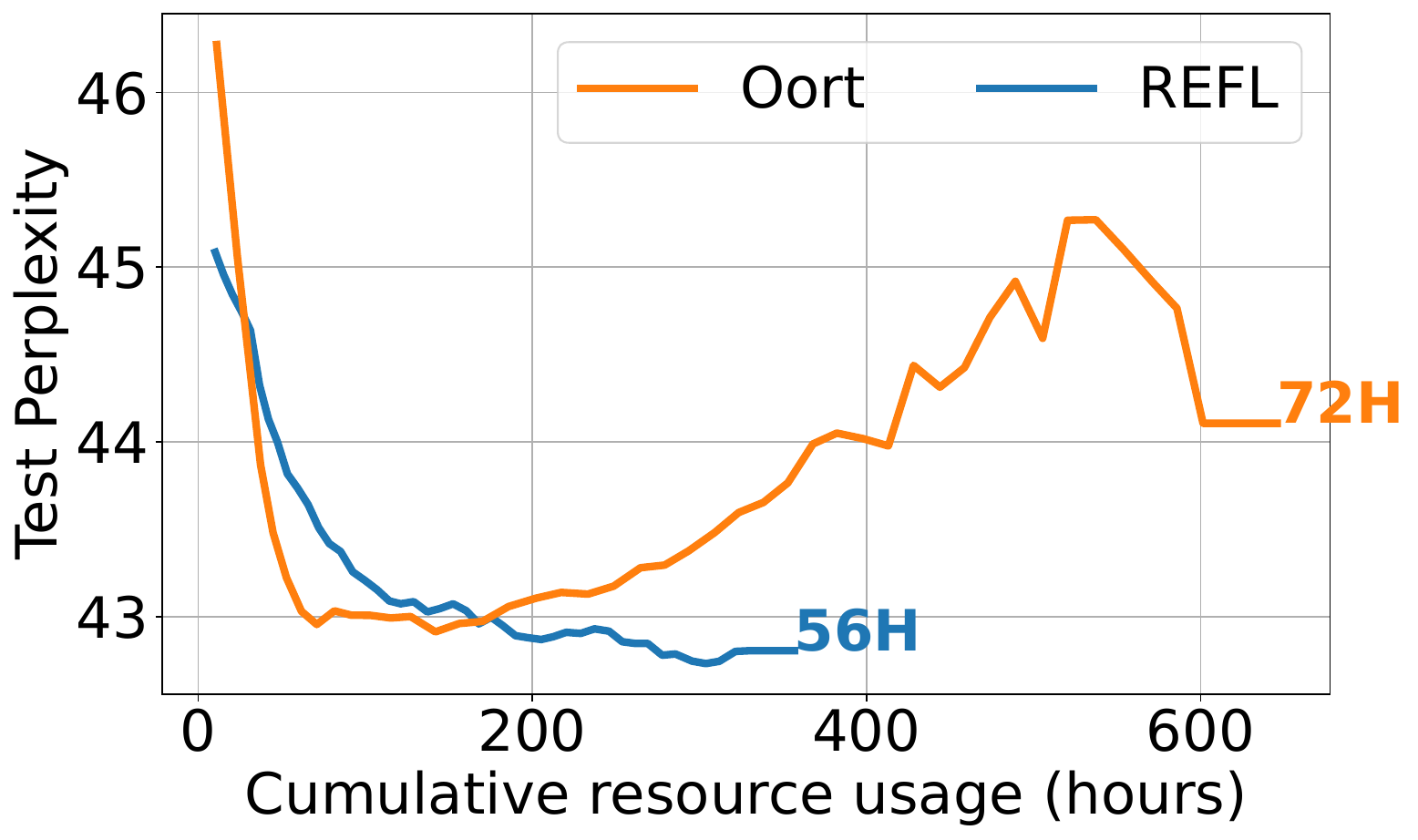}%
	\caption{Reddit}
	\label{fig:avail-exp1-reddit}
     \end{subfigure}
     \hfill
    \begin{subfigure}[t]{0.49\columnwidth}
     \includegraphics[width=\linewidth]{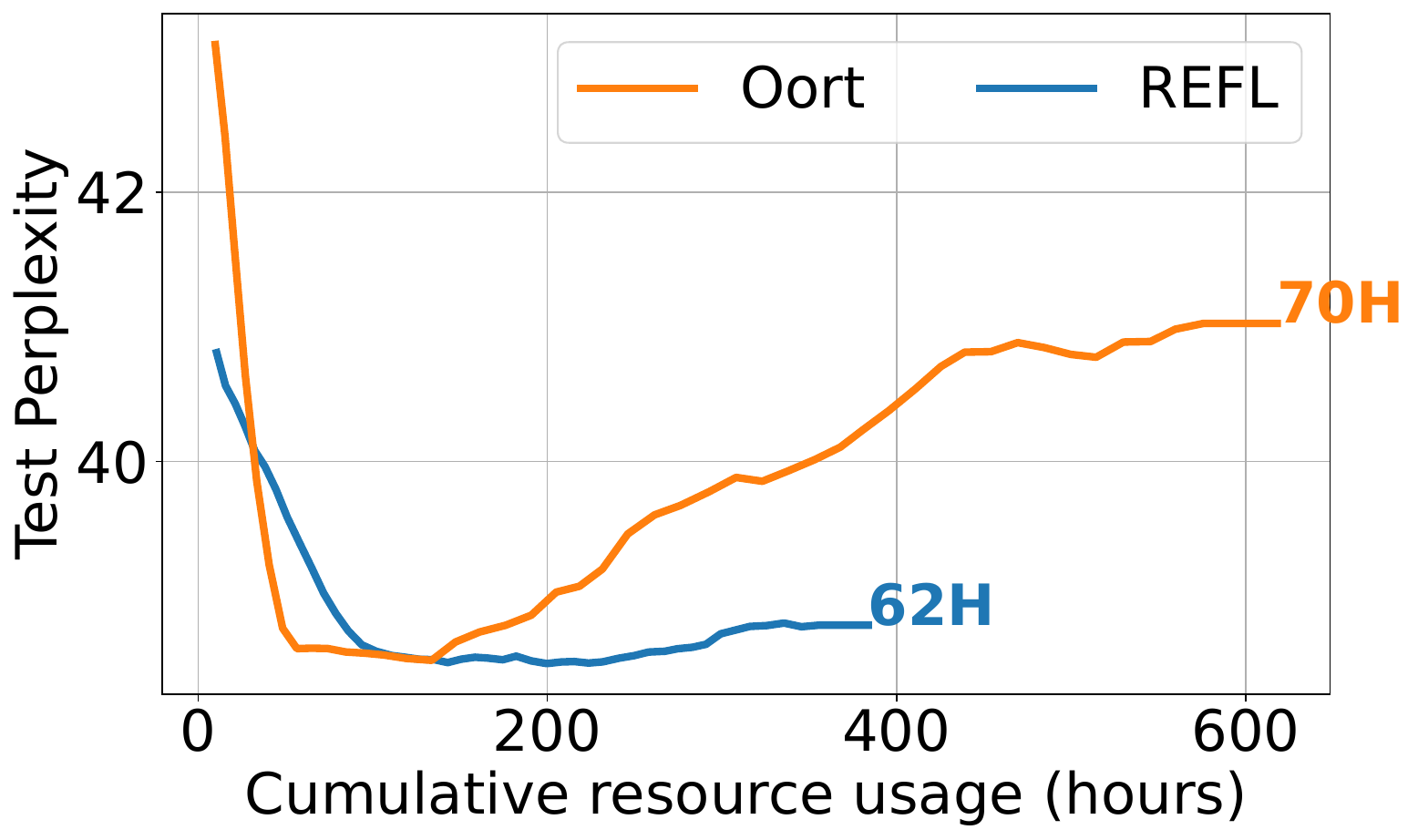}%
	\caption{StackOverFlow}
	\label{fig:avail-exp1-stackoverflow}
     \end{subfigure}
    \\
     \begin{subfigure}[t]{0.49\columnwidth}
     \includegraphics[width=\linewidth]{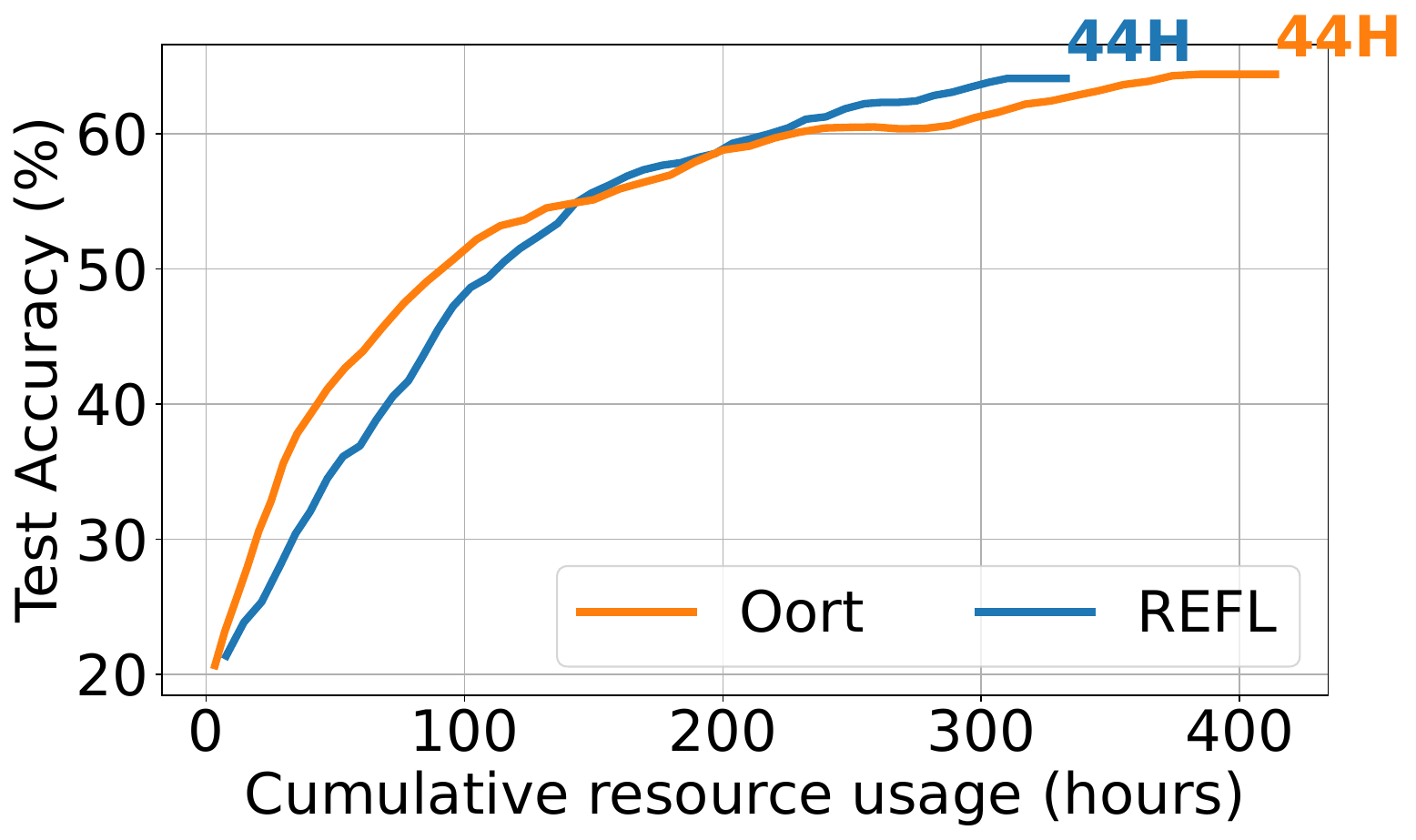}%
	\caption{OpenImage}
	\label{fig:avail-exp1-openimage}
     \end{subfigure}
     \hfill
      \begin{subfigure}[t]{0.49\linewidth}
     \includegraphics[width=\linewidth]{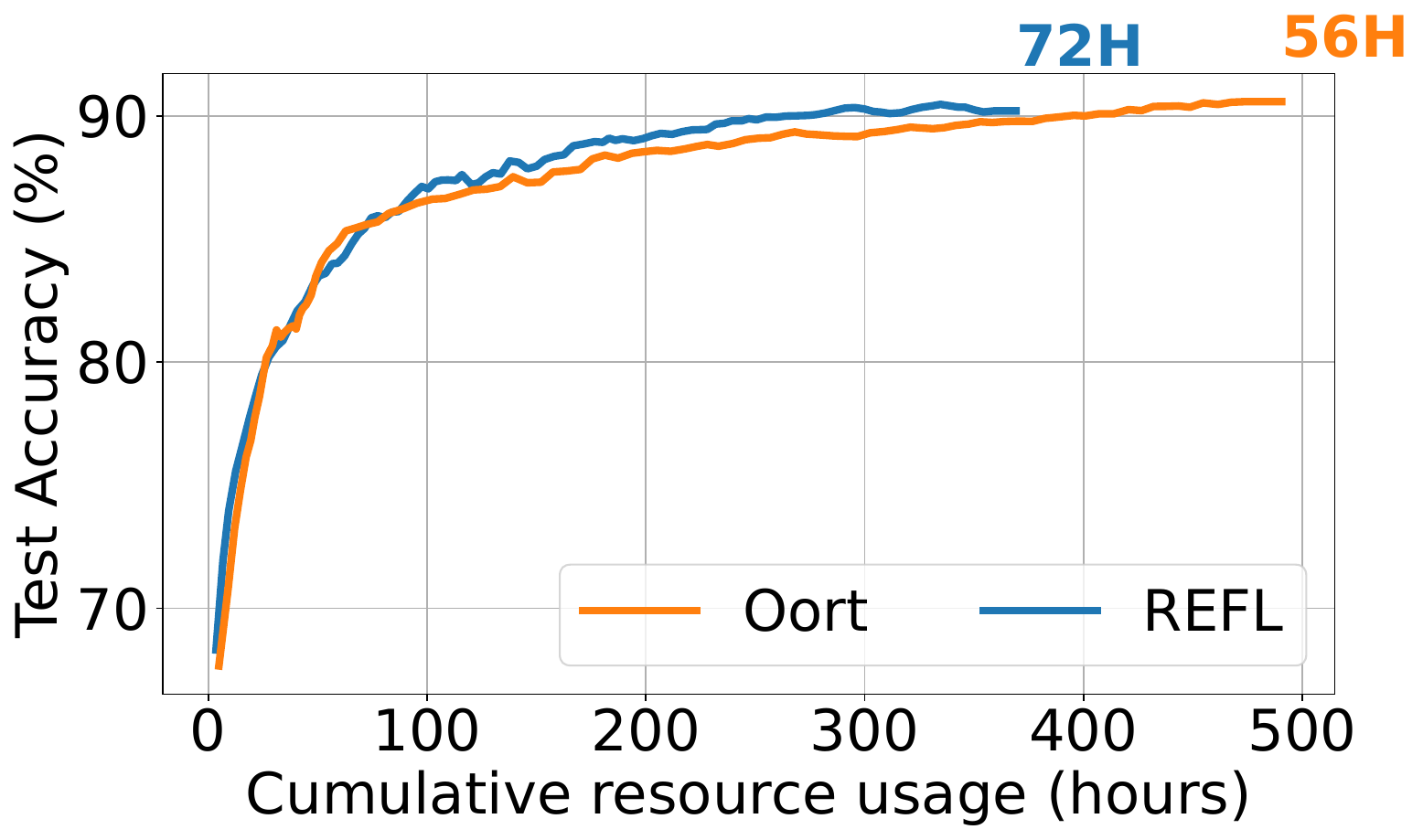}%
	\caption{CIFAR-10}
	\label{fig:avail-exp1-cifar10}
     \end{subfigure}
\caption{Performance for NLP and CNN benchmarks in \emph{\bf OC+DynAvail}. We use the Label-limited mapping for Stackoverflow and Reddit, FedScale data mapping for OpenImage and Uniform IID mapping for CIFAR10.}
\label{fig:avail-other}
\end{figure}

\subsubsection{Availability-based prioritization}
The results in \cref{fig:avail-exp-google} show that Priority selection achieves better model accuracy thanks to prioritizing the least available learners. The results suggest that, especially in non-IID settings, selecting participants with low availability results in a better rate of unique learners with valuable (likely new) data samples, and hence the resource usage to achieve a certain accuracy is also reduced.

\subsubsection{Adaptive target} We run experiments in the \emph{\bf OC} setting with both \emph{\bf DynAvail} and \emph{\bf AllAvail} scenarios using the label-limited (uniform) mapping and 50 participants per round.

\cref{fig:APS} shows that, in both scenarios, \scheme and \scheme+APT have higher model quality with lower resource usage compared to Oort and Random. Moreover, the resource consumption of \scheme can be further reduced with APT by trading off extra run time (i.e., 15H vs. 28H).\footnote{\scheme achieves higher accuracy compared to {\scheme}+APT (i.e., 48\% vs. 42\%) within 15 hours run time but \scheme uses approximately 55\% more resources on average within this time.} Compared to \emph{\bf AllAvail}, the improvements are maintained in \emph{\bf DynAvail} when \scheme prioritizes the least available clients and comparable accuracy is achieved. However, depending on the benchmark, APT may yield no benefits when the target number of participants is small (e.g., $N_0\leq10$). %

\subsubsection{Stale aggregation}
We experiment with \scheme, Oort, and Random in the  \emph{\bf OC+AllAvail} setting.

\cref{fig:stale-exp1-google} shows the achieved test accuracy vs training rounds. We observe that, over different data mappings, \scheme achieves good model quality with lower resource usage thanks to the SAA component. Notably, the benefits are more profound for non-IID distributions. In non-IID settings, the stale updates of delayed participants are more important compared to IID settings. This demonstrates that stale updates can boost the statistical efficiency of training. Also, \scheme achieves run time similar to Random because learners have their availability probability set to 1 (always available), and so \scheme reverts to random selection. %

\subsubsection{Stale weight scaling} We use  \emph{\bf OC+DynAvail} and set the deadline to 100 seconds. We evaluate the weight scaling rules of \S\ref{subsec:scale} and present test accuracy results over the training rounds in \cref{fig:scaling}.

We observe that the \scheme's scaling rule consistently outperforms the other scaling rules in all data distribution scenarios. In the IID cases (\cref{fig:scaling-uniform,fig:scaling-real}), the differences among the scaling rules are small. However, in the non-IID cases (\cref{fig:scaling-limited-uniform,fig:scaling-limited-zipfian,fig:scaling-limited-balanced}), the  performance is inconsistent except for \scheme's rule. These results show the benefits of \scheme's rule for mitigating the potential negative impact of stale updates. The same observations are made for  \emph{\bf OC+AllAvail} with FedAvg.%

\subsubsection{Availability prediction model}
We evaluate the forecasting model.
The model is trained on the Stunner trace~\cite{stunner}, which contains device events from large number of worldwide mobile users. We used devices with at least 1,000 samples (i.e., 137 devices) in the trace collected during September 2018. We extracted the plugged-in and charging state to train a model for each device using the first half of each devices' samples. We compared the devices' model predictions on the remaining samples, which are thus used as the testing dataset.

The results show that the models predict future availability states with high accuracy. The values of the coefficient of determination, mean square error and mean absolute error, averaged across devices, are 0.93, 0.01, and 0.028, respectively.

\subsubsection{Results of other benchmarks}
\label{subsec:moreresults}
We run NLP (Reddit and StackOverFlow) and CV (CIFAR10 and OpenImage) benchmarks in the \emph{\bf OC+DynAvail} setting. We use YoGi as the aggregation algorithm for the Open Image, Reddit and StackOverFlow benchmarks and FedAvg for the CIFAR10 benchmark. Adaptive Participant Target (APT) is enabled for \scheme. We also limit NLP dataset sizes to 20\% (i.e., $\approx$ 8 million samples) as indicated in~\cref{tab:benchmarks}.

The results in \cref{fig:avail-other} directly compare \scheme and Oort. The results for Reddit (\cref{fig:avail-exp1-reddit}) and StackOverFlow (\cref{fig:avail-exp1-stackoverflow}) demonstrate that \scheme achieves considerable reduction in both learners' resource consumption and the final test perplexity compared to Oort. We note that during the initial rounds, the performance is comparable between \scheme and Oort. Then, Oort's low diversity results in divergence which may be attributed to the lack of new samples (or participants with new data which help improve the model convergence). Similarly, the results for OpenImage (\cref{fig:avail-exp1-openimage}) and CIFAR10 (\cref{fig:avail-exp1-cifar10}) show that \scheme achieves the same model accuracy with lower resource consumption compared to Oort in scenarios when the data distribution among the learners is IID or FedScale's data mapping (i.e., closer to IID). %

\section{Discussion}
\label{sec:discussion}

We discuss the implications of \scheme in future scenarios.

\smartparagraph{Large-scale federated learning:} We project that future FL deployments will scale significantly to include learner devices such as sensors, IoT devices, autonomous vehicles, etc., that may not be connected to power sources and have limited computational capabilities. \scheme enables efficient scaling over a larger number of resources, in contrast to FL systems that perform post-training selection (e.g., SAFA) or skew participant selection to fast devices (e.g., Oort). Invoking all devices for training would overwhelm the server and impose significant energy usage by learners, much of which would be wasted.

We show the impact of large populations on resource usage using the Google speech benchmark and 3$\times$ the number of learners (3,000). As shown in \cref{fig:safa-large}, we observe that SAFA wastes many resources in the IID (\cref{fig:safa-large-uniform}) and even more in the non-IID (\cref{fig:safa-large-limited}) setting.

 \begin{figure}[!t]
\captionsetup[subfigure]{justification=centering}
\centering
     \begin{subfigure}[ht]{0.49\linewidth}
     \includegraphics[width=\linewidth]{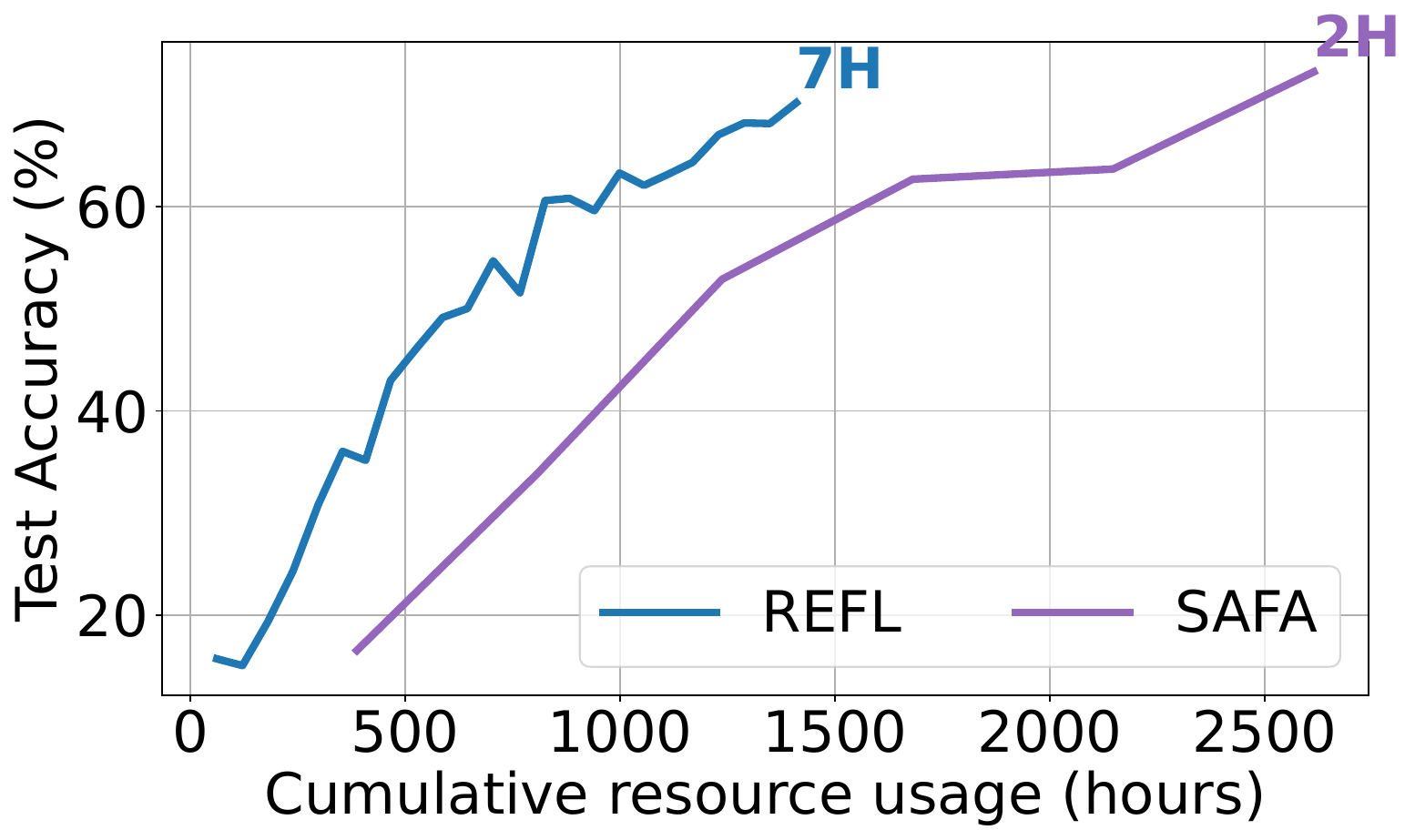}
	\caption{Uniform mapping\\ (IID)}
	\label{fig:safa-large-uniform}
     \end{subfigure}
     \hfill
    \begin{subfigure}[ht]{0.49\linewidth}
     \includegraphics[width=\linewidth]{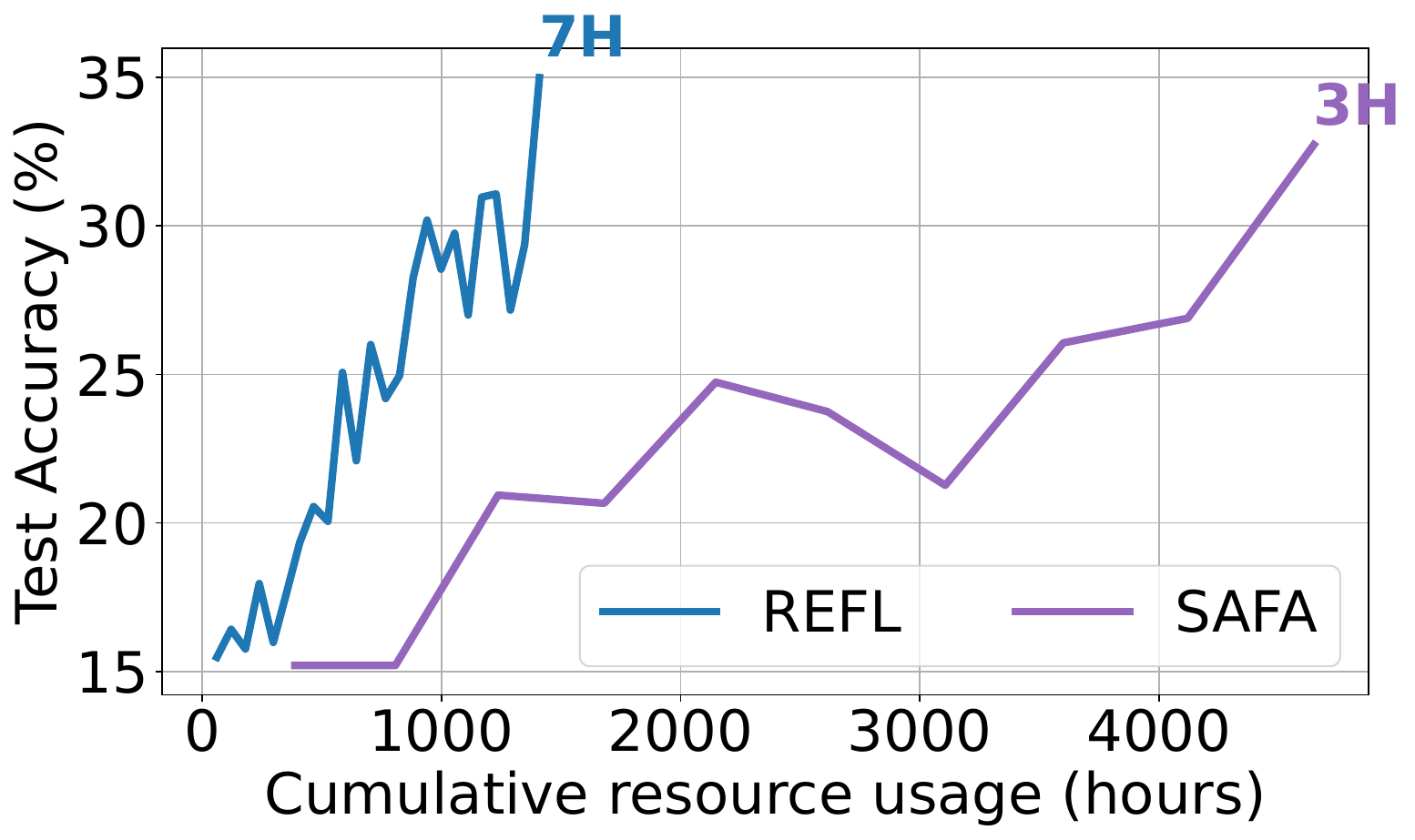} 
	\caption{Label-limited\\ (uniform)}
	\label{fig:safa-large-limited}
     \end{subfigure}
\caption{Resource efficiency in large-scale FL settings.}
\label{fig:safa-large}
\end{figure}

\smartparagraph{Future hardware advancements:} We project that computational capability of devices will continue to improve and so FL systems should benefit from this. Therefore, schemes such as Oort that favor faster learners are likely to see increased under-representation of low-capability learners, resulting in models that do not generalize well over a large population. In contrast, \scheme copes with hardware advancements by benefiting from faster learners without overlooking low-capability learners.

We run the Google Speech benchmark in 4 settings using: current device configurations (\emph{HS1}); device configurations with completion times (i.e., computation and communication) doubled for the top $X$ percentile of devices (where $X$ is 25\% (\emph{HS2}), 75\% (\emph{HS3}), and 100\% (\emph{HS4})). As shown in \cref{fig:oort-hardware-large-uniform,fig:relay-hardware-large-uniform}, we observe both Oort and \scheme benefit from hardware enhancements in IID settings. However, as shown in \cref{fig:oort-hardware-large-limited}, with realistic label-limited non-IID settings, \scheme sees significant performance benefits due to the aggregation of stale updates and higher participant diversity. Oort does not benefit from improved speeds because the selection still favors the faster learners, which reduces training time but not model quality.

\smartparagraph{Implications of the availability prediction model:}
When new learners join the FL process, they may not have traces with which to train their availability prediction model, resulting in low confidence in their predictions. Moreover, malicious or adversarial laerners may attempt to influence the system by consistently reporting low future availability. %
To deal with these scenarios, and as recommended by~\cite{Bonawitz19}, we use in \cref{subsec:ips} a filtering mechanism that prevents a participant to be reselected in the following $X$ rounds (in our experiments 5 rounds).

\begin{figure}[!t]
\captionsetup[subfigure]{justification=centering}
\centering
 \includegraphics[width=0.8\linewidth]{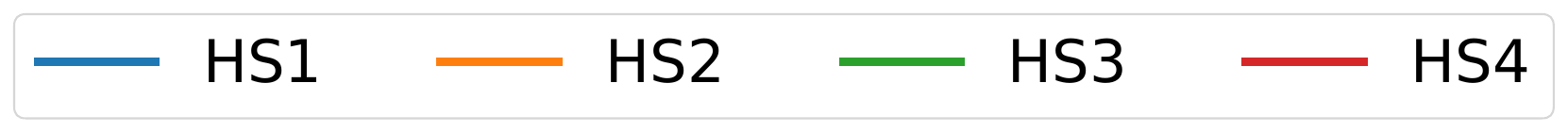}
    \\  
     \begin{subfigure}[ht]{0.49\linewidth}
     \includegraphics[width=\linewidth]{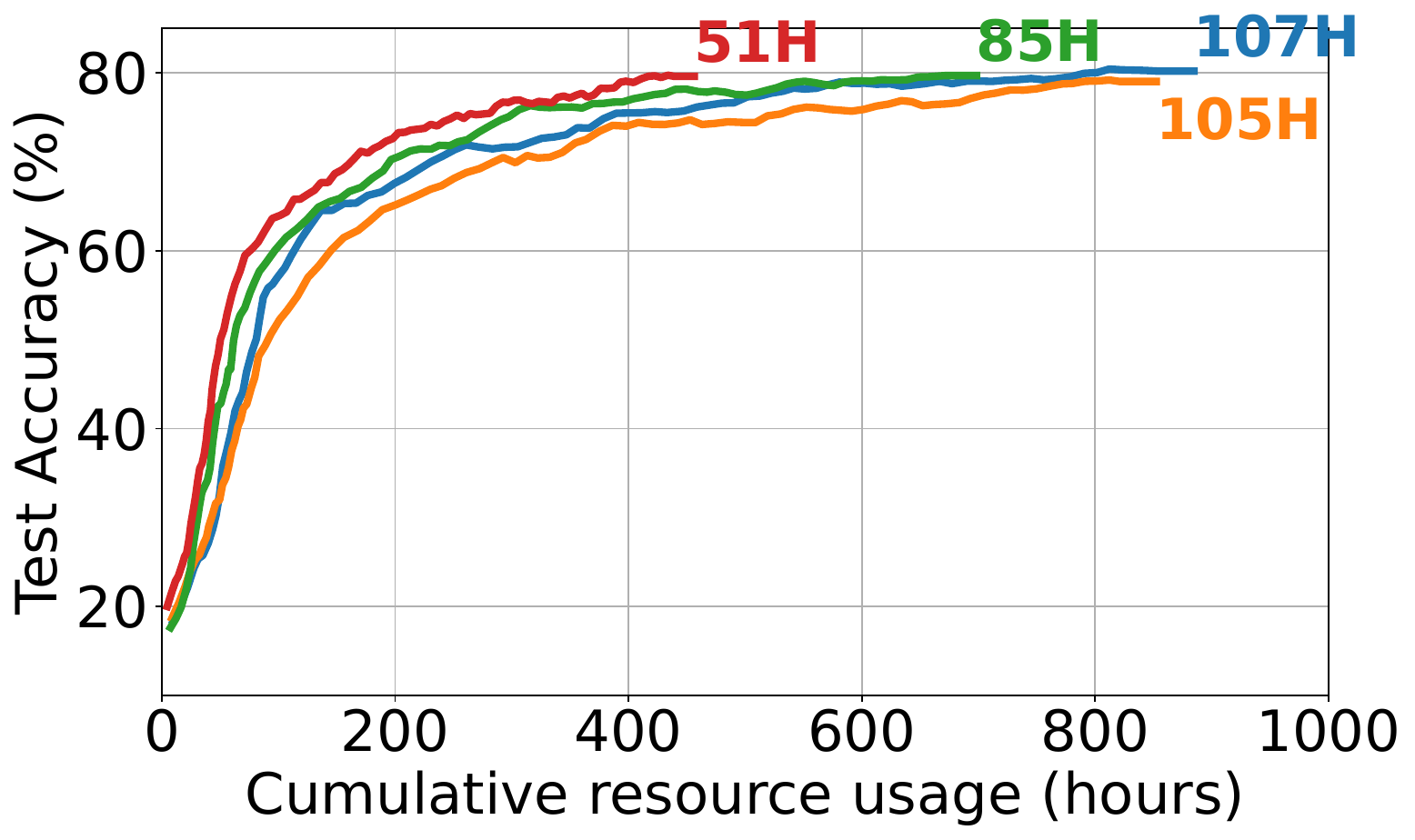}
	\caption{Oort - Uniform (IID)}
	\label{fig:oort-hardware-large-uniform}
     \end{subfigure}
     \hfill
    \begin{subfigure}[ht]{0.49\linewidth}
     \includegraphics[width=\linewidth]{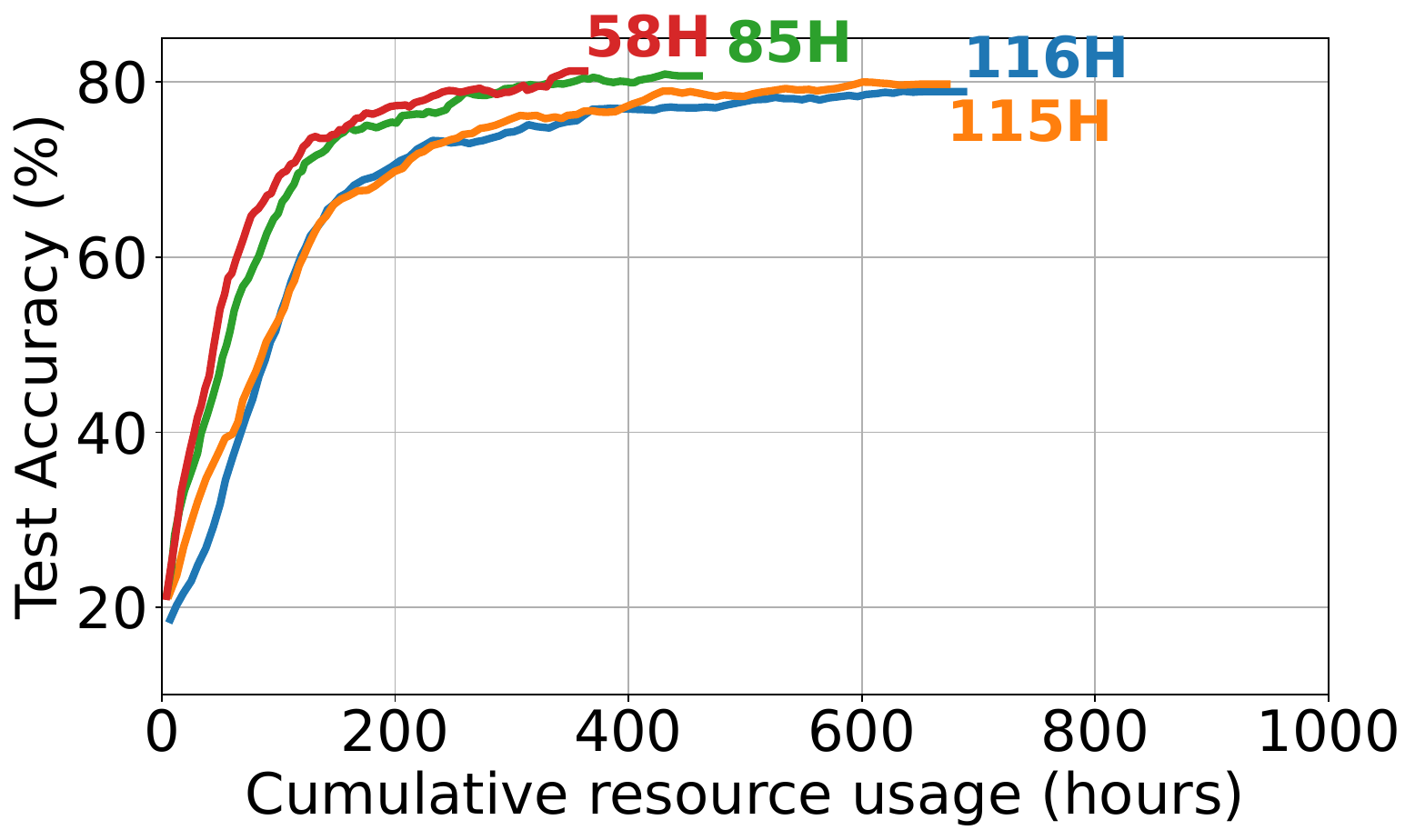} 
	\caption{\scheme - Uniform (IID)}
	\label{fig:relay-hardware-large-uniform}
     \end{subfigure}
     \\ 
     \begin{subfigure}[ht]{0.49\linewidth}
     \includegraphics[width=\linewidth]{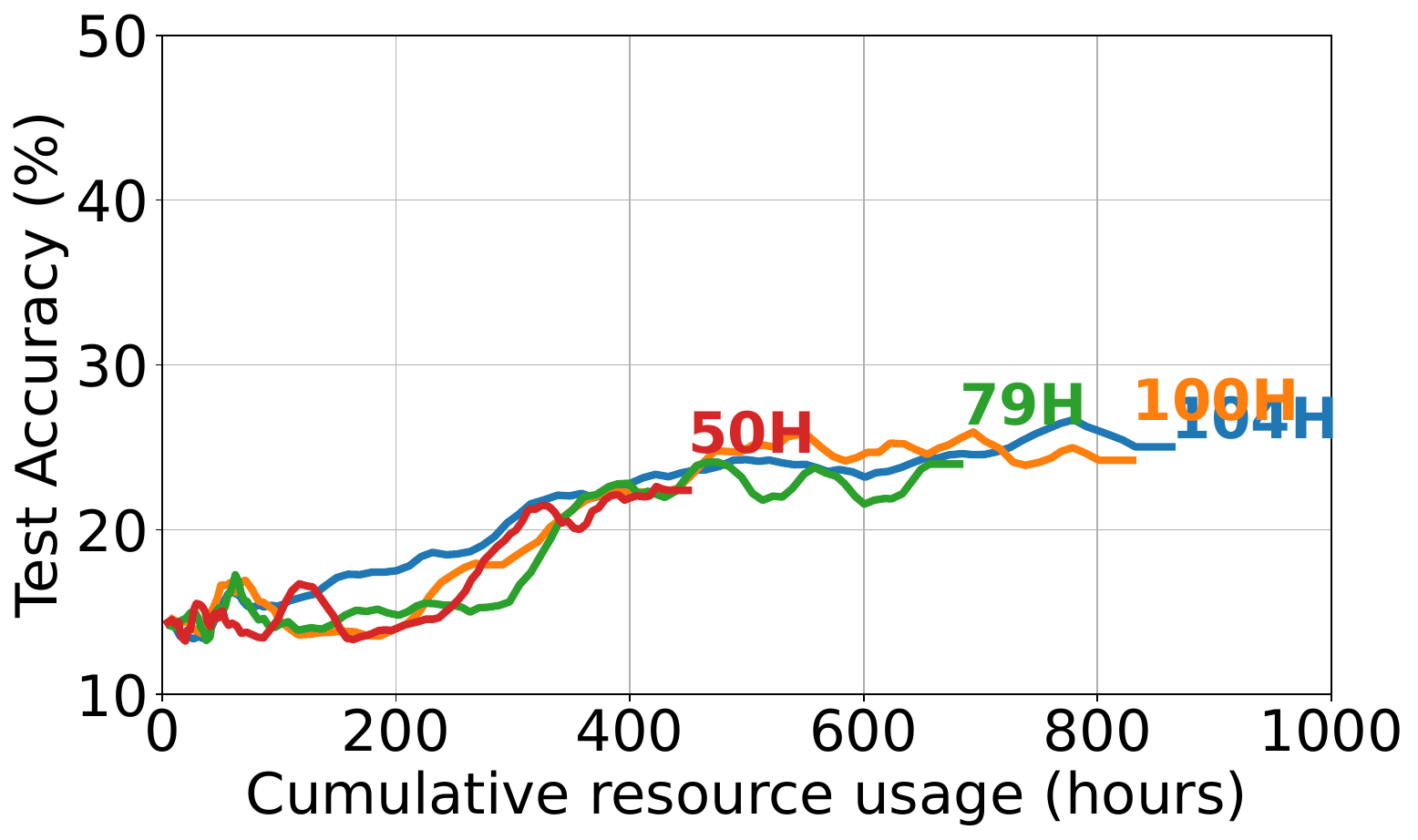} 
	\caption{Oort - Label-limited uniform}
	\label{fig:oort-hardware-large-limited}
     \end{subfigure}
     \hfill
      \begin{subfigure}[ht]{0.49\linewidth}
     \includegraphics[width=\linewidth]{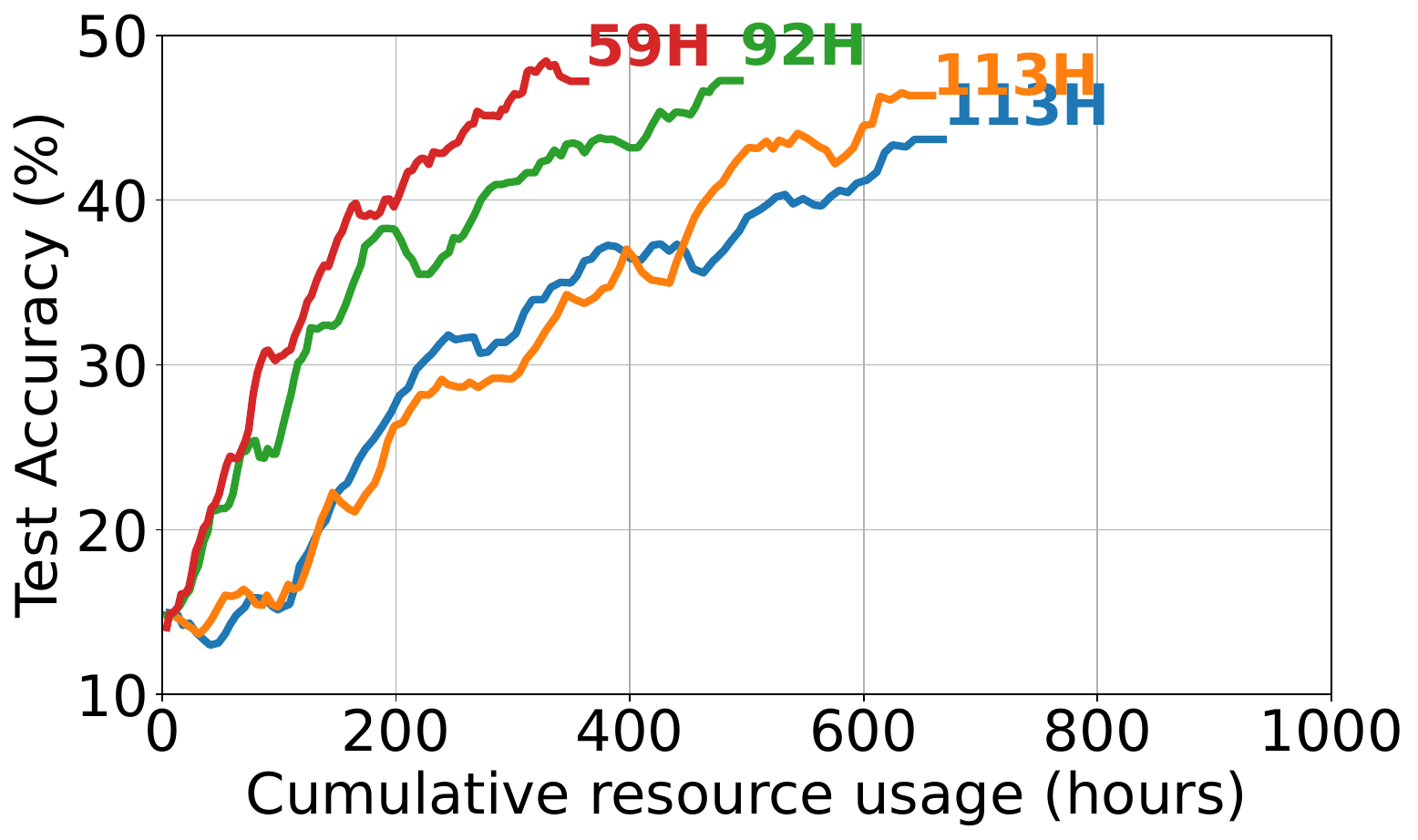} 
	\caption{\scheme - Label-limited uniform}
	\label{fig:relay-hardware-large-limited}
     \end{subfigure}
\caption{Impact of future hardware advancements.} %
\label{fig:hardware-advance}
\end{figure}

\section{Integration with FL Frameworks}
\label{sec:integrate}
The design of \scheme is lightweight and can operate as an online service or a plug-in module for existing FL frameworks. Therefore, \scheme suits large-scale FL deployments dealing with likely thousands to millions of end-devices.

{\bf \scheme selects participants as follows:}
\begin{inparaenum}[1)]
\item First, the server updates its estimate of round duration $\mu_t$ and send an estimate of the time period $a=(\mu_t, 2\mu_t)$ of the next round to the learners;
\item learners maintain a local trace of their charging events and periodically train the forecasting model, which produces a probabilistic value for their charging state during future time periods.
\item upon receiving an availability query $a$, each the learner $l$ uses the prediction model to produce its availability probability $p_l(a)$ during time period $a$ and sends $p_l(a)$ to the server;
\item the server collects the probabilities $P_t$ and selects the participants using Algorithm~\cref{algo:priority}; and 
\item the server sends each of the selected participants a random hash ID which encodes a time-stamp of the current round as well as the FL task (e.g., the model) and relevant parameter configuration. 
\end{inparaenum}

{\bf \scheme handles stale updates as follows:} 
\begin{inparaenum}[i)]
    \item The server collects the updates which are received before the end of current training round $t$. If the time-stamp of a received update's hash ID does not match the the current round, it is categorized as a stale update; 
    \item at the end of the round during aggregation, the server aggregates the fresh updates first to produce $\hat{u}_\cF$.
    \item for each stale update, $u_s$, the server computes the level of staleness $\tau_t$ using the timestamp of the stale update $\hat{u}_\cF$;  
    \item for each stale update, the server computes the deviation of the stale update from the fresh updates $\Lambda_t$ and uses the proposed rule in \cref{eq:hybrid} to assign the scaling weight $w_s$ to the stale update; and
    \item the server aggregates the scaled stale updates with the aggregated fresh updates to produce the new model using \cref{algo:ssfedavg}.
\end{inparaenum}

To integrate \scheme with PySyft, minimal exchanges between the server and learners are needed at the selection stage. The server sends an estimate time-slot of the next training round. The learners use the forecasting model and send their availability probability. Therefore, the learners  do not need to exchange any sensitive information about their data. The FL developer programs the client-side to train the forecasting model and respond to the availability query from the server which poses minimal memory and communication overhead. \scheme can also run as a distributed service using a communication library (e.g., XML-RPC~\cite{xmlrpc}) to establish the communication channel between the \scheme process and the server. The server can share metadata from the participants with the \scheme service. The server can use the PySyft API \textsf{model.send(participant\_id)} to invoke the participants selected by \scheme, and \textsf{model.get(participant\_id)} to collect the model and metadata updates from the participant.

\section{Related Work}
\label{sec:related}

\smartparagraph{Federated learning:} FL is commonly viewed as a ML paradigm wherein a server distributes the training process on a set of decentralized participants that train a shared global model using local data that is never communicated with other entities~\cite{mcmahan2017,kairouz2019advances}. FL has been used to enhance prediction quality for virtual keyboards among other applications~\cite{Bonawitz19,yang2018applied}. A number of FL frameworks have facilitated research in this area~\cite{caldas2018leaf,ryffel2018generic,paddle,tff,yang2020heterogeneityaware,lai2021fedscale}. Flash~\cite{yang2020heterogeneityaware} extended Leaf~\cite{caldas2018leaf} to incorporate heterogeneity-related parameters. FedScale~\cite{lai2021fedscale} enables FL experimentation using a diverse set of challenging and realistic benchmark datasets; we use it as the emulation framework in this work.

\smartparagraph{Participant selection strategies:} In each round, the FL server samples among online learners and trains the global model on the selected participants (e.g., 10s of learners) among those currently online (e.g., 1,000s of learners). %
A number of recent works have proposed enhanced participant selection strategies. Biasing the selection process towards learners with fast hardware and network speeds has been proposed~\cite{Nishio2018}. Other work has sought to enhance statistical efficiency by selecting participants with better model updates~\cite{pmlr-v130-ruan21a,Chen2020,cho2020client}. Recently, Oort~\cite{Oort-osdi21} proposed a strategy that combines both system and statistical efficiency. As we demonstrate in this work, these approaches either result in wasted computation or low coverage of the learners.

\smartparagraph{Heterogeneity in FL:} A significant challenge facing wider adoption of FL systems at scale is uncertainties in system behavior due to learner, system, and data heterogeneity. Learners' computational capacity can restrict contributions and extend round duration~\cite{Li2020FedProx,yang2020heterogeneityaware,Li2021a,Ahmed-EuroMLSys-22}. Architectural and algorithmic solutions to tackle heterogeneity have been proposed~\cite{Ahmed-AQFL-21,wang2020tackling,Li2020FedProx,Li2021,Oort-osdi21}. Heterogeneity in FL is particularly challenging because participants have varying data distributions and availability, as well as heterogeneous system configurations that cannot be controlled~\cite{yang2018applied,Bonawitz19,kairouz2019advances,Ahmed-EuroMLSys-22}.

\smartparagraph{FL proposals:} 
Broader improvements in FL systems have included reducing communication costs~\cite{Jakub2016,Smith2017,Bonawitz19,Chen2019CommunicationEfficientFD,pmlr-v108-reisizadeh20a}, improving privacy guarantees~\cite{McMahan2018,Melis2019,Bonawitz19,Nasr2019,Bagdasaryan20}, compensating for partial work~\cite{Li2020FedProx,wang2020tackling}, minimizing energy consumption on edge devices~\cite{Li2019smartpc}, and personalizing global models trained by participants~\cite{jiang2019improving}. Recent works have sought to address the challenge of data heterogeneity~\cite{Mohri2019AgnosticFL,Li2020Fair,Li2020FedProx}. 

Our work complements these efforts aiming to optimize the FL ecosystem. We aim to produce a resource-efficient FL framework making better use of learners' resources to achieve target model quality without stretching training time. \scheme's design can easily benefit from the existing techniques for secure aggregation or differential privacy.

\section{Conclusion}
We studied two key issues preventing the wider adoption of FL systems: resource wastage and low data diversity. We proposed \scheme that addresses these issues through two core components that encompass novel selection and aggregation algorithms. Compared to existing systems, \scheme is shown, both theoretically and empirically, to improve model quality while reducing resource usage with low impact on training time. \scheme is a vital step towards establishing a novel and practical ecosystem for resource-efficient federated learning.

\begin{acks}
We thank our shepherd, Somali Chaterji, and the anonymous reviewers for their feedback.
We also thank the Artifact Evaluation Committee for their efforts.
This publication is based upon work supported by the \grantsponsor{KAUST-ORA}{King Abdullah University of Science and Technology (KAUST) Office of Research Administration (ORA)}~~under Award No. \grantnum{KAUST-ORA}{ORA-CRG2021-4699}.
For computer time, this research used the resources of the Supercomputing Laboratory at KAUST.
\end{acks}

\bibliographystyle{ACM-Reference-Format}
\bibliography{main.bib}


\begin{thebibliography}{70}


\ifx \showCODEN    \undefined \def \showCODEN     #1{\unskip}     \fi
\ifx \showDOI      \undefined \def \showDOI       #1{#1}\fi
\ifx \showISBNx    \undefined \def \showISBNx     #1{\unskip}     \fi
\ifx \showISBNxiii \undefined \def \showISBNxiii  #1{\unskip}     \fi
\ifx \showISSN     \undefined \def \showISSN      #1{\unskip}     \fi
\ifx \showLCCN     \undefined \def \showLCCN      #1{\unskip}     \fi
\ifx \shownote     \undefined \def \shownote      #1{#1}          \fi
\ifx \showarticletitle \undefined \def \showarticletitle #1{#1}   \fi
\ifx \showURL      \undefined \def \showURL       {\relax}        \fi
\providecommand\bibfield[2]{#2}
\providecommand\bibinfo[2]{#2}
\providecommand\natexlab[1]{#1}
\providecommand\showeprint[2][]{arXiv:#2}

\bibitem[Abdelmoniem(2022)]%
        {refl-repo}
\bibfield{author}{\bibinfo{person}{Ahmed~M. Abdelmoniem}.}
  \bibinfo{year}{2022}\natexlab{}.
\newblock \bibinfo{booktitle}{\emph{{This Paper's Artifacts Repository}}}.
\newblock
\urldef\tempurl%
\url{https://doi.org/10.5281/zenodo.7141105}
\showDOI{\tempurl}


\bibitem[Abdelmoniem and Canini(2021)]%
        {Ahmed-AQFL-21}
\bibfield{author}{\bibinfo{person}{Ahmed~M. Abdelmoniem} {and}
  \bibinfo{person}{Marco Canini}.} \bibinfo{year}{2021}\natexlab{}.
\newblock \showarticletitle{Towards Mitigating Device Heterogeneity in
  Federated Learning via Adaptive Model Quantization}. In
  \bibinfo{booktitle}{\emph{EuroMLSys}}.
\newblock


\bibitem[Abdelmoniem et~al\mbox{.}(2022)]%
        {Ahmed-EuroMLSys-22}
\bibfield{author}{\bibinfo{person}{Ahmed~M. Abdelmoniem},
  \bibinfo{person}{Chen-Yu Ho}, \bibinfo{person}{Pantelis Papageorgiou}, {and}
  \bibinfo{person}{Marco Canini}.} \bibinfo{year}{2022}\natexlab{}.
\newblock \showarticletitle{Empirical Analysis of Federated Learning in
  Heterogeneous Environments}. In \bibinfo{booktitle}{\emph{EuroMLSys}}.
\newblock


\bibitem[Bagdasaryan et~al\mbox{.}(2020)]%
        {Bagdasaryan20}
\bibfield{author}{\bibinfo{person}{Eugene Bagdasaryan},
  \bibinfo{person}{Andreas Veit}, \bibinfo{person}{Yiqing Hua},
  \bibinfo{person}{Deborah Estrin}, {and} \bibinfo{person}{Vitaly Shmatikov}.}
  \bibinfo{year}{2020}\natexlab{}.
\newblock \showarticletitle{{How To Backdoor Federated Learning}}. In
  \bibinfo{booktitle}{\emph{AISTATS}}.
\newblock


\bibitem[Benchmark(2021)]%
        {AIranking}
\bibfield{author}{\bibinfo{person}{AI Benchmark}.}
  \bibinfo{year}{2021}\natexlab{}.
\newblock \bibinfo{title}{Performance Ranking}.
\newblock
\newblock
\urldef\tempurl%
\url{https://ai-benchmark.com/ranking.html}
\showURL{%
\tempurl}


\bibitem[Bonawitz et~al\mbox{.}(2019a)]%
        {Bonawitz19}
\bibfield{author}{\bibinfo{person}{Keith Bonawitz}, \bibinfo{person}{Hubert
  Eichner}, \bibinfo{person}{Wolfgang Grieskamp}, \bibinfo{person}{Dzmitry
  Huba}, \bibinfo{person}{Alex Ingerman}, \bibinfo{person}{Vladimir Ivanov},
  \bibinfo{person}{Chlo\'{e} Kiddon}, \bibinfo{person}{Jakub Kone\v{c}n\'{y}},
  \bibinfo{person}{Stefano Mazzocchi}, \bibinfo{person}{Brendan McMahan},
  \bibinfo{person}{Timon Van~Overveldt}, \bibinfo{person}{David Petrou},
  \bibinfo{person}{Daniel Ramage}, {and} \bibinfo{person}{Jason Roselander}.}
  \bibinfo{year}{2019}\natexlab{a}.
\newblock \showarticletitle{{Towards Federated Learning at Scale: System
  Design}}. In \bibinfo{booktitle}{\emph{MLSys}}.
\newblock


\bibitem[Bonawitz et~al\mbox{.}(2017)]%
        {Keith2018}
\bibfield{author}{\bibinfo{person}{Keith Bonawitz}, \bibinfo{person}{Vladimir
  Ivanov}, \bibinfo{person}{Ben Kreuter}, \bibinfo{person}{Antonio Marcedone},
  \bibinfo{person}{H.~Brendan McMahan}, \bibinfo{person}{Sarvar Patel},
  \bibinfo{person}{Daniel Ramage}, \bibinfo{person}{Aaron Segal}, {and}
  \bibinfo{person}{Karn Seth}.} \bibinfo{year}{2017}\natexlab{}.
\newblock \showarticletitle{{Practical Secure Aggregation for
  Privacy-Preserving Machine Learning}}. In \bibinfo{booktitle}{\emph{CCS}}.
\newblock


\bibitem[Bonawitz et~al\mbox{.}(2019b)]%
        {Bonawitz19secure}
\bibfield{author}{\bibinfo{person}{Keith Bonawitz}, \bibinfo{person}{Fariborz
  Salehi}, \bibinfo{person}{Jakub Kone{\v{c}}n{\'y}}, \bibinfo{person}{Brendan
  McMahan}, {and} \bibinfo{person}{Marco Gruteser}.}
  \bibinfo{year}{2019}\natexlab{b}.
\newblock \showarticletitle{Federated Learning with Autotuned
  Communication-Efficient Secure Aggregation}. In
  \bibinfo{booktitle}{\emph{53rd Asilomar Conference on Signals, Systems, and
  Computers}}.
\newblock


\bibitem[Caldas et~al\mbox{.}(2019)]%
        {caldas2018leaf}
\bibfield{author}{\bibinfo{person}{Sebastian Caldas}, \bibinfo{person}{Sai
  Meher~Karthik Duddu}, \bibinfo{person}{Peter Wu}, \bibinfo{person}{Tian Li},
  \bibinfo{person}{Jakub Konečný}, \bibinfo{person}{H.~Brendan McMahan},
  \bibinfo{person}{Virginia Smith}, {and} \bibinfo{person}{Ameet Talwalkar}.}
  \bibinfo{year}{2019}\natexlab{}.
\newblock \showarticletitle{{LEAF: A Benchmark for Federated Settings}}. In
  \bibinfo{booktitle}{\emph{Workshop on Federated Learning for Data Privacy and
  Confidentiality}}.
\newblock


\bibitem[Chen et~al\mbox{.}(2020a)]%
        {Chen2020}
\bibfield{author}{\bibinfo{person}{Wenlin Chen}, \bibinfo{person}{Samuel
  Horv{\'{a}}th}, {and} \bibinfo{person}{Peter Richt{\'{a}}rik}.}
  \bibinfo{year}{2020}\natexlab{a}.
\newblock \bibinfo{title}{{Optimal Client Sampling for Federated Learning}}.
\newblock
\newblock
\showeprint[arxiv]{2010.13723}~[cs.LG]


\bibitem[Chen et~al\mbox{.}(2020b)]%
        {Chen2019CommunicationEfficientFD}
\bibfield{author}{\bibinfo{person}{Yang Chen}, \bibinfo{person}{Xiaoyan Sun},
  {and} \bibinfo{person}{Yaochu Jin}.} \bibinfo{year}{2020}\natexlab{b}.
\newblock \showarticletitle{{Communication-Efficient Federated Deep Learning
  With Layerwise Asynchronous Model Update and Temporally Weighted
  Aggregation}}.
\newblock \bibinfo{journal}{\emph{IEEE Transactions on Neural Networks and
  Learning Systems}} \bibinfo{volume}{31}, \bibinfo{number}{10}
  (\bibinfo{year}{2020}).
\newblock


\bibitem[Cho et~al\mbox{.}(2022)]%
        {cho2020client}
\bibfield{author}{\bibinfo{person}{Yae~Jee Cho}, \bibinfo{person}{Jianyu Wang},
  {and} \bibinfo{person}{Gauri Joshi}.} \bibinfo{year}{2022}\natexlab{}.
\newblock \showarticletitle{{Towards Understanding Biased Client Selection in
  Federated Learning}}. In \bibinfo{booktitle}{\emph{AISTATS}}.
\newblock


\bibitem[Damaskinos et~al\mbox{.}(2020)]%
        {fleet}
\bibfield{author}{\bibinfo{person}{Georgios Damaskinos},
  \bibinfo{person}{Rachid Guerraoui}, \bibinfo{person}{Anne-Marie Kermarrec},
  \bibinfo{person}{Vlad Nitu}, \bibinfo{person}{Rhicheek Patra}, {and}
  \bibinfo{person}{Francois Taiani}.} \bibinfo{year}{2020}\natexlab{}.
\newblock \showarticletitle{{FLeet: Online Federated Learning via Staleness
  Awareness and Performance Prediction}}. In
  \bibinfo{booktitle}{\emph{Middleware}}.
\newblock


\bibitem[Docs(2020)]%
        {xmlrpc}
\bibfield{author}{\bibinfo{person}{Python Docs}.}
  \bibinfo{year}{2020}\natexlab{}.
\newblock \bibinfo{title}{XMLRPC server and client modules}.
\newblock
\newblock
\urldef\tempurl%
\url{https://docs.python.org/3/library/xmlrpc.html}
\showURL{%
\tempurl}


\bibitem[FaceBook(2021)]%
        {opacus}
\bibfield{author}{\bibinfo{person}{FaceBook}.} \bibinfo{year}{2021}\natexlab{}.
\newblock \bibinfo{title}{Opacus: High-speed library for applying differential
  privacy for Pytorch}.
\newblock
\newblock
\urldef\tempurl%
\url{https://github.com/pytorch/opacus}
\showURL{%
\tempurl}


\bibitem[FedAI(2021)]%
        {FAI}
\bibfield{author}{\bibinfo{person}{FedAI}.} \bibinfo{year}{2021}\natexlab{}.
\newblock \bibinfo{title}{{Federated AI Technology Enabler}}.
\newblock
\newblock
\urldef\tempurl%
\url{https://www.fedai.org}
\showURL{%
\tempurl}


\bibitem[Hartmann et~al\mbox{.}(2019)]%
        {hartmann2019federated}
\bibfield{author}{\bibinfo{person}{Florian Hartmann}, \bibinfo{person}{Sunah
  Suh}, \bibinfo{person}{Arkadiusz Komarzewski}, \bibinfo{person}{Tim~D.
  Smith}, {and} \bibinfo{person}{Ilana Segall}.}
  \bibinfo{year}{2019}\natexlab{}.
\newblock \bibinfo{title}{{Federated Learning for Ranking Browser History
  Suggestions}}.
\newblock
\newblock
\showeprint[arxiv]{1911.11807}~[cs.LG]


\bibitem[He et~al\mbox{.}(2016)]%
        {resnet}
\bibfield{author}{\bibinfo{person}{Kaiming He}, \bibinfo{person}{Xiangyu
  Zhang}, \bibinfo{person}{Shaoqing Ren}, {and} \bibinfo{person}{Jian Sun}.}
  \bibinfo{year}{2016}\natexlab{}.
\newblock \showarticletitle{{Deep Residual Learning for Image Recognition}}. In
  \bibinfo{booktitle}{\emph{CVPR}}.
\newblock


\bibitem[Ho et~al\mbox{.}(2013)]%
        {Qirong2013}
\bibfield{author}{\bibinfo{person}{Qirong Ho}, \bibinfo{person}{James Cipar},
  \bibinfo{person}{Henggang Cui}, \bibinfo{person}{Jin~Kyu Kim},
  \bibinfo{person}{Seunghak Lee}, \bibinfo{person}{Phillip~B. Gibbons},
  \bibinfo{person}{Garth~A. Gibson}, \bibinfo{person}{Gregory~R. Ganger}, {and}
  \bibinfo{person}{Eric~P. Xing}.} \bibinfo{year}{2013}\natexlab{}.
\newblock \showarticletitle{{More Effective Distributed ML via a Stale
  Synchronous Parallel Parameter Server}}. In
  \bibinfo{booktitle}{\emph{NeurIPS}}.
\newblock


\bibitem[Hsu et~al\mbox{.}(2020)]%
        {googleVCFL}
\bibfield{author}{\bibinfo{person}{T. Hsu}, \bibinfo{person}{Hang Qi}, {and}
  \bibinfo{person}{Matthew Brown}.} \bibinfo{year}{2020}\natexlab{}.
\newblock \showarticletitle{{Federated Visual Classification with Real-World
  Data Distribution}}. In \bibinfo{booktitle}{\emph{ECCV}}.
\newblock


\bibitem[Huang et~al\mbox{.}(2022)]%
        {Huang2021}
\bibfield{author}{\bibinfo{person}{Jiyue Huang}, \bibinfo{person}{Chi Hong},
  \bibinfo{person}{Yang Liu}, \bibinfo{person}{Lydia~Y. Chen}, {and}
  \bibinfo{person}{Stefanie Roos}.} \bibinfo{year}{2022}\natexlab{}.
\newblock \showarticletitle{{Tackling Mavericks in Federated Learning via
  Adaptive Client Selection Strategy}}. In \bibinfo{booktitle}{\emph{AAAI}}.
\newblock


\bibitem[Hyndman and Athanasopoulos(2021)]%
        {forecastingbook}
\bibfield{author}{\bibinfo{person}{Rob~J Hyndman} {and} \bibinfo{person}{George
  Athanasopoulos}.} \bibinfo{year}{2021}\natexlab{}.
\newblock \bibinfo{booktitle}{\emph{{Forecasting: Principles and Practice}}
  (\bibinfo{edition}{3rd} ed.)}.
\newblock \bibinfo{publisher}{OTexts}, \bibinfo{address}{Melbourne, Australia}.
\newblock


\bibitem[Hyndman et~al\mbox{.}(2002)]%
        {HYNDMAN2002439}
\bibfield{author}{\bibinfo{person}{Rob~J Hyndman}, \bibinfo{person}{Anne~B
  Koehler}, \bibinfo{person}{Ralph~D Snyder}, {and} \bibinfo{person}{Simone
  Grose}.} \bibinfo{year}{2002}\natexlab{}.
\newblock \showarticletitle{A state space framework for automatic forecasting
  using exponential smoothing methods}.
\newblock \bibinfo{journal}{\emph{International Journal of Forecasting}}
  \bibinfo{volume}{18}, \bibinfo{number}{3} (\bibinfo{year}{2002}).
\newblock


\bibitem[Jiang et~al\mbox{.}(2017)]%
        {Jiang2017}
\bibfield{author}{\bibinfo{person}{Jiawei Jiang}, \bibinfo{person}{Bin Cui},
  \bibinfo{person}{Ce Zhang}, {and} \bibinfo{person}{Lele Yu}.}
  \bibinfo{year}{2017}\natexlab{}.
\newblock \showarticletitle{{Heterogeneity-Aware Distributed Parameter
  Servers}}. In \bibinfo{booktitle}{\emph{SIGMOD}}.
\newblock


\bibitem[Jiang et~al\mbox{.}(2019b)]%
        {cameraai}
\bibfield{author}{\bibinfo{person}{Junchen Jiang}, \bibinfo{person}{Yuhao
  Zhou}, \bibinfo{person}{Ganesh Ananthanarayanan}, \bibinfo{person}{Yuanchao
  Shu}, {and} \bibinfo{person}{Andrew~A. Chien}.}
  \bibinfo{year}{2019}\natexlab{b}.
\newblock \showarticletitle{{Networked Cameras Are the New Big Data Clusters}}.
  In \bibinfo{booktitle}{\emph{HotEdgeVideo}}.
\newblock


\bibitem[Jiang et~al\mbox{.}(2019a)]%
        {jiang2019improving}
\bibfield{author}{\bibinfo{person}{Yihan Jiang}, \bibinfo{person}{Jakub
  Konečný}, \bibinfo{person}{Keith Rush}, {and} \bibinfo{person}{Sreeram
  Kannan}.} \bibinfo{year}{2019}\natexlab{a}.
\newblock \bibinfo{title}{{Improving Federated Learning Personalization via
  Model Agnostic Meta Learning}}.
\newblock
\newblock
\showeprint[arxiv]{1909.12488}~[cs.LG]


\bibitem[Kairouz et~al\mbox{.}(2019)]%
        {kairouz2019advances}
\bibfield{author}{\bibinfo{person}{Peter Kairouz}, \bibinfo{person}{H.~Brendan
  McMahan}, \bibinfo{person}{Brendan Avent}, \bibinfo{person}{Aur{\'{e}}lien
  Bellet}, \bibinfo{person}{Mehdi Bennis}, \bibinfo{person}{Arjun~Nitin
  Bhagoji}, \bibinfo{person}{Kallista~A. Bonawitz}, \bibinfo{person}{Zachary
  Charles}, \bibinfo{person}{Graham Cormode}, \bibinfo{person}{Rachel
  Cummings}, \bibinfo{person}{Rafael G.~L. D'Oliveira},
  \bibinfo{person}{Salim~El Rouayheb}, \bibinfo{person}{David Evans},
  \bibinfo{person}{Josh Gardner}, \bibinfo{person}{Zachary Garrett},
  \bibinfo{person}{Adri{\`{a}} Gasc{\'{o}}n}, \bibinfo{person}{Badih Ghazi},
  \bibinfo{person}{Phillip~B. Gibbons}, \bibinfo{person}{Marco Gruteser},
  \bibinfo{person}{Za{\"{\i}}d Harchaoui}, \bibinfo{person}{Chaoyang He},
  \bibinfo{person}{Lie He}, \bibinfo{person}{Zhouyuan Huo},
  \bibinfo{person}{Ben Hutchinson}, \bibinfo{person}{Justin Hsu},
  \bibinfo{person}{Martin Jaggi}, \bibinfo{person}{Tara Javidi},
  \bibinfo{person}{Gauri Joshi}, \bibinfo{person}{Mikhail Khodak},
  \bibinfo{person}{Jakub Kone{\v{c}}n{\'y}}, \bibinfo{person}{Aleksandra
  Korolova}, \bibinfo{person}{Farinaz Koushanfar}, \bibinfo{person}{Sanmi
  Koyejo}, \bibinfo{person}{Tancr{\`{e}}de Lepoint}, \bibinfo{person}{Yang
  Liu}, \bibinfo{person}{Prateek Mittal}, \bibinfo{person}{Mehryar Mohri},
  \bibinfo{person}{Richard Nock}, \bibinfo{person}{Ayfer {\"{O}}zg{\"{u}}r},
  \bibinfo{person}{Rasmus Pagh}, \bibinfo{person}{Mariana Raykova},
  \bibinfo{person}{Hang Qi}, \bibinfo{person}{Daniel Ramage},
  \bibinfo{person}{Ramesh Raskar}, \bibinfo{person}{Dawn Song},
  \bibinfo{person}{Weikang Song}, \bibinfo{person}{Sebastian~U. Stich},
  \bibinfo{person}{Ziteng Sun}, \bibinfo{person}{Ananda~Theertha Suresh},
  \bibinfo{person}{Florian Tram{\`{e}}r}, \bibinfo{person}{Praneeth Vepakomma},
  \bibinfo{person}{Jianyu Wang}, \bibinfo{person}{Li Xiong},
  \bibinfo{person}{Zheng Xu}, \bibinfo{person}{Qiang Yang},
  \bibinfo{person}{Felix~X. Yu}, \bibinfo{person}{Han Yu}, {and}
  \bibinfo{person}{Sen Zhao}.} \bibinfo{year}{2019}\natexlab{}.
\newblock \bibinfo{title}{{Advances and Open Problems in Federated Learning}}.
\newblock
\newblock
\showeprint[arxiv]{1912.04977}~[cs.LG]


\bibitem[Kone{\v c}n{\'y} et~al\mbox{.}(2016)]%
        {Jakub2016}
\bibfield{author}{\bibinfo{person}{Jakub Kone{\v c}n{\'y}},
  \bibinfo{person}{H.~Brendan McMahan}, \bibinfo{person}{Felix~X. Yu},
  \bibinfo{person}{Peter Richt{\'a}rik}, \bibinfo{person}{Ananda~Theertha
  Suresh}, {and} \bibinfo{person}{Dave Bacon}.}
  \bibinfo{year}{2016}\natexlab{}.
\newblock \showarticletitle{{Federated Learning: Strategies for Improving
  Communication Efficiency}}. In \bibinfo{booktitle}{\emph{Workshop on Private
  Multi-Party Machine Learning}}.
\newblock


\bibitem[Krizhevsky(2009)]%
        {cifar10}
\bibfield{author}{\bibinfo{person}{Alex Krizhevsky}.}
  \bibinfo{year}{2009}\natexlab{}.
\newblock \bibinfo{booktitle}{\emph{Learning Multiple Layers of Features from
  Tiny Images}}.
\newblock \bibinfo{type}{{T}echnical {R}eport}.
  \bibinfo{institution}{University of Toronto}.
\newblock


\bibitem[Kuznetsova et~al\mbox{.}(2020)]%
        {openimage}
\bibfield{author}{\bibinfo{person}{Alina Kuznetsova}, \bibinfo{person}{Hassan
  Rom}, \bibinfo{person}{Neil Alldrin}, \bibinfo{person}{Jasper Uijlings},
  \bibinfo{person}{Ivan Krasin}, \bibinfo{person}{Jordi Pont-Tuset},
  \bibinfo{person}{Shahab Kamali}, \bibinfo{person}{Stefan Popov},
  \bibinfo{person}{Matteo Malloci}, \bibinfo{person}{Alexander Kolesnikov},
  \bibinfo{person}{Tom Duerig}, {and} \bibinfo{person}{Vittorio Ferrari}.}
  \bibinfo{year}{2020}\natexlab{}.
\newblock \showarticletitle{The Open Images Dataset V4: Unified image
  classification, object detection, and visual relationship detection at
  scale}.
\newblock \bibinfo{journal}{\emph{International Journal of Computer Vision}}
  \bibinfo{volume}{128} (\bibinfo{year}{2020}).
\newblock


\bibitem[Lai et~al\mbox{.}(2022)]%
        {lai2021fedscale}
\bibfield{author}{\bibinfo{person}{Fan Lai}, \bibinfo{person}{Yinwei Dai},
  \bibinfo{person}{Xiangfeng Zhu}, {and} \bibinfo{person}{Mosharaf Chowdhury}.}
  \bibinfo{year}{2022}\natexlab{}.
\newblock \showarticletitle{{FedScale: Benchmarking Model and System
  Performance of Federated Learning}}. In \bibinfo{booktitle}{\emph{ICML}}.
\newblock


\bibitem[Lai et~al\mbox{.}(2021)]%
        {Oort-osdi21}
\bibfield{author}{\bibinfo{person}{Fan Lai}, \bibinfo{person}{Xiangfeng Zhu},
  \bibinfo{person}{Harsha~V. Madhyastha}, {and} \bibinfo{person}{Mosharaf
  Chowdhury}.} \bibinfo{year}{2021}\natexlab{}.
\newblock \showarticletitle{{Efficient Federated Learning via Guided
  Participant Selection}}. In \bibinfo{booktitle}{\emph{OSDI}}.
\newblock


\bibitem[Lan et~al\mbox{.}(2020)]%
        {albert}
\bibfield{author}{\bibinfo{person}{Zhenzhong Lan}, \bibinfo{person}{Mingda
  Chen}, \bibinfo{person}{Sebastian Goodman}, \bibinfo{person}{Kevin Gimpel},
  \bibinfo{person}{Piyush Sharma}, {and} \bibinfo{person}{Radu Soricut}.}
  \bibinfo{year}{2020}\natexlab{}.
\newblock \showarticletitle{{ALBERT: A Lite BERT for Self-supervised Learning
  of Language Representations}}. In \bibinfo{booktitle}{\emph{ICLR}}.
\newblock


\bibitem[Li et~al\mbox{.}(2021)]%
        {Li2021}
\bibfield{author}{\bibinfo{person}{Li Li}, \bibinfo{person}{Moming Duan},
  \bibinfo{person}{Duo Liu}, \bibinfo{person}{Yu Zhang}, \bibinfo{person}{Ao
  Ren}, \bibinfo{person}{Xianzhang Chen}, \bibinfo{person}{Yujuan Tan}, {and}
  \bibinfo{person}{Chengliang Wang}.} \bibinfo{year}{2021}\natexlab{}.
\newblock \showarticletitle{{FedSAE: A Novel Self-Adaptive Federated Learning
  Framework in Heterogeneous Systems}}. In \bibinfo{booktitle}{\emph{IJCNN}}.
\newblock


\bibitem[Li et~al\mbox{.}(2019b)]%
        {Li2019smartpc}
\bibfield{author}{\bibinfo{person}{Li Li}, \bibinfo{person}{Haoyi Xiong},
  \bibinfo{person}{Zhishan Guo}, \bibinfo{person}{Jun Wang}, {and}
  \bibinfo{person}{Cheng-Zhong Xu}.} \bibinfo{year}{2019}\natexlab{b}.
\newblock \showarticletitle{{SmartPC: Hierarchical Pace Control in Real-Time
  Federated Learning System}}. In \bibinfo{booktitle}{\emph{RTSS}}.
\newblock


\bibitem[Li et~al\mbox{.}(2022)]%
        {Li2021a}
\bibfield{author}{\bibinfo{person}{Qinbin Li}, \bibinfo{person}{Yiqun Diao},
  \bibinfo{person}{Quan Chen}, {and} \bibinfo{person}{Bingsheng He}.}
  \bibinfo{year}{2022}\natexlab{}.
\newblock \showarticletitle{{Federated Learning on Non-IID Data Silos: An
  Experimental Study}}. In \bibinfo{booktitle}{\emph{ICDE}}.
\newblock


\bibitem[Li et~al\mbox{.}(2020a)]%
        {Li2020FedProx}
\bibfield{author}{\bibinfo{person}{Tian Li}, \bibinfo{person}{Anit~Kumar Sahu},
  \bibinfo{person}{Manzil Zaheer}, \bibinfo{person}{Maziar Sanjabi},
  \bibinfo{person}{Ameet Talwalkar}, {and} \bibinfo{person}{Virginia Smith}.}
  \bibinfo{year}{2020}\natexlab{a}.
\newblock \showarticletitle{Federated Optimization in Heterogeneous Networks}.
  In \bibinfo{booktitle}{\emph{MLSys}}.
\newblock


\bibitem[Li et~al\mbox{.}(2020b)]%
        {Li2020Fair}
\bibfield{author}{\bibinfo{person}{Tian Li}, \bibinfo{person}{Maziar Sanjabi},
  \bibinfo{person}{Ahmad Beirami}, {and} \bibinfo{person}{Virginia Smith}.}
  \bibinfo{year}{2020}\natexlab{b}.
\newblock \showarticletitle{Fair Resource Allocation in Federated Learning}. In
  \bibinfo{booktitle}{\emph{ICLR}}.
\newblock


\bibitem[Li et~al\mbox{.}(2019a)]%
        {nvidiafl}
\bibfield{author}{\bibinfo{person}{Wenqi Li}, \bibinfo{person}{Fausto
  Milletar{\`i}}, \bibinfo{person}{Daguang Xu}, \bibinfo{person}{Nicola Rieke},
  \bibinfo{person}{Jonny Hancox}, \bibinfo{person}{Wentao Zhu},
  \bibinfo{person}{Maximilian Baust}, \bibinfo{person}{Yan Cheng},
  \bibinfo{person}{S{\'e}bastien Ourselin}, \bibinfo{person}{M.~Jorge Cardoso},
  {and} \bibinfo{person}{Andrew Feng}.} \bibinfo{year}{2019}\natexlab{a}.
\newblock \showarticletitle{{Privacy-Preserving Federated Brain Tumour
  Segmentation}}. In \bibinfo{booktitle}{\emph{{Machine Learning in Medical
  Imaging}}}.
\newblock


\bibitem[M-Lab(2021)]%
        {mobiperf}
\bibfield{author}{\bibinfo{person}{M-Lab}.} \bibinfo{year}{2021}\natexlab{}.
\newblock \bibinfo{title}{{MobiPerf: an open source application for measuring
  network performance on mobile platforms}}.
\newblock
\newblock
\urldef\tempurl%
\url{https://www.measurementlab.net/tests/mobiperf/}
\showURL{%
\tempurl}


\bibitem[Mania et~al\mbox{.}(2017)]%
        {mania2017perturbed}
\bibfield{author}{\bibinfo{person}{Horia Mania}, \bibinfo{person}{Xinghao Pan},
  \bibinfo{person}{Dimitris Papailiopoulos}, \bibinfo{person}{Benjamin Recht},
  \bibinfo{person}{Kannan Ramchandran}, {and} \bibinfo{person}{Michael~I
  Jordan}.} \bibinfo{year}{2017}\natexlab{}.
\newblock \showarticletitle{{Perturbed Iterate Analysis for Asynchronous
  Stochastic Optimization}}.
\newblock \bibinfo{journal}{\emph{SIAM Journal on Optimization}}
  \bibinfo{volume}{27}, \bibinfo{number}{4} (\bibinfo{year}{2017}).
\newblock


\bibitem[McMahan et~al\mbox{.}(2018)]%
        {McMahan2018}
\bibfield{author}{\bibinfo{person}{Brendan McMahan}, \bibinfo{person}{Daniel
  Ramage}, \bibinfo{person}{Kunal Talwar}, {and} \bibinfo{person}{Li Zhang}.}
  \bibinfo{year}{2018}\natexlab{}.
\newblock \showarticletitle{{Learning Differentially Private Recurrent Language
  Models}}. In \bibinfo{booktitle}{\emph{ICLR}}.
\newblock


\bibitem[McMahan et~al\mbox{.}(2017)]%
        {mcmahan2017}
\bibfield{author}{\bibinfo{person}{H.~Brendan McMahan}, \bibinfo{person}{Eider
  Moore}, \bibinfo{person}{Daniel Ramage}, \bibinfo{person}{Seth Hampson},
  {and} \bibinfo{person}{Blaise~Agüera y Arcas}.}
  \bibinfo{year}{2017}\natexlab{}.
\newblock \showarticletitle{Communication-Efficient Learning of Deep Networks
  from Decentralized Data}. In \bibinfo{booktitle}{\emph{AISTATS}}.
\newblock


\bibitem[Melis et~al\mbox{.}(2019)]%
        {Melis2019}
\bibfield{author}{\bibinfo{person}{Luca Melis}, \bibinfo{person}{Congzheng
  Song}, \bibinfo{person}{Emiliano {De Cristofaro}}, {and}
  \bibinfo{person}{Vitaly Shmatikov}.} \bibinfo{year}{2019}\natexlab{}.
\newblock \showarticletitle{{Exploiting Unintended Feature Leakage in
  Collaborative Learning}}. In \bibinfo{booktitle}{\emph{SP}}.
\newblock


\bibitem[Mohri et~al\mbox{.}(2019)]%
        {Mohri2019AgnosticFL}
\bibfield{author}{\bibinfo{person}{Mehryar Mohri}, \bibinfo{person}{Gary
  Sivek}, {and} \bibinfo{person}{Ananda~Theertha Suresh}.}
  \bibinfo{year}{2019}\natexlab{}.
\newblock \showarticletitle{Agnostic Federated Learning}. In
  \bibinfo{booktitle}{\emph{ICML}}.
\newblock


\bibitem[Nasr et~al\mbox{.}(2019)]%
        {Nasr2019}
\bibfield{author}{\bibinfo{person}{Milad Nasr}, \bibinfo{person}{Reza Shokri},
  {and} \bibinfo{person}{Amir Houmansadr}.} \bibinfo{year}{2019}\natexlab{}.
\newblock \showarticletitle{{Comprehensive Privacy Analysis of Deep Learning:
  Passive and Active White-box Inference Attacks against Centralized and
  Federated Learning}}. In \bibinfo{booktitle}{\emph{SP}}.
\newblock


\bibitem[Nishio and Yonetani(2019)]%
        {Nishio2018}
\bibfield{author}{\bibinfo{person}{Takayuki Nishio} {and} \bibinfo{person}{Ryo
  Yonetani}.} \bibinfo{year}{2019}\natexlab{}.
\newblock \showarticletitle{{Client Selection for Federated Learning with
  Heterogeneous Resources in Mobile Edge}}. In \bibinfo{booktitle}{\emph{ICC}}.
\newblock


\bibitem[PaddlePaddle.org(2020)]%
        {paddle}
\bibfield{author}{\bibinfo{person}{PaddlePaddle.org}.}
  \bibinfo{year}{2020}\natexlab{}.
\newblock \bibinfo{title}{PArallel Distributed Deep LEarning: Machine Learning
  Framework from Industrial Practice}.
\newblock
\newblock
\urldef\tempurl%
\url{https://github.com/PaddlePaddle/PaddleFL}
\showURL{%
\tempurl}


\bibitem[Pushshift(2020)]%
        {REDDIT}
\bibfield{author}{\bibinfo{person}{Pushshift}.}
  \bibinfo{year}{2020}\natexlab{}.
\newblock \bibinfo{title}{Reddit Datasets}.
\newblock
\newblock
\urldef\tempurl%
\url{https://files.pushshift.io/reddit/}
\showURL{%
\tempurl}


\bibitem[Ramaswamy et~al\mbox{.}(2020)]%
        {YoGi}
\bibfield{author}{\bibinfo{person}{Swaroop Ramaswamy}, \bibinfo{person}{Om
  Thakkar}, \bibinfo{person}{Rajiv Mathews}, \bibinfo{person}{Galen Andrew},
  \bibinfo{person}{H.~Brendan McMahan}, {and} \bibinfo{person}{Françoise
  Beaufays}.} \bibinfo{year}{2020}\natexlab{}.
\newblock \bibinfo{title}{{Training Production Language Models without
  Memorizing User Data}}.
\newblock
\newblock
\showeprint[arxiv]{2009.10031}~[cs.LG]


\bibitem[Reisizadeh et~al\mbox{.}(2020)]%
        {pmlr-v108-reisizadeh20a}
\bibfield{author}{\bibinfo{person}{Amirhossein Reisizadeh},
  \bibinfo{person}{Aryan Mokhtari}, \bibinfo{person}{Hamed Hassani},
  \bibinfo{person}{Ali Jadbabaie}, {and} \bibinfo{person}{Ramtin Pedarsani}.}
  \bibinfo{year}{2020}\natexlab{}.
\newblock \showarticletitle{{FedPAQ: A Communication-Efficient Federated
  Learning Method with Periodic Averaging and Quantization}}. In
  \bibinfo{booktitle}{\emph{AISTATS}}.
\newblock


\bibitem[Ruan et~al\mbox{.}(2021)]%
        {pmlr-v130-ruan21a}
\bibfield{author}{\bibinfo{person}{Yichen Ruan}, \bibinfo{person}{Xiaoxi
  Zhang}, \bibinfo{person}{Shu-Che Liang}, {and} \bibinfo{person}{Carlee
  Joe-Wong}.} \bibinfo{year}{2021}\natexlab{}.
\newblock \showarticletitle{{Towards Flexible Device Participation in Federated
  Learning}}. In \bibinfo{booktitle}{\emph{AISTATS}}.
\newblock


\bibitem[Ryffel et~al\mbox{.}(2018)]%
        {ryffel2018generic}
\bibfield{author}{\bibinfo{person}{Theo Ryffel}, \bibinfo{person}{Andrew
  Trask}, \bibinfo{person}{Morten Dahl}, \bibinfo{person}{Bobby Wagner},
  \bibinfo{person}{Jason Mancuso}, \bibinfo{person}{Daniel Rueckert}, {and}
  \bibinfo{person}{Jonathan Passerat-Palmbach}.}
  \bibinfo{year}{2018}\natexlab{}.
\newblock \bibinfo{title}{A generic framework for privacy preserving deep
  learning}.
\newblock
\newblock
\showeprint[arxiv]{1811.04017}~[cs.LG]


\bibitem[{Sebastian U. Stich and Sai Praneeth Karimireddy}(2020)]%
        {stich2019error}
\bibfield{author}{\bibinfo{person}{{Sebastian U. Stich and Sai Praneeth
  Karimireddy}}.} \bibinfo{year}{2020}\natexlab{}.
\newblock \showarticletitle{{The Error-Feedback Framework: Better Rates for SGD
  with Delayed Gradients and Compressed Updates}}.
\newblock \bibinfo{journal}{\emph{Journal of Machine Learning Research}}
  \bibinfo{volume}{21} (\bibinfo{year}{2020}).
\newblock


\bibitem[Smith et~al\mbox{.}(2017)]%
        {Smith2017}
\bibfield{author}{\bibinfo{person}{Virginia Smith}, \bibinfo{person}{Chao-Kai
  Chiang}, \bibinfo{person}{Maziar Sanjabi}, {and} \bibinfo{person}{Ameet~S
  Talwalkar}.} \bibinfo{year}{2017}\natexlab{}.
\newblock \showarticletitle{{Federated Multi-Task Learning}}.
\newblock In \bibinfo{booktitle}{\emph{NeurIPS}}.
\newblock


\bibitem[Srinivasan et~al\mbox{.}(2014)]%
        {prediction1}
\bibfield{author}{\bibinfo{person}{Vijay Srinivasan}, \bibinfo{person}{Saeed
  Moghaddam}, \bibinfo{person}{Abhishek Mukherji}, \bibinfo{person}{Kiran~K.
  Rachuri}, \bibinfo{person}{Chenren Xu}, {and}
  \bibinfo{person}{Emmanuel~Munguia Tapia}.} \bibinfo{year}{2014}\natexlab{}.
\newblock \showarticletitle{MobileMiner: Mining Your Frequent Patterns on Your
  Phone}. In \bibinfo{booktitle}{\emph{UbiComp}}.
\newblock


\bibitem[Szabó et~al\mbox{.}(2019)]%
        {stunner}
\bibfield{author}{\bibinfo{person}{Zoltán Szabó}, \bibinfo{person}{Krisztián
  Téglás}, \bibinfo{person}{Arpád Berta}, \bibinfo{person}{Márk Jelasity},
  {and} \bibinfo{person}{Vilmos Bilicki}.} \bibinfo{year}{2019}\natexlab{}.
\newblock \showarticletitle{{Stunner: A Smart Phone Trace for Developing
  Decentralized Edge Systems}}. In \bibinfo{booktitle}{\emph{{DAIS}}}.
\newblock


\bibitem[Taylor and Letham(2017)]%
        {prophet}
\bibfield{author}{\bibinfo{person}{Sean~J Taylor} {and}
  \bibinfo{person}{Benjamin Letham}.} \bibinfo{year}{2017}\natexlab{}.
\newblock \bibinfo{title}{Forecasting at scale}.
\newblock
\newblock
\showeprint[{PeerJ Preprints}]{{5:e3190v2}}


\bibitem[Team(2017)]%
        {applefl}
\bibfield{author}{\bibinfo{person}{Apple Differential~Privacy Team}.}
  \bibinfo{year}{2017}\natexlab{}.
\newblock \showarticletitle{{Learning with privacy at scale}}.
\newblock \bibinfo{journal}{\emph{Apple Machine Learning Journal}}
  (\bibinfo{year}{2017}).
\newblock


\bibitem[tensorflow.org(2020)]%
        {tff}
\bibfield{author}{\bibinfo{person}{tensorflow.org}.}
  \bibinfo{year}{2020}\natexlab{}.
\newblock \bibinfo{title}{TensorFlow Federated: Machine Learning on
  Decentralized Data}.
\newblock
\newblock
\urldef\tempurl%
\url{https://www.tensorflow.org/federated}
\showURL{%
\tempurl}


\bibitem[Wang and Joshi(2019)]%
        {wang2019adaptive}
\bibfield{author}{\bibinfo{person}{Jianyu Wang} {and} \bibinfo{person}{Gauri
  Joshi}.} \bibinfo{year}{2019}\natexlab{}.
\newblock \showarticletitle{{Adaptive Communication Strategies to Achieve the
  Best Error-Runtime Trade-off in Local-Update SGD}}. In
  \bibinfo{booktitle}{\emph{MLSys}}.
\newblock


\bibitem[Wang et~al\mbox{.}(2020)]%
        {wang2020tackling}
\bibfield{author}{\bibinfo{person}{Jianyu Wang}, \bibinfo{person}{Qinghua Liu},
  \bibinfo{person}{Hao Liang}, \bibinfo{person}{Gauri Joshi}, {and}
  \bibinfo{person}{H.~Vincent Poor}.} \bibinfo{year}{2020}\natexlab{}.
\newblock \showarticletitle{Tackling the Objective Inconsistency Problem in
  Heterogeneous Federated Optimization}. In
  \bibinfo{booktitle}{\emph{NeurIPS}}.
\newblock


\bibitem[Warden(2018)]%
        {googlespeech}
\bibfield{author}{\bibinfo{person}{Pete Warden}.}
  \bibinfo{year}{2018}\natexlab{}.
\newblock \bibinfo{title}{{Speech Commands: A Dataset for Limited-Vocabulary
  Speech Recognition}}.
\newblock
\newblock
\showeprint[arxiv]{1804.03209}~[cs.CL]


\bibitem[Wu et~al\mbox{.}(2021)]%
        {Safa-Wu2021}
\bibfield{author}{\bibinfo{person}{Wentai Wu}, \bibinfo{person}{Ligang He},
  \bibinfo{person}{Weiwei Lin}, \bibinfo{person}{Rui Mao},
  \bibinfo{person}{Carsten Maple}, {and} \bibinfo{person}{Stephen Jarvis}.}
  \bibinfo{year}{2021}\natexlab{}.
\newblock \showarticletitle{{SAFA: A Semi-Asynchronous Protocol for Fast
  Federated Learning With Low Overhead}}.
\newblock \bibinfo{journal}{\emph{IEEE Trans. Comput.}} \bibinfo{volume}{70},
  \bibinfo{number}{5} (\bibinfo{year}{2021}).
\newblock


\bibitem[Xie et~al\mbox{.}(2020)]%
        {Xie2019}
\bibfield{author}{\bibinfo{person}{Cong Xie}, \bibinfo{person}{Oluwasanmi
  Koyejo}, {and} \bibinfo{person}{Indranil Gupta}.}
  \bibinfo{year}{2020}\natexlab{}.
\newblock \showarticletitle{{Asynchronous Federated Optimization}}. In
  \bibinfo{booktitle}{\emph{Workshop on Optimization for Machine Learning}}.
\newblock


\bibitem[Yang et~al\mbox{.}(2018b)]%
        {prediction2}
\bibfield{author}{\bibinfo{person}{Carl Yang}, \bibinfo{person}{Xiaolin Shi},
  \bibinfo{person}{Luo Jie}, {and} \bibinfo{person}{Jiawei Han}.}
  \bibinfo{year}{2018}\natexlab{b}.
\newblock \showarticletitle{{I Know You'll Be Back: Interpretable New User
  Clustering and Churn Prediction on a Mobile Social Application}}. In
  \bibinfo{booktitle}{\emph{KDD}}.
\newblock


\bibitem[Yang et~al\mbox{.}(2021)]%
        {yang2020heterogeneityaware}
\bibfield{author}{\bibinfo{person}{Chengxu Yang}, \bibinfo{person}{Qipeng
  Wang}, \bibinfo{person}{Mengwei Xu}, \bibinfo{person}{Zhenpeng Chen},
  \bibinfo{person}{Kaigui Bian}, \bibinfo{person}{Yunxin Liu}, {and}
  \bibinfo{person}{Xuanzhe Liu}.} \bibinfo{year}{2021}\natexlab{}.
\newblock \showarticletitle{Characterizing Impacts of Heterogeneity in
  Federated Learning upon Large-Scale Smartphone Data}. In
  \bibinfo{booktitle}{\emph{The Web Conference}}.
\newblock


\bibitem[Yang et~al\mbox{.}(2018a)]%
        {yang2018applied}
\bibfield{author}{\bibinfo{person}{Timothy Yang}, \bibinfo{person}{Galen
  Andrew}, \bibinfo{person}{Hubert Eichner}, \bibinfo{person}{Haicheng Sun},
  \bibinfo{person}{Wei Li}, \bibinfo{person}{Nicholas Kong},
  \bibinfo{person}{Daniel Ramage}, {and} \bibinfo{person}{Françoise
  Beaufays}.} \bibinfo{year}{2018}\natexlab{a}.
\newblock \bibinfo{title}{{Applied Federated Learning: Improving Google
  Keyboard Query Suggestions}}.
\newblock
\newblock
\showeprint[arxiv]{1812.02903}~[cs.LG]


\bibitem[Zhang et~al\mbox{.}(2016)]%
        {staleaware}
\bibfield{author}{\bibinfo{person}{Wei Zhang}, \bibinfo{person}{Suyog Gupta},
  \bibinfo{person}{Xiangru Lian}, {and} \bibinfo{person}{Ji Liu}.}
  \bibinfo{year}{2016}\natexlab{}.
\newblock \showarticletitle{{Staleness-Aware Async-SGD for Distributed Deep
  Learning}}. In \bibinfo{booktitle}{\emph{IJCAI}}.
\newblock


\bibitem[Zhang et~al\mbox{.}(2018)]%
        {shufflenet}
\bibfield{author}{\bibinfo{person}{Xiangyu Zhang}, \bibinfo{person}{Xinyu
  Zhou}, \bibinfo{person}{Mengxiao Lin}, {and} \bibinfo{person}{Jian Sun}.}
  \bibinfo{year}{2018}\natexlab{}.
\newblock \showarticletitle{{ShuffleNet: An Extremely Efficient Convolutional
  Neural Network for Mobile Devices}}. In \bibinfo{booktitle}{\emph{CVPR}}.
\newblock


\end{thebibliography}

\end{document}